\documentclass[twoside,11pt]{article}

\usepackage{blindtext}

\usepackage[preprint]{jmlr2e}

\usepackage{amsmath}
\usepackage{algorithm}
\usepackage{algorithmic}
\newtheorem{assumption}[theorem]{Assumption}

\usepackage{lastpage}

\ShortHeadings{Robust Conformalized Selection with Noisy Responses}{YU, Wei and Jing}
\firstpageno{1}

\begin{document}

\title{Robust Conformalized Selection with Noisy Responses}

\author{\name Chengyao Yu
\email 12532239@mail.sustech.edu.cn\\
\addr Department of Statistics and Data Science\\
Southern University of Science and Technology\\
Shenzhen 518055, China
\AND
\name Hongxin Wei
\email weihx@sustech.edu.cn\\
\addr Department of Statistics and Data Science\\
Southern University of Science and Technology\\
Shenzhen 518055, China
\AND
\name Bingyi Jing\thanks{Corresponding author.} \email bingyijing@cuhk.edu.cn \\
       \addr School of Artificial Intelligence\\
       The Chinese University of Hong Kong, Shenzhen\\
       Shenzhen Loop Area Institute\\
       Shenzhen 518172, China}

\editor{My editor}

\maketitle

\begin{abstract}
Conformalized selection has been widely applied to select high-quality candidates from large datasets with rigorous uncertainty quantification, such as reliable labeling, drug discovery, and the alignment of large language models. Nevertheless, existing methods assume clean responses on calibration data, an assumption that rarely holds in practice. In this paper, we formulate the above tasks as selecting candidates with true predicted labels or with responses exceeding certain values. We demonstrate that existing conformal selection methods fail to control the false discovery rate (FDR) or suffer from severe power loss under contaminated calibration data. To that end, we propose \textit{Robust Conformalized Selection} (RCS), a unified framework for selective classification with valid FDR control under general label contamination. The key insight of RCS lies in a novel statistical reduction: by separately conditioning on different classes, we translate the intractable label noise into a localized covariate shift problem, which then enables a covariate-adjusted empirical-Bayes-type estimate of the number of false selections. Statistical properties such as the asymptotic FDR control, power optimality, and robustness of RCS are established. We further develop an instantiation of RCS under randomized response model, and also apply RCS to the task of selecting candidates with large response values. Extensive experiments on both simulated and real-world datasets demonstrate the effectiveness of RCS.
\end{abstract}


\section{Introduction}
Many scientific discovery and decision-making tasks often revolve around filtering a large pool of candidates to find a subset of promising ones. For classification tasks such as automatic data labeling, the goal is to select candidates that have the correct predicted labels~\citep{wang2022unsupervised,candes2025probably}. For regression tasks such as drug discovery, scientists aim to find drug candidates that have high binding affinity to a target from huge chemical spaces~\citep{szymanski2011adaptation, scannell2022predictive,zhang2025artificial}.
For these problems, the selection is usually built on the predicted responses (e.g., labels) of (black-box) machine learning models, since the true test responses are often too expensive to obtain. Furthermore, to ensure the reliability of selected candidates, the selection algorithm is required to rigorously control false selections.
These considerations highlight the necessity of model-free, principled selection methods.

Several approaches have been proposed to tackle the above problems, among which a popular direction is the \textit{conformalized selection}: constructing conformal p-values based on the spirit of conformal prediction~\citep{vovk2005algorithmic} and subsequently applying a multiple testing framework. A pioneering method is conformal selection (CS)~\citep{jin2023selection}, a model-free method that selects candidates whose unobserved outcomes exceed user-specified values while ensuring finite-sample false discovery rate (FDR) control~\citep{benjamini1995controlling}. Subsequently, numerous studies such as \cite{jin2025model, lu2025feedback, bai2025multivariate, gui2025acs,gui2024conformal, wang2024conformalized} expanded CS in different directions. Another important line includes the work of \cite{rava2021burden,sun2025unified}, which aims to select candidates that are correctly classified while guaranteeing FDR control. Nonetheless, all aforementioned studies assume \textit{perfect response accuracy} of the calibration data. To the best of our knowledge, no work has explored conformalized selection under \textit{contaminated} calibration data.

In many scenarios, however, it can be expensive or even infeasible to acquire data with accurate responses~\citep{snow2008cheap, natarajan2013learning, sukhbaatar2014training, kennedy2020shape, song2022learning}:
\begin{itemize}
\item In \textit{label annotation}, crowd-sourced platforms such as Amazon Mechanical Turk enable scalable annotation~\citep{sorokin2008utility}, but the resulting labels may be low-quality due to limited expertise, inattention, or task ambiguity. Similar imperfections also arise in expert-driven domains such as medical imaging, where inter-observer variability is common~\citep{karimi2020deep}.
\item In \textit{differential privacy}, for responses that contain personal information, the recorded value is usually a randomized version of the true response to protect user privacy~\citep{evfimievski2003limiting, kasiviswanathan2011can, ghazi2021deep}.
\item In \textit{drug discovery}, responses are often assay-derived measurements or thresholded active/inactive calls. Assay variability, plate or batch effects, and chemical interference can therefore contaminate the recorded responses~\citep{kramer2012experimental}.
\end{itemize}
Applying prior conformal methods naively to such noisy data is concerning, as contamination directly undermines the exchangeability between calibration and test data, which is a fundamental assumption for the validity of these methods. Although a growing body of work has begun to study conformal prediction under contaminated data~\citep{barber2023conformal,cauchois2024predictive, einbinder2024label,penso2024conformal, xi2025exploring, sesia2025adaptive}, their results focus on the prediction of a single test point; selection under noisy data remains unexplored.
The primary challenge lies in constructing suitable test statistics and handling their complex dependence as well as multiplicity issues~\citep{angelopoulos2023conformal}. This paper aims to fill these gaps.

\subsection{Our Contributions}
We focus on \textit{selection under contamination} for two important tasks:
\begin{itemize}
\item[1.] Given unlabeled test samples, a pre-trained classifier, and calibration data with noisy labels, select test samples that are correctly classified while controlling FDR.
\item[2.] Given the features of a pool of candidates, select those whose unobserved responses exceed certain values while maintaining FDR control.
\end{itemize}
We theoretically demonstrate that prior frameworks, including \cite{jin2023selection,sun2025unified}, can fail to control FDR or suffer severe power loss under contaminated calibration data for the above tasks.
To this end, we develop \textit{Robust Conformalized Selection} (RCS), a unified framework for selective classification that provides valid group-wise FDR control under general noisy calibration data (Task 1). We then specialize the RCS approach under the randomized response model. Finally, to tackle Task 2, a robust alternative to CS~\citep{jin2023selection} is developed by leveraging RCS.

Methodologically, the key insight of RCS lies in translating the label noise problem to \textit{localized covariate shift}: by conditioning on each predicted class, we employ a \textit{covariate-adjusted} empirical-Bayes type estimate of FDP to facilitate FDR control; see Remark~\ref{bayes_fdp} for detailed discussions. Under this perspective, the proposed RCS is surprisingly simplified for the common randomized response model. Theoretically, we establish the asymptotic FDR control, the power optimality, and the robustness of RCS.
Numerical results in simulations and real data analysis demonstrate that RCS always controls FDR and can be substantially more powerful than existing methods.

\paragraph*{Paper outline.}
The rest of the paper is organized as follows. Section~\ref{sec2} formulates the two tasks and recaps the related methods.
Section \ref{sec3} investigates the effects of response contamination on existing methods from both empirical and theoretical perspectives.
Section \ref{sec4} develops the RCS framework for Task 1 under general label contamination, establishes its theoretical properties, and discusses its implementation under class-conditional label noise.
Section \ref{sec5} develops a simplified specialization of RCS under randomized response model and applies RCS to Task 2.
Simulation results and real data applications are presented in Sections \ref{sec6} and \ref{sec7}, respectively. Proofs of all theoretical results are deferred to Appendix~\ref{proofs}.
The reproduction codes for all experiments can be found at \url{https://github.com/ChengyaoYu1/Robust-Conformalized-Selection}.

\section{Preliminary}\label{sec2}
\subsection{Problem Setup}
Let $X\in\mathcal{X} $ denote the feature and $Y \in \mathcal{Y}\subseteq \mathbb{R} $ denote the response variable. The joint distribution of $(X,Y)$ is denoted by $P_{XY}$. Assume access to a calibration dataset $\mathcal{D}_{\text{cal}} = \{(X_i, \tilde{Y}_i)\}_{i=1}^n$ and a set of unlabeled (test) data $\mathcal{D}_{\text{test}} = \{X_{n+j}\}_{j=1}^m$, where $\{\tilde{Y}_i\}_{i=1}^n$ are contaminated responses following $(X,\tilde{Y})\sim P_{X\tilde{Y}}$.
The clean responses $\{Y_{i}\}_{i=1}^{n}$ for the calibration data are unobserved, while the responses $\{Y_{n+j}\}_{j=1}^{m}$ for the test data (also unobserved) are what we aim to infer.
We assume that $\{(X_i,Y_i,\tilde{Y}_i)\}_{i=1}^{n+m}$ are i.i.d. triples across $i\in[n+m]:=\{1,\dots, n+m\}$.
To make inference on $\{Y_{n+j}\}_{j=1}^{m}$, we let $\hat{f}:\mathcal{X}\to \mathcal{Y}$ be any fixed pre-trained predictor of response $Y$ and write $\hat{Y}_i=\hat f(X_{i})$ for $i\in[n+m]$.
This paper focuses on the following two tasks, which cover the main settings considered in the existing conformalized selection literature.

\paragraph{Task 1.} The first task concerns trustworthy classification, which is closely related to several applications such as reliable automatic labeling and classification with fairness. Specifically, let $\mathcal{Y}=\{1,\dots,K\}$. This task is formulated as a multiple testing problem:
\begin{equation}\label{task2}
H_j: \hat{Y}_{n+j}\neq Y_{n+j},\quad j=1,\dots,m.
\end{equation}
The goal is to determine the largest rejection set $\mathcal{R}$ while controlling the overall FDR that
\begin{equation}\label{global_control}
\text{FDR}_{\text{o}}=\mathbb{E}[\text{FDP}_{\text{o}}],~
\text{FDP}_{\text{o}}=\frac{\sum_{j=1}^m \mathbb{I}\{Y_{n+j}\neq \hat{Y}_{n+j},j\in\mathcal{R} \}}{1\vee |\mathcal{R}|}.
\end{equation}
Furthermore, let $\mathcal{A}_1,\dots,\mathcal{A}_G\subset [K]$ be a disjoint partition of $[K]$ and let $\mathcal{R}_{\mathcal{A}_g}$ be the rejection set for hypotheses $\{H_j:\hat Y_{n+j}\in A_g,j\in[m]\}$.
We aim to achieve a stricter target: controlling FDR for different groups simultaneously. That is, given $(\alpha_{\mathcal{A}_1},\dots,\alpha_{\mathcal{A}_{G}})\in(0,1)^{G}$,
\begin{equation}\label{group_fdr}
\text{FDR}_{\mathcal{A}_g}=\mathbb{E}[\text{FDP}_{\mathcal{A}_{g}}]\leq\alpha_{\mathcal{A}_g},~\text{FDP}_{\mathcal{A}_{g}}=\frac{\sum_{j=1}^m\mathbb{I}\{Y_{n+j}\neq \hat{Y}_{n+j},j\in\mathcal{R}_{\mathcal{A}_g}\}}{1\vee |\mathcal{R}_{\mathcal{A}_g}|}, g\in[G].
\end{equation}

\paragraph{Task 2.} 
In applications such as drug detection and LLM alignment,
outcomes with higher values are of interest. This task can be formulated as the multiple testing problem that
\begin{equation}\label{jin_ying}
H_j: Y_{n+j}\leq c_j, \quad j=1,\dots,m,
\end{equation}
where $\{c_j\}_{j=1}^m$ are (random) thresholds.
Letting $\mathcal{R}\subseteq\{1,\dots,m\}$ be the selection set, we aim to determine $\mathcal{R}$ with as large a size as possible while controlling the FDR below the preset level, where 
\begin{equation}\label{task_2_fdr}
\text{FDR}=\mathbb{E}[\text{FDP}],~
\text{FDP}=\frac{\sum_{j=1}^m \mathbb{I}\{Y_{n+j}\leq c_j,j\in\mathcal{R} \}}{1\vee |\mathcal{R}|}.
\end{equation}

The goal of this paper is to reveal the shortcomings of existing methods for addressing the two tasks under contaminated calibration data, and develop a unified and robust alternative.

\subsection{Recap: Methods for Clean Calibration Data}
When calibration data $\mathcal{D}_{\text{cal}}$ contain only clean responses, \cite{jin2023selection} and \cite{sun2025unified} have developed CS (also called \texttt{cfBH}) and PSP for Task 2 and Task 1, respectively. We briefly describe the two methods below.

\paragraph{\texttt{cfBH}.} To test $\{H_{j}\}_{j=1}^m$ \eqref{jin_ying}, \texttt{cfBH} first introduces a \textit{monotone nonconformity score} $V(\cdot,\cdot):\mathcal{X}\times \mathcal{Y}\to \mathbb{R}$. That is, $V(x,y)\leq V(x,y')$ holds for any $x\in \mathcal{X}$ and any $y,y'\in\mathcal{Y}$ satisfying $y\leq y'$, and $V(x,y)$ quantifies how extreme the value $y$ is relative to the conditional behavior of $Y$ given $X=x$ (e.g., $V(x,y)=y-\hat\mu(x)$, where $\hat\mu(x)$ is an estimate of the conditional mean function). \texttt{cfBH} constructs conformal p-values for hypothesis $H_j$ by
\begin{equation}\label{cfbh}
p_j^{\text{cs}}=\frac{1+\sum_{i=1}^{n}\mathbb{I}\{V(X_i,Y_i)<V(X_{n+j},c_j)\}}{n+1},\quad j=1,\dots,m.
\end{equation}
It then runs the BH procedure~\citep{benjamini1995controlling} with $\{p_j^{\text{cs}}\}_{j=1}^m$ given above. That is, $\mathcal{R}=\{j\in[m]:p_j^{\text{cs}}\leq \alpha \hat{l}/m\}$, where $\hat{l}=\max\{l\in[m]:\sum_{j=1}^{m}\mathbb{I}\{p_j^{\text{cs}}\leq\alpha l /m\}\geq l\}$.

\paragraph{\texttt{PSP}.} We describe \texttt{PSP} that controls $\text{FDR}_{\text{o}}$ for clarity. It constructs selective p-values
 \begin{equation*}
p_j^{\text{psp}}=\frac{1+\sum_{i=1}^{n}\mathbb{I}\left\{\hat{Y}_i \neq Y_{i}, \mu^{\text{psp}}_{\hat{Y}_i}(X_i)\geq\mu^{\text{psp}}_{\hat{Y}_{n+j}}(X_{n+j})\right\}}{1+\sum_{i=1}^n \mathbb{I}\{\hat{Y}_i \neq Y_{i}\}},\quad j=1,\dots,m,
 \end{equation*}
where $\mu^{\text{psp}}_{k}$ represents the confidence for class $k$.
Let $p_{(1)}^{\text{psp}}\leq\dots\leq p_{(m)}^{\text{psp}}$ denote the order statistics of $\{p_j^{\text{psp}}\}$. \texttt{PSP} computes $\hat{l}=\max\{l\in[m]:p^{\text{psp}}_{(l)}\leq l\alpha/(m\hat{\theta}_n) \}$, where $\hat{\theta}_n=(1+\sum_{i=1}^{n}\mathbb{I} \{\hat{Y}_i\neq Y_i\})/(1+n)$. Finally, each hypothesis $H_j$ with $p_j^{\text{psp}}\leq p_{\hat{l}}^{\text{psp}}$ is rejected.

The FDR control guarantees of both \texttt{cfBH} and \texttt{PSP} depend on the validity of the constructed ``p-values'' and their dependence. These ideal properties are derived from the i.i.d. (or exchangeable) assumption of $\mathcal{D}_{\text{cal}}\cup\{(X_{n+j},Y_{n+j})\}_{j=1}^m$, which implies that \texttt{cfBH} and \texttt{PSP} may not be reliable when $\{Y_i\}_{i=1}^n$ are contaminated.
This paper aims to investigate the effects of response contamination in calibration dataset on these methods and develops a unified framework for robust conformalized selection;
further discussion of the connections to and distinctions from related work is deferred to Appendix~\ref{related_work}.

\section{Effects of Contamination on Conformalized Selective Inference}\label{sec3}
In this section, we explore the effects of data contamination on \texttt{cfBH}~\citep{jin2023selection} and \texttt{PSP}~\citep{sun2025unified}. Both empirically and theoretically, we show that response contamination in the calibration data can result in a substantial loss of power or compromise FDR control in different cases.

\subsection{Empirical Evidence}

\subsubsection{Evidence on \texttt{cfBH}}\label{s}

We first empirically investigate the effect of response contamination.
Following Setting 8 of \cite{jin2023selection}, we generate i.i.d. covariates $X_i\sim\text{Unif}[-1,1]^{20}$ and responses $Y_i=\mu(X_i)+\epsilon_i$, where $\mu(x)=2(x_1x_2+x_3^2 +e^{x_4-1}-1)$ with $X_i=x:=(x_1,\dots,x_{20})$ and $\epsilon_i|X_i=x\sim N(0,\sigma(x)^{2})$ with $\sigma(x)=0.1\mu(x)^2\mathbb{I}\{|\mu(x)|<2\}+0.2|\mu(x)|\mathbb{I}\{|\mu(x)|\geq1\}$.
We set $n=1000$, $m=100$, and use an independent training sample of size $1000$ to fit the gradient-boosting regressor.
The contaminated calibration responses are generated by $\tilde{Y}_{i}=Y_{i}+\eta_{i}$, where $\eta_i\sim N(\mu_{\text{drift}},0.5^2)$ and $\mu_{\mathrm{drift}}\in\{-0.5,0,0.5\}$.
We compare the specialization of \texttt{cfBH}, named \texttt{BH\_clip}, under each contaminated calibration distribution with its clean-calibration counterpart
over target FDR levels $\alpha\in\{0.05,0.1,0.15,0.2\}$. 

\begin{figure}[ht]
\vskip 0.2in
\centering
    \includegraphics[width=0.8\textwidth]{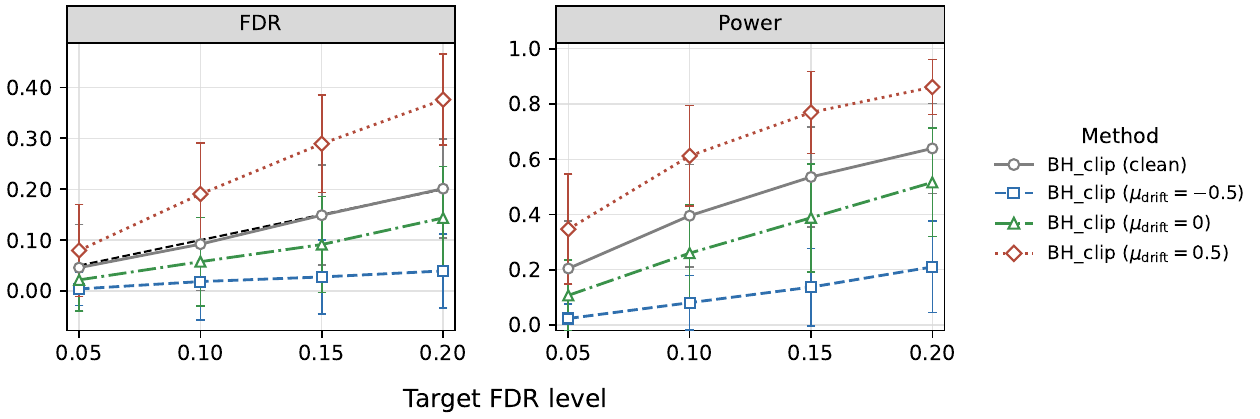}
    \caption{
    Empirical FDR and power of \texttt{BH\_clip} under clean and contaminated calibration responses as the target FDR level varies. The dashed diagonal in the FDR panel represents the nominal target level. Points and error bars denote the mean and one standard deviation over 500 independent repetitions.
    }
    \label{cfBH_noisy}
\end{figure}

Results are presented in Figure~\ref{cfBH_noisy}.
It shows that the additive noise can lead to a severe power loss (small $\mu_{\text{drift}}$) or to a significant expansion of FDR (large $\mu_{\text{drift}}$). Even for $\mu_{\text{drift}}=0$, the power loss of \texttt{BH\_clip} is greater than 20\%.

\subsubsection{Evidence on \texttt{PSP}}

We next examine the performance of \texttt{PSP} under the randomized response model, where the noisy response retains the clean label with probability $1-\epsilon$ and otherwise replaces it by a label sampled uniformly from the $K$ classes.
We fix $K=10$ and leverage a LightGBM~\citep{ke2017lightgbm} to obtain $\{\mu_k^{\text{psp}}\}_{k\in[K]}$ and the base classifier.
The training, calibration, and test sample sizes are set as $n_{\mathrm{tr}}=n=m=100K$. The probabilities $\mathbb{P}(Y=k), k\in[K]$ are given by $(Z_1,\dots,Z_K)/\sum_{k\in[K]}Z_k$, where $Z_k 
\overset{i.i.d}{\sim} \text{Uniform(1,2)}$ and $X_j|Y_j=k \sim N(m_k,I_p)$ with $m_k=2k \times1_{p} /p^{1/4},p=10$. We compare \texttt{PSP} using clean calibration labels with \texttt{PSP} under $\epsilon\in\{0.05,0.10,0.15\}$ over target FDR levels $\alpha\in\{0.05,0.10,0.15,0.20\}$.
Results in Figure~\ref{PSP_noisy} demonstrate that the randomized response contamination makes \texttt{PSP} over conservative, consistently resulting in a significant loss of power.

\begin{figure}[ht]
\vskip 0.2in
\centering
    \includegraphics[width=0.8\textwidth]{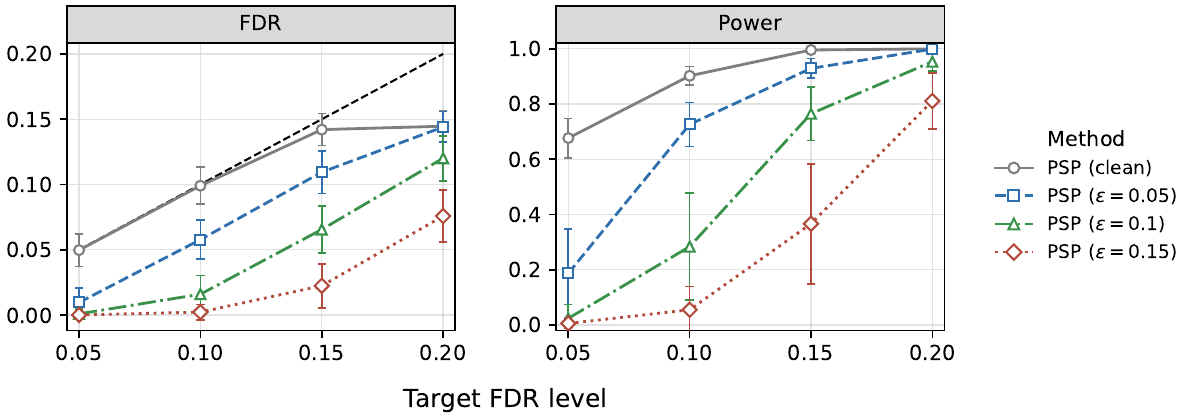}
    \caption{
    Empirical FDR and power of \texttt{PSP} under clean calibration labels and randomized-response contamination as the target FDR level varies.
    All results are averaged over 500 independent repetitions.
    }
    \label{PSP_noisy}
\end{figure}

\subsection{Theoretical Results}
\subsubsection{Theory for \texttt{cfBH}}
We first establish theory for \texttt{cfBH} under \textit{general} response contamination.
Let
\begin{equation}\label{contaminated_p-value}
\tilde{p}_j^{\text{cs}}=\frac{1+\sum_{i=1}^{n}\mathbb{I}\{V(X_i,\tilde{Y}_i)<V(X_{n+j},c_j)\}}{n+1},\quad j=1,\dots,m.
\end{equation}
The following theorem demonstrates that directly applying \texttt{cfBH} to contaminated data, i.e., applying the BH procedure to $\{\tilde{p}_j^{\text{cs}}\}$, can result in an inflated FDR or power loss.

\begin{theorem}\label{cs_jingYIN}
Let $V$ be a fixed monotone nonconformity score and $c_j\equiv c\in\mathcal{Y}$ be a fixed constant, and suppose $\{(X_i,Y_i,\tilde{Y}_i)\}_{i=1}^{n+m}$ are i.i.d. triples.
Denote $V(X,Y)$, $V(X,\tilde{Y})$, $V(X,c)$ by $V^{\circ}$, $\tilde{V}$, and $\hat{V}$, respectively.
Denote the distribution functions of $V^{\circ}$ and $\tilde{V}$ by $F_V$ and $F_{\tilde{V}}$, respectively.
Suppose that $F_{\tilde V}$ is continuous and strictly
increasing on its support, $F_{\tilde V}(\hat{V})$ has a continuous distribution, and $F_V\ll F_{\tilde V}$ with bounded Radon-Nikodym derivative $dF_{V}/d F_{\tilde{V}}$.
For any $\alpha\in(0,1)$, let $\overline{\mathcal{R}}_{\text{cs}}^{*}$ be the output of BH procedure with inputs $(0,F_{\tilde{V}}(V(X_{n+2},c)),\dots,F_{\tilde{V}}(V(X_{n+m},c)))$, whose randomness comes from the distribution of $(X_{n+2},\dots,X_{n+m})$.
Then the output of the BH procedure with inputs $\{\tilde{p}_{j}^{\text{cs}}\}_{j=1}^{m}$~\eqref{contaminated_p-value} satisfies
\begin{equation*}
\limsup_{n\to \infty}\mathrm{FDR}\leq  m\mathbb{E}\left[\frac{1}{|\overline{\mathcal{R}}^*_{\text{cs}}|} F_V \left(F^{-1}_{\tilde{V}}\left(\frac{\alpha |\overline{\mathcal{R}}^*_{\text{cs}}|}{m}\right)\right) \right].
\end{equation*}
\end{theorem}

\begin{remark}
Theorem~\ref{cs_jingYIN} remains true for random $c_j:=W_{n+j}$ as long as $\{(X_{n+j},W_{n+j})\}_{j=1}^m$ are exchangeable.
This condition is usually directly guaranteed by the exchangeability of $\{X_{n+j}\}_{j=1}^m$, because $W_{n+j}$ can typically be represented as $h(X_{n+j})$ for some function $h(\cdot)$.
\end{remark}

If $\tilde{Y}_i=Y_i$ for $i\in[n]$, then $\limsup_{n\to \infty}\text{FDR}\leq \alpha$, which coincides with the finite-sample result of \cite{jin2023selection}. If contamination renders the nonconformity scores stochastically larger, then the \texttt{cfBH} tends to yield an inflated FDR. In contrast, if contamination satisfies $F_{\tilde{V}}(v)> F_{V}(v)$ for $v\leq F_{\tilde{V}}^{-1}(\alpha)$, \texttt{cfBH} maintains FDR control but tends to be conservative. As an example, Proposition~\ref{additive_noise_cs} reveals that \texttt{cfBH} tends to be conservative under the common zero-mean additive noise.

\begin{proposition}\label{additive_noise_cs}
Consider the additive noise that $\tilde{Y}_i=Y_i+\eta_i$, $i\in[n]$, where the i.i.d. noise variables $\{\eta_i\}_{i=1}^n$ are non-degenerate, symmetric around zero, and independent of $\{(X_i,Y_i)\}_{i=1}^{n+m}$.
Assume that $Y \mid X=x$ has a continuous density that is symmetric and strictly unimodal around a conditional mode $m(x)$ and the conditions of Theorem~\ref{cs_jingYIN} hold. Given any monotone nonconformity score $V(\cdot,\cdot)$, for any target FDR level satisfying $0<\alpha< \inf_{x\in \mathcal{X}} F_{\tilde{V}} (V(x,m(x)))$, the \texttt{cfBH} satisfies $\lim\sup_{n\to\infty} \mathrm{FDR}<\alpha$.
\end{proposition}

\subsubsection{Theory for \texttt{PSP}}
Now we establish theory for \texttt{PSP} under general label noise. We write $V_i=V(X_i)=-\mu^{\text{psp}}_{\hat{Y}_i}(X_i)$ for $i\in[n+m]$. Let $E = \mathbb{I}\{\hat{f}(X) \neq Y\}$ and $\tilde{E} = \mathbb{I}\{\hat{f}(X) \neq \tilde{Y}\}$ with $\theta=\mathbb{P}(E)$ and $\tilde{\theta}=\mathbb{P}(\tilde{E})$. Write $\tilde{E}_i=\mathbb{I}\{\hat{f}(X_i)\neq \tilde{Y}_i\}$. \texttt{PSP} is exactly applying the BH procedure to noisy ``p-values'' $(\tilde{p}_1^{\text{psp}},\dots,\tilde{p}_m^{\text{psp}})$ under target FDR level $\alpha/\tilde{\theta}_n$, where
\begin{equation}\label{noisy_psp_inputs}
\tilde{p}_j^{\text{psp}}=\frac{1+\sum_{i=1}^{n}\mathbb{I}\left\{\tilde{E}_i=1, V_i\leq V_{n+j}\right\}}{1+\sum_{i=1}^n \tilde{E}_i},
\quad
\tilde{\theta}_n=\frac{1+\sum_{i=1}^n \tilde{E}_i}{1+n}.
\end{equation}
Theorem~\ref{thm:psp_general} establishes the upper bound of $\text{FDR}_\text{o}$ under general label noise.

\begin{theorem}\label{thm:psp_general}
Suppose $\{(X_i, Y_i, \tilde{Y}_i)\}_{i=1}^{n+m}$ are i.i.d. 
Let $F_{V|E=1}(v) = \mathbb{P}(V(X) \le v \mid E=1)$ and $F_{V|\tilde{E}=1}(v) = \mathbb{P}(V(X) \le v \mid \tilde{E}=1)$. Assume that $F_{V|E=1}(\cdot)$ and $F_{V|\tilde{E}=1}$ are continuous, $F_{V|\tilde{E}=1}$ is strictly increasing on its support, and $F_{V|\tilde{E}=1}(V(X))$ has a continuous distribution.
For any fixed $\alpha\in(0,\tilde{\theta})$, let $\bar{\mathcal{R}}_{psp}$ be the rejection set of BH procedure with inputs $(0,F_{V|\tilde{E}=1}(V_{n+2}),\dots,F_{V|\tilde{E}=1}(V_{n+m}))$ at the adjusted level $\tilde{\alpha} = \alpha / \tilde{\theta}$.
Then \texttt{PSP} with noisy inputs $(\tilde{p}_1^{\text{psp}},\dots,\tilde{p}_m^{\text{psp}})$ with $\tilde{\theta}_n$ defined in \eqref{noisy_psp_inputs} satisfies
\begin{equation*}
\limsup_{n\to\infty}\mathrm{FDR}_{\mathrm{o}}
\leq m\theta\mathbb{E}\left[ \frac{1}{|\bar{\mathcal{R}}_{\text{psp}}|}F_{V|E=1}\left(F_{V|\tilde{E}=1}^{-1}\left(\frac{\alpha|\bar{\mathcal{R}}_{\text{psp}}|}{m\tilde{\theta}}\right)\right)\right].
\end{equation*}
\end{theorem}

\begin{remark}
The precision of the base classifier $\hat{f}$ is measured by $1-\theta$. It is reasonable for $\tilde{\theta}\geq \theta$ since $\hat{f}$ should not exhibit better accuracy on noisy targets. For the case $\alpha\geq \tilde{\theta}\geq \theta$, it is trivial to control $\text{FDR}_{\text{o}}$ asymptotically by simply rejecting $H_j$ for all $j\in[m]$.
For this reason, we only need to consider the non-trivial case that $\alpha <\tilde{\theta}$ in Theorem \ref{thm:psp_general}. 
\end{remark}

For clean calibration data, we have $\tilde{\theta}=\theta$ and $E=\tilde{E}$, which implies $\lim\sup_{n\to\infty}\text{FDR}_{\text{o}}\leq \alpha$. For common contamination such as random label flipping, \texttt{PSP} tends to be conservative.

\begin{proposition}\label{random_flip_psp}
Assume that the conditions of Theorem~\ref{thm:psp_general} hold.
Consider binary classification under random flip label noise with noise rate $\epsilon\in(0,1/2)$, and let the classifier $\hat{f}$ be better than random guessing ($\theta<1/2$). For $\alpha\in(0,\tilde{\theta})$, assume that $F_{V|E=1}(v)< F_{V|E=0}(v)$ for $v\leq F^{-1}_{V|\tilde{E}=1}(\alpha/\tilde{\theta})$. Then we have $\theta < \tilde{\theta}$ and $\limsup_{n\to\infty}\mathrm{FDR}_{\text{o}}<\alpha\theta/\tilde{\theta}$.
\end{proposition}

We note that the intuitive idea of the assumption $F_{V|E=1}(v) <F_{V|E=0}(v)$ is that the pre-trained model tends to be more confident (i.e., assigning lower nonconformity scores) when it makes a correct prediction than when it makes a mistake. By Proposition~\ref{random_flip_psp}, PSP can suffer severe power loss under random flip label noise if $\theta/\tilde{\theta}$ is relatively small. 
Similarly, while certain asymmetric noise patterns may induce anti-conservative behavior, PSP is prone to over-conservativeness under common label noise mechanisms, since the label noise usually results in a small factor $\theta/\tilde{\theta}$.

\section{Robust Conformalized Selection}\label{sec4}
This section introduces the RCS approach for the classification problem (Task 1). Section~\ref{methodology} describes the methodology with given localized weights. Section~\ref{theory_rcs} establishes the asymptotic FDR control and the power optimality of RCS, and Section \ref{estimate_w} demonstrates the robustness of RCS under estimated localized weights.

\subsection{Methodology}\label{methodology}
Recall that we have a pre-trained classifier $\hat{f}:\mathcal{X}\to[K]$ that provides the predicted label $\hat{Y}_i=\hat{f}(X_i)$ for $i\in[n+m]$.
Let $U: \mathcal{X} \to [0,1]$ be a scoring function that measures the uncertainty (nonconformity) of the predicted class.
Write $U_i=U(X_i)$, $i\in[n+m]$. 
For example, the pre-trained classifier may take the form that
\begin{equation}\label{classifier}
\hat{f}(x)=\arg\max_{k\in[K]}\hat{\mu}_k(x),
\end{equation}
where $\hat{\mu}_k(x)$ is the estimator of $\mu_k(x):=\mathbb{P}(Y=k|X=x)$ for $k\in[K]$. In this case, the uncertainty score of the prediction $\hat{Y}_i$ can be constructed by
\begin{equation}\label{uncertainty_Score}
U_i=1-\hat{\mu}_{\hat{Y}_i}(X_i) =1-\max_k \hat{\mu}_k(X_i).
\end{equation}
For $j\in[m]$, a small value of $U_{n+j}$ reflects a higher confidence in the correctness of the prediction $\hat{Y}_{n+j}$.
In line with intuition, the confidence level should be positively correlated with the likelihood of $\hat{Y}_{n+j}=Y_{n+j}$.
Therefore, we reject hypothesis $H_j: \hat{Y}_{n+j}\neq Y_{n+j}$ if $U_{n+j}$ is below some threshold.

\begin{remark}[Pre-trained classifier]\label{Pre-trained_classifier}
The classifier function $\hat{f}$ can be trained using any (black-box) model, as long as it is independent of the calibration and test datasets. Therefore, the classifier need not to be based on a soft-classifier \eqref{classifier}. Furthermore, the RCS approach is also applicable to classifiers with inherent randomness (e.g., LLM as a classifier) as long as the randomness is independent of the data $\{(X_i,Y_i)\}_{i=1}^{n+m}$.
\end{remark}

\begin{remark}[Uncertainty score]
Practitioners can utilize various uncertainty scores since the FDR control of RCS holds for a broad class of scores; see Theorem \ref{theorem2}.
However, as previously analyzed, the uncertainty score should reflect evidence against $H_{j}$ to gain power. Theorem~\ref{optimal_uncertainty} establishes the optimal choice of uncertainty score that maximizes power.
\end{remark}

To determine the rejection set $\mathcal{R}_{\mathcal{A}_{g}}$ that satisfies $\text{FDR}_{\mathcal{A}_g}\leq\alpha_{\mathcal{A}_g}$, $g\in[G]$, the selection algorithm should explicitly or implicitly provide an accurate estimate of $\text{FDP}_{\mathcal{A}_g}$ with given thresholds.
For clean calibration data that are i.i.d. with the test data, a useful strategy is to directly estimate this rate from the calibration data. However, such an estimate can be heavily biased for the label contamination setting. The key idea of RCS lies in translating this problem into a \textit{localized covariate shift} problem, and subsequently employing a \textit{covariate-adjusted} empirical-Bayes type estimate of FDP. In what follows, we describe the details of threshold determination.

We first partition the calibration data and the test data based on their predicted classes:
\begin{equation*}
\mathcal{S}_{k}^{\text{cal}}:=\{i\in[n]:\hat{Y}_{i}=k\},\quad
\mathcal{S}_{k}^{\text{test}}:=\{j\in[m]:\hat{Y}_{n+j}=k\},\quad k\in[K].
\end{equation*}
We define the localized weights by
\begin{equation}\label{localized_weights}
w_{k}(x):=\frac{\mathbb{P}(Y\neq k|X=x)}{\mathbb{P}(\tilde{Y}\neq k|X=x)},\quad k\in[K].
\end{equation}
The estimation of $w_k(x)$ is deferred to Section~\ref{estimate_w}.
For each group $\mathcal{A}_g$, RCS uses a shared threshold $t$ for all predicted classes in the group. We estimate $\text{FDP}_{\mathcal{A}_g}(t)$ by
\begin{equation}\label{estimate_fdp}
\widehat{\text{FDP}}_{\mathcal{A}_g}(t)=
\frac{\sum_{k\in\mathcal{A}_g}\frac{|\mathcal{S}_k^{\text{test}}|}{|\mathcal{S}_k^{\text{cal}}|}\sum_{i\in \mathcal{S}_{k}^{\text{cal}}}w_k(X_i)\mathbb{I}\{ \tilde{Y}_i \neq k,~U_i\leq t\}}{1\vee\sum_{k\in\mathcal{A}_g}\sum_{j\in\mathcal{S}_k^{\text{test}}}\mathbb{I}\{U_{n+j}\leq t\}}.
\end{equation}
The final threshold is determined by 
\begin{align}\label{threshold}
\hat{t}_{\mathcal{A}_g} =\max \left\{ t\in[0,1]:\widehat{\text{FDP}}_{\mathcal{A}_g}(t)\leq \alpha_{\mathcal{A}_g}\right\}.
\end{align}
Then the selection set for group $\mathcal{A}_g$ is given by $\mathcal{R}_{\mathcal{A}_g}=\{j\in[m]: \hat{Y}_{n+j}\in\mathcal{A}_g, U_{n+j} \leq \hat{t}_{\mathcal{A}_g}\}$. The explanation of $\widehat{\text{FDP}}_{\mathcal{A}_g}(t)$ is provided as follows.

\begin{remark}[Insight of $\widehat{\text{FDP}}_{\mathcal{A}_g}(t)$]\label{bayes_fdp}
Leveraging a mirror process \citep{barber2015controlling,du2023false} or adopting an empirical-Bayes perspective \citep{ramdas2019unified,yu2024generalized} are currently two powerful paradigms for estimating FDP. 
Unlike existing works, the proposed $\widehat{\text{FDP}}_{\mathcal{A}_g}(t)$ can be regarded as a covariate-adjusted empirical-Bayes-type estimate of the FDP. Specifically, with $w_k(x)$, we have 
\begin{align*}
 \mathbb{P}\left(\hat{Y}=k, Y \neq k, U(X) \le t\right)
 &=
\mathbb{E} \left[ w_k(X) \mathbb{I}\{\hat{Y}=k,\tilde{Y} \neq k, U(X) \le t\}\right]\\
&\approx\frac{1}{n}\sum_{i\in\mathcal{S}_k^{\text{cal}}}w_k(X_i)\mathbb{I}\{\tilde{Y}_i\neq k, U_i\leq t\}.
\end{align*}
Combining with the fact that $|\mathcal{S}_k^{\text{cal}}|/n\approx\mathbb{P}(\hat{Y}=k)$, we have
\begin{align*}
\frac{1}{|\mathcal{S}_k^{\text{cal}}|}\sum_{i\in\mathcal{S}_k^{\text{cal}}}w_k(X_i)\mathbb{I}\{\tilde{Y}_i\neq k, U_i\leq t\}
\approx \mathbb{P}(Y\neq k,U(X)\leq t\mid\hat{Y}=k).
\end{align*}
This leads to a reasonable estimate of the number of false rejections:
\begin{align*}
\sum_{k\in\mathcal{A}_g}\#\{j\in\mathcal{S}_k^{\text{test}}:U_{n+j}\leq t,Y_{n+j}\neq k\}
\approx
\sum_{k\in\mathcal{A}_g}|\mathcal{S}_k^{\text{test}}|\cdot\frac{1}{|\mathcal{S}_k^{\text{cal}}|}\sum_{i\in\mathcal{S}_k^{\text{cal}}}w_k(X_i)\mathbb{I}\{\tilde{Y}_i\neq k, U_i\leq t\}.
\end{align*}
\end{remark}

\subsection{FDR Control and Optimality}\label{theory_rcs}
In this section, we establish the FDR control and power optimality for general label contamination models.
Define the oracle feasible set for the threshold by
\begin{equation}\label{oracle_threshold}
\mathcal{F}_{\mathcal{A}_g} = \left\{t \in [0,1]: R_{\mathcal{A}_g}^{\infty}(t)>0,~ \frac{V_{\mathcal{A}_g}^{\infty}(t)}{R_{\mathcal{A}_g}^{\infty}(t)}\le \alpha_{\mathcal{A}_g}\right\},
\end{equation}
where
\begin{equation*}
R_{\mathcal{A}_g}^{\infty}(t):=\sum_{k\in\mathcal{A}_g}\mathbb{P}(\hat{Y}=k,U\leq t),\quad
V_{\mathcal{A}_g}^{\infty}(t):=\sum_{k\in\mathcal{A}_g}\mathbb{P}(\hat{Y}=k,Y\neq k,U\leq t).
\end{equation*}
Let $R_{\mathcal{A}_g}^*:=\sup_{t\in\mathcal{F}_{\mathcal{A}_g}}R_{\mathcal{A}_g}^{\infty}(t)$.
We define the \textit{recovery-type power} of RCS for group $\mathcal{A}_g$ by
\begin{equation}\label{power_1}
\text{rPower}_{\mathcal{A}_g}:=\mathbb{E}\left[\frac{\sum_{k\in\mathcal{A}_g}\sum_{j=1}^{m}\mathbb{I}\{\hat{Y}_{n+j}=Y_{n+j}=k,j\in\mathcal{R}_{\mathcal{A}_g}\}}{1\vee\sum_{j=1}^{m} \mathbb{I}\{Y_{n+j}\in\mathcal{A}_g\}}\right].
\end{equation}
The asymptotic power under threshold $t$ is defined by
\begin{align*}\label{power_2}
\text{rPower}^{\infty}_{\mathcal{A}_g}(t):=\frac{\sum_{k\in\mathcal{A}_g}\mathbb{P}(\hat{Y}=Y=k,U\leq t)}{\mathbb{P}(Y\in\mathcal{A}_g)}.
\end{align*}
Before presenting the theoretical results, we impose a mild regularity condition that rules out atoms at the threshold and excludes the degenerate case in which no nontrivial strictly feasible selection rule exists.

\begin{assumption}\label{techniqe_condition}
Consider any fixed $g\in[G]$. (1) $\mathbb{P}(Y\in\mathcal{A}_g)>0$ and $\mathbb{P}(\hat{Y}=k)>0$ for $k\in\mathcal{A}_g$; (2) For every $k\in[K]$ and $l\in\mathcal{A}_g$ with $\mathbb{P}(Y=k,\hat{Y}=l)>0$, the conditional distribution function $G_{k,l}(t):=\mathbb{P}(U\leq t\mid Y=k,\hat{Y}=l)$ is continuous on $[0,1]$. (3) The oracle feasible frontier is non-degenerate and strictly approximable, i.e., 
$R_{\mathcal{A}_g}^*>0$ and, for every $\delta>0$, there exists $t_\delta\in[0,1]$ such that
$V_{\mathcal{A}_g}^{\infty}(t_\delta)<\alpha_{\mathcal{A}_g}R_{\mathcal{A}_g}^{\infty}(t_\delta)$ and
$R_{\mathcal{A}_g}^{\infty}(t_\delta)\geq R_{\mathcal{A}_g}^*-\delta$.
\end{assumption}

\begin{theorem}\label{theorem2}
Consider i.i.d. samples $\{(X_i,Y_i,\tilde{Y}_i)\}_{i=1}^{n+m}$, any fixed classifier $\hat{f}:\mathcal{X}\to[K]$, and any fixed score function $U:\mathcal{X}\to[0,1]$.
Suppose that the localized weight $w_k(x)$ in \eqref{localized_weights} is well-defined and uniformly bounded for $k\in[K]$.
For each $g\in[G]$, let $\mathcal{R}_{\mathcal{A}_g}=\{j\in[m]: \hat{Y}_{n+j}\in\mathcal{A}_g, U_{n+j} \leq \hat{t}_{\mathcal{A}_g}\}$, where the threshold $\hat{t}_{\mathcal{A}_g}$ is defined by \eqref{threshold}. Let $\mathrm{FDR}_{\mathcal{A}_g}$ and $\mathrm{rPower}_{\mathcal{A}_g}$ be defined by \eqref{group_fdr} and \eqref{power_1}, respectively.
Under Assumption \ref{techniqe_condition}, we have 
\begin{itemize}
\item[(1)] $\limsup_{n,m \to \infty} \mathrm{FDR}_{\mathcal{A}_g} \le \alpha_{\mathcal{A}_g}$.
\item[(2)] $\lim_{n,m \to \infty} \mathrm{rPower}_{\mathcal{A}_g} = \sup_{\tau \in \mathcal{F}_{\mathcal{A}_g}} \mathrm{rPower}_{\mathcal{A}_g}^\infty(\tau)$.
\end{itemize}
\end{theorem}

Theorem~\ref{theorem2} demonstrates that, given the uncertainty score function $U(\cdot)$, RCS controls $\text{FDR}_{\mathcal{A}_g}$ and achieves the optimal recovery power for any $g\in[G]$ asymptotically.
We note that the metric $\mathrm{rPower}_{\mathcal{A}_g}$ is somewhat different from the classical definition of power in multiple testing, although RCS also achieves asymptotic optimality under the latter definition; see Remark~\ref{remark:power} for a detailed discussion.

{
\begin{remark}[Relationship with power in multiple testing]\label{remark:power}
In classical multiple testing, power is defined as the proportion of false null hypotheses that are rejected, i.e.,
\begin{equation*}
\mathrm{Power}_{\mathcal{A}_g}:=
\mathbb{E}\left[\frac{
\sum_{k\in\mathcal{A}_g}\sum_{j=1}^{m}
\mathbb{I}\{\hat{Y}_{n+j}=Y_{n+j}=k,\ U_{n+j}\leq \hat{t}_{\mathcal{A}_g}\}
}{
1\vee\sum_{j=1}^{m}
\mathbb{I}\{\hat{Y}_{n+j}=Y_{n+j}\in\mathcal{A}_g\}
}
\right].
\end{equation*}
The only difference between $\mathrm{Power}_{\mathcal{A}_g}$ and $\mathrm{rPower}_{\mathcal{A}_g}$ is the denominator within the expectation. Since both denominators are independent of the thresholding rule, maximizing the two notions asymptotically leads to the same oracle threshold whenever the classical power is non-degenerate. 
That is, the proposed RCS also achieves the optimal power in terms of classical multiple testing.
The metric $\mathrm{Power}_{\mathcal{A}_g}$ reflects the extent to which the testing procedure can preserve the correct predictions made by the pre-trained classifier $\hat f$, but it cannot directly reflect the utility of the test dataset and is not comparable across different $\hat f$. In contrast, $\mathrm{rPower}_{\mathcal{A}_g}$ is a more direct metric for practitioners, which measures the final recovery utility among test samples whose true label falls into the group $\mathcal{A}_g$.
\end{remark}

}

We then discuss the optimal choice of the score function $U$.
For a fixed group $\mathcal{A}_g$, let $\mathcal{U}$ be the collection of all measurable scores $U:\mathcal{X}\to[0,1]$ for which Assumption~\ref{techniqe_condition} holds for this group.
We rewrite $\mathcal{F}_{\mathcal{A}_g}$, $\mathrm{rPower}_{\mathcal{A}_g}$, and $\mathrm{rPower}^{\infty}_{\mathcal{A}_g}(\tau)$ by $\mathcal{F}_{\mathcal{A}_g}(U)$, $\mathrm{rPower}_{\mathcal{A}_g}(U)$, and $\mathrm{rPower}_{\mathcal{A}_g}^{\infty}(\tau;U)$, respectively.
The following theorem establishes the optimal choice of uncertainty score for the proposed RCS.

\begin{theorem}\label{optimal_uncertainty}
Fix $g\in[G]$ and let
$U^*(X)=1-\mathbb{P}(Y=\hat f(X)\mid X)$.
Suppose $U^*\in\mathcal U$. Then
\begin{equation*}\label{eq:oracle_optimal_score}
\sup_{\tau\in \mathcal{F}_{\mathcal{A}_g}(U^*)}
\mathrm{rPower}_{\mathcal{A}_g}^{\infty}(\tau;U^*)
=
\sup_{U\in\mathcal U}\sup_{\tau\in \mathcal{F}_{\mathcal{A}_g}(U)}
\mathrm{rPower}_{\mathcal{A}_g}^{\infty}(\tau;U).
\end{equation*}
Consequently, for i.i.d. samples $\{(X_i,Y_i,\tilde{Y}_i)\}_{i=1}^{n+m}$, if the localized weights $w_k(x)$ in \eqref{localized_weights} are well-defined and uniformly bounded, the RCS procedure implemented with $U^*$ satisfies
\begin{equation*}
\lim_{n,m\to \infty}\mathrm{rPower}_{\mathcal{A}_g}(U^*)=
\sup_{U\in\mathcal U}\sup_{\tau\in \mathcal{F}_{\mathcal{A}_g}(U)}
\mathrm{rPower}_{\mathcal{A}_g}^{\infty}(\tau;U).
\end{equation*}
\end{theorem}

\subsection{Robustness and a Strategy for Estimating Localized Weights}\label{estimate_w}

\begin{algorithm}[tb]
   \caption{RCS for Class-conditional Contamination Model}
   \label{alg:RCS}
\begin{algorithmic}[1]
   \STATE {\bfseries Input:} Pre-trained classifier $\hat{f}:\mathcal{X}\to[K]$, calibration data $\{(X_i,\tilde{Y}_i)\}_{i=1}^n$, test data $\{X_{n+j}\}_{j=1}^m$, pre-trained estimators $(\hat{\mu}_{1}(x),\dots,\hat{\mu}_{K}(x))$ of $(\mathbb{P}(Y=1|X=x),\dots,\mathbb{P}(Y=K|X=x))$, uncertainty score function $U:\mathcal{X}\to [0,1]$, transition matrix $\mathbf{T}\in\mathbb{R}^{K\times K}$, partition groups $\mathcal{A}_1,\dots,\mathcal{A}_G$, FDR target $(\alpha_{\mathcal{A}_1},\dots,\alpha_{\mathcal{A}_G})\in(0,1)^{G}$.
   
   \STATE {\bfseries Output:} Set of selected candidates $\mathcal{R}_{\mathcal{A}_g}$, $g\in[G]$.\\
   
   \STATE Compute the predicted label $\hat{Y}_{i}=\hat{f}(X_i)$ and its uncertainty score $U_i$, $i\in[n+m]$.
   \STATE For each $k\in [K]$, identify the subsets predicted as class $k$ as follows: 
   \begin{align*}
   &\mathcal{S}^{\text{test}}_k =\{j\in[m]:\hat{Y}_{n+j}=k\},\quad\mathcal{S}_{k}^{\text{cal}}=\{i\in[n]:\hat{Y}_i=k\}.
   \end{align*}
   
\STATE  For each $k\in[K]$, compute the estimated localized weights 
\begin{equation*}
\hat{w}_{k}(X_i)=\frac{1-\hat{\mu}_k(X_i)}{1-\sum_{l=1}^K T_{lk} \hat{\mu}_l (X_i)}, \quad i\in   \mathcal{S}^{\text{cal}}_k.
\end{equation*}

\FOR{$g=1,\dots, G$} 

\STATE Calculate the shared threshold $\hat{t}_{\mathcal{A}_g}$ as
\begin{align*}
\hat{t}_{\mathcal{A}_g} =\max\{t\in \mathcal{T}:  \widehat{\text{FDP}}_{\mathcal{A}_g}(t)\leq \alpha_{\mathcal{A}_g}\},
\end{align*}
where $\mathcal{T}=\cup_{k\in\mathcal{A}_g} (\{U_{i}:i\in \mathcal{S}_k^{\text{cal}}\}\cup \{U_{n+j}:j\in\mathcal{S}_k^{\text{test}}\})$ and
\begin{equation*}
\widehat{\text{FDP}}_{\mathcal{A}_g}(t)=
\frac{\sum_{k\in\mathcal{A}_g}\frac{|\mathcal{S}_k^{\text{test}}|}{|\mathcal{S}_k^{\text{cal}}|}\sum_{i\in \mathcal{S}_{k}^{\text{cal}}}\hat{w}_k(X_i)\mathbb{I}\{ \tilde{Y}_i \neq k,~U_i\leq t\}}{1\vee\sum_{k\in\mathcal{A}_g}\sum_{j\in\mathcal{S}_k^{\text{test}}}\mathbb{I}\{U_{n+j}\leq t\}}.
\end{equation*}
\STATE Compute $\mathcal{R}_{\mathcal{A}_g}=\left\{j\in[m]: \hat{Y}_{n+j}\in\mathcal{A}_g,U_{n+j}\leq \hat{t}_{\mathcal{A}_g}\right\}$.
\ENDFOR 
   \STATE {\bfseries Return} $\mathcal{R}_{\mathcal{A}_g}$, $g\in[G]$.
\end{algorithmic}
\end{algorithm}

We complete the description of RCS by establishing its robustness to consistently estimated localized weights $\{\hat{w}_k(x)\}$ and further providing a convenient strategy for estimating $w_k(x)$ under the widely used class-conditional contamination model.
The following theorem demonstrates that the proposed RCS remains robust under any consistent estimators $\{\hat{w}_k(\cdot)\}_{k\in[K]}$, without the need to specify a concrete label contamination model. 

\begin{theorem}\label{robust_rcs}
Consider the same setting as Theorem~\ref{theorem2}. Let $\hat w_{n,m,k}:\mathcal{X}	\to \mathbb{R}$ be the estimated localized weight function for class $k$ trained independently of $\{(X_i,Y_i,\tilde{Y}_i)\}_{i=1}^{n+m}$ and satisfying $ 
\max_{k\in[K]}\|\hat w_{n,m,k}(\cdot)-w_k(\cdot)\|_{L_2(P_X)}=o_p(1)$ as $n,m\to\infty$.
For $g\in[G]$, let $\hat t^{\mathrm{est}}_{\mathcal A_g}$ be the threshold determined by \eqref{threshold} using $\hat w_{n,m,k}(\cdot)$ in place of $w_k(\cdot)$, and let $\mathrm{FDR}_{\mathcal A_g}$ be the corresponding group FDR. Under Assumption \ref{techniqe_condition}, we have
$\limsup_{n,m\to\infty}\mathrm{FDR}_{\mathcal A_g}\le \alpha_{\mathcal A_g}$.
\end{theorem}

\paragraph*{Estimation of $w_k(x)$ under class-conditional contamination model.}
We now provide a convenient strategy to estimate $w_k(x)$ for the widely studied class-conditional label contamination model, where the noisy label is generated according to
\begin{equation}\label{eq:class_conditional_noise}
\tilde Y \perp X\mid Y,
\qquad
T_{jk}:=\mathbb P(\tilde Y=k\mid Y=j),\qquad j,k\in[K].
\end{equation}
We note that the class-conditional label contamination model encompasses various types of label noise~\citep{ghosh2017robust,natarajan2013learning,sesia2025adaptive}. For example, it includes the widely studied randomized response model~\citep{einbinder2024label,penso2024conformal,xi2025exploring}, under which the true label $Y_i$ is replaced, with probability $\epsilon\in(0,1)$, by a label $\tilde{Y}_i$ drawn uniformly from the $K$ classes. 
Let $\mu_k(x)=\mathbb P(Y=k\mid X=x)$. The following proposition links the noisy class probabilities, the clean class probabilities, and the transition matrix $\mathbf T=(T_{jk})$.

\begin{proposition}\label{prop:estimated_weights}
Consider the class-conditional label contamination model described in~\eqref{eq:class_conditional_noise}. For each $x\in\mathcal X$, it holds that
\[
\bigl(\mathbb P(\tilde Y=1\mid X=x),\dots,
\mathbb P(\tilde Y=K\mid X=x)\bigr)
=
\bigl(\mu_1(x),\dots,\mu_K(x)\bigr)\mathbf T.
\]
Consequently, we have $w_k(x)
=(1-\mu_k(x))/(1-\sum_{j=1}^K T_{jk}\mu_j(x))$, $k\in[K]$.
\end{proposition}

By Proposition~\ref{prop:estimated_weights}, a plug-in estimate of $w_k(x)$ is given by
\begin{equation}\label{localized_weight}
\hat{w}_k(x)=\frac{1-\hat{\mu}_k(x)}{1-\sum_{j=1}^{K} \hat{T}_{jk} \hat{\mu}_j(x)},
\end{equation}
where $\hat{\mu}_k(x)$ is an estimate of $\mu_k(x)$ and $\hat{T}_{ij}$ is an estimate of $T_{ij}$.
If a soft-classifier is trained on dataset with clean labels, its output can be used directly as $\hat{\mu}_k(x)$. If instead the soft-classifier is trained on dataset with noisy labels, its output, denoted by $\hat{\mu}^{\text{noisy}}_k(x)$, is the estimate of $\mathbb{P}(\tilde{Y}=k \mid X=x)$ rather than $\mu_k(x)$.
In this case, if $\mathbf T$ is non-singular, then by Proposition~\ref{prop:estimated_weights}, we have
\begin{equation*}
(\mu_1(x),\dots,\mu_K(x))=(\mathbb{P}(\tilde{Y}=1 \mid X=x),\dots, \mathbb{P}(\tilde{Y}=K \mid X=x))\mathbf{T}^{-1}.
\end{equation*}
This implies that we can take $(\hat{\mu}^{\text{noisy}}_1(x),\dots,\hat{\mu}^{\text{noisy}}_K(x))\hat{\mathbf{T}}^{-1}$ as an estimate of $(\mu_1(x),\dots,\mu_K(x))$.
By plugging $(\hat{\mu}_1(x),\dots,\hat{\mu}_K(x))=(\hat{\mu}^{\text{noisy}}_1(x),\dots,\hat{\mu}^{\text{noisy}}_K(x))\hat{\mathbf{T}}^{-1}$ into \eqref{localized_weight}, we have
\begin{equation}\label{calibrated_localweights}
\hat{w}_k(x)=\frac{1-\sum_{j=1}^{K}\hat{\mu}^{\text{noisy}}_j(x)(\hat{\mathbf{T}}^{-1})_{jk}}{1-\hat{\mu}^{\text{noisy}}_k(x)}.
\end{equation}
The RCS method with estimated weights is summarized in Algorithm~\ref{alg:RCS}.

The estimation of the transition matrix $\mathbf{T}$ is a well-studied problem in learning with noisy labels~\citep{yao2020dual,li2022estimating,zhu2022beyond}. 
In applications such as differential privacy, the matrix $\mathbf{T}$ is known by design~\citep{ghazi2021deep}.
In other applications, the matrix $\mathbf T$ can often be estimated from a small dataset containing samples with clean labels, by assuming the structural form of the contamination model. For example, the randomized response model and its variants only require estimating a small number of noise parameters~\citep{sesia2025adaptive}.
In Appendix~\ref{simu:estimate}, we also provide convenient strategies for estimating $\mathbf{T}$, and numerical results in Section \ref{sec6} further demonstrate the robustness of RCS to estimated or even mis-specified $\mathbf{T}$.

\section{Specializations}\label{sec5}
In this section, we develop specializations of RCS for specific contamination models and for Task 2. Section~\ref{sec:RCS_uniform_noise} focuses on the randomized response model. Section~\ref{sec:cf} provides a robust alternative of \texttt{cfBH} by leveraging RCS, thus tackling Task 2.

\subsection{Selection under Randomized Response Model.}\label{sec:RCS_uniform_noise}

An important contamination type is the randomized response model, which is widely used in differential privacy~\citep{ghazi2021deep}.
The noisy labels are generated by
\begin{equation}\label{randomized_model}
\tilde{Y}_i=Y_i\cdot \mathbb{I}\{\xi_i\geq \epsilon \}+\sum_{k=1}^K k\cdot\mathbb{I}\left\{ \frac{k-1}{K} \epsilon\leq \xi_i < \frac{k}{K}\epsilon\right\}, \quad i\in[n],
\end{equation}
where $\xi_i\overset{i.i.d.}{\sim}\text{Uniform}(0,1)$ and is independent of the calibration and test data.
For this model, a more concise form of RCS can be derived without the need of $w_k(x)$.

Specifically, we denote the empirical-Bayes type estimate of FDP without covariate adjustment by
\begin{equation*}\label{estimate_fdp_no_weighted}
\widetilde{\text{FDP}}_{\mathcal{A}_g}(t):=
\frac{\sum_{k\in\mathcal{A}_g}\frac{|\mathcal{S}_k^{\text{test}}|}{|\mathcal{S}_k^{\text{cal}}|}\sum_{i\in \mathcal{S}_{k}^{\text{cal}}}\mathbb{I}\{ \tilde{Y}_i \neq k,~U_i\leq t\}}{1\vee\sum_{k\in\mathcal{A}_g}\sum_{j\in\mathcal{S}_k^{\text{test}}}\mathbb{I}\{U_{n+j}\leq t\}}.
\end{equation*}
The key insight is that, for the randomized response model with noise rate $\epsilon$,  $\widetilde{\mathrm{FDP}}_{\mathcal{A}_g}(t)$ provides a reasonable estimate of $(1-\epsilon)\text{FDP}_{\mathcal{A}_g}(t)+\epsilon\left(1-1/K\right)$. Therefore, to ensure $\mathrm{FDR}_{\mathcal{A}_g}\leq \alpha_{\mathcal{A}_g}$, we determine the threshold by
\begin{align}\label{uniform_noise_threshold}
\tilde{t}_{\mathcal{A}_g} =
\underset{t\in[0,1]}{\arg\max}\sum_{k\in\mathcal{A}_g}\sum_{j\in\mathcal{S}_k^{\text{test}}}\mathbb{I}\{U_{n+j}\leq t\},
\quad\text{s.t. } \widetilde{\text{FDP}}_{\mathcal{A}_g}(t)\leq\alpha_{\mathcal{A}_g}+\epsilon(1-\alpha_{\mathcal{A}_g}-1/K).
\end{align}
The selection set is determined by $\mathcal{R}_{\mathcal{A}_g}=\{j\in[m]:\hat{Y}_{n+j}\in\mathcal{A}_g,U_{n+j}\leq \tilde{t}_{\mathcal{A}_g}\}$.
The following proposition demonstrates the validity and the optimality of this specialization under the randomized response model.
\begin{proposition}\label{FDR_uniform_label}
Consider the randomized response model with noise rate $\epsilon\in(0,1)$. 
For each $g\in[G]$, let $\mathcal{R}_{\mathcal{A}_g}=\{j\in[m]: \hat{Y}_{n+j}\in\mathcal{A}_g, U_{n+j} \leq \tilde{t}_{\mathcal{A}_g}\}$, where the threshold $\tilde{t}_{\mathcal{A}_g}$ is defined by \eqref{uniform_noise_threshold}.
Let $\mathrm{FDR}_{\mathcal{A}_g}$ and $\mathrm{rPower}_{\mathcal{A}_g}$ be defined by \eqref{group_fdr} and \eqref{power_1}, respectively.
Under the same conditions of Theorem~\ref{theorem2}, we have
\begin{itemize}
\item[(1)] $\limsup_{n,m \to \infty} \mathrm{FDR}_{\mathcal{A}_g} \le \alpha_{\mathcal{A}_g}$.
\item[(2)] $\lim_{n,m \to \infty} \mathrm{rPower}_{\mathcal{A}_g} = \sup_{\tau \in \mathcal{F}_{\mathcal{A}_g}} \mathrm{rPower}_{\mathcal{A}_g}^\infty(\tau)$.
\end{itemize}
\end{proposition}

\subsection{A Robust Alternative to \texttt{cfBH}}\label{sec:cf}
In this section, we develop a robust alternative of \texttt{cfBH} by reducing Task 2 to a binary classification instance of RCS.
As described in \cite{jin2023selection}, $c_j$ is usually either a fixed constant for $j\in[m]$ or a random variable such that $\{(X_{n+j},Y_{n+j},c_j)\}_{j=1}^m$ are i.i.d. triples. For clarity of presentation, we consider $c_j\equiv c$ for a fixed threshold $c$, and the method remains valid for the random case.

Let $L_i=1+\mathbb{I}\{Y_i >c\}$ and $\tilde{L}_i=1+\mathbb{I}\{\tilde{Y}_i >c\}$ for $i\in[n+m]$.
Then the null hypothesis $H_j:Y_{n+j}\leq c$ is equivalent to $H_j:L_{n+j}\neq 2$, $j\in[m]$.
Let $\hat{L}_i\in\{1,2\}$ be an estimate of $L_i$ (e.g., constructed by $\hat{L}_i=1+\mathbb{I}\{\hat Y_i >c\}$).
Since candidates with $Y_{n+j}\leq c$ are of no interest, we simply set $\hat{L}_i\equiv 2$ for each $i\in[n+m]$ to gain power. Then the goal is to test $H_j:\hat{L}_{n+j} \neq L_{n+j}$, $j\in[m]$, where the proposed RCS can be applied. Specifically, the localized weight for class $2$ equals
\begin{equation}\label{local_weight_continuous}
w(X_i):=\frac{\mathbb{P}(L_i\neq 2|X=X_i)}{\mathbb{P}(\tilde{L}_i\neq 2|X=X_i)}=\frac{\mathbb{P}(Y_i\leq c \mid X=X_i)}{\mathbb{P}(\tilde{Y}_i\leq c \mid X=X_i)}.
\end{equation}
Let $\mathcal{A}_2=\{2\}$.
By \eqref{estimate_fdp}, the estimated value of FDP under threshold $t$ is 
\begin{align*}
\widehat{\text{FDP}}_{\mathcal{A}_2}(t)
&=\frac{m}{n}\cdot \frac{\sum_{i=1}^n w(X_i)\mathbb{I}\{\tilde{Y}_i\leq c, U_i\leq t\}}{1\vee\sum_{j=1}^m \mathbb{I}\{U_{n+j}\leq t\}}.
\end{align*}
The data-dependent threshold $\hat{t}$ is determined by 
\begin{equation}\label{rcs_cs_threshold}
\hat{t}=\max\{t\in [0,1]:\widehat{\text{FDP}}_{\mathcal{A}_2}(t)\leq \alpha\}.
\end{equation}
Then the rejection set is given by $\mathcal{R}=\mathcal{R}_{\mathcal{A}_2}=\{j\in[m]:U_{n+j}\leq \hat{t}\}$.

We note that the resulting group-wise FDP coincides with the $\mathrm{FDP}$ defined in \eqref{task_2_fdr}:
\begin{align*}
\text{FDP}=\frac{\sum_{j=1}^m \mathbb{I}\{Y_{n+j}\leq c,j\in\mathcal{R} \}}{1\vee |\mathcal{R}|}=\frac{\sum_{j=1}^m\mathbb{I}\{L_{n+j}=1, j\in\mathcal{R}_{\mathcal{A}_2}\}}{1\vee |\mathcal{R}_{\mathcal{A}_2}|}=\text{FDP}_{\mathcal{A}_2}.
\end{align*}
Similarly, the finite-sample power for Task 2 equals the recovery power for $\mathcal{A}_2$:
\begin{equation*}
\text{Power}=\mathbb{E}\left[\frac{\sum_{j=1}^{m}\mathbb{I}\{Y_{n+j}>c,j\in\mathcal{R}\}}{1\vee \sum_{j=1}^{m} \mathbb{I}\{Y_{n+j}>c\}}\right]
=\mathbb{E}\left[\frac{\sum_{j=1}^m\mathbb{I}\{L_{n+j}=2, j\in\mathcal{R}_{\mathcal{A}_2}\}}{1\vee \sum_{j=1}^m\mathbb{I}\{L_{n+j}=2\}}\right]=\mathrm{rPower}_{\mathcal{A}_2}.
\end{equation*}
For this binary reduction, the oracle feasible set is 
\[
\mathcal{F}_{\mathcal{A}_2}=\left\{t \in [0,1]: \mathbb{P}(U\le t)>0,\frac{\mathbb{P}(Y \leq c,U\leq t)}{\mathbb{P}(U\le t)}\le \alpha\right\}.
\]
Corollary~\ref{cor:cfBH_robust} demonstrates the validity and optimality of the RCS method for Task 2.

\begin{corollary}\label{cor:cfBH_robust}
Consider $(X_i,Y_i,\tilde{Y}_i)\overset{i.i.d.}{\sim}(X,Y,\tilde{Y})$ with $\mathbb{P}(Y>c)\in(0,1)$, $i\in[n+m]$.
Let $w(X_i)$ and $\hat{t}$ be defined by \eqref{local_weight_continuous} and \eqref{rcs_cs_threshold}, respectively, and let $\mathcal{R}=\{j\in[m]:U_{n+j}\leq \hat{t}\}$.
Suppose that (1) $
\mathbb P(U\le t \mid Y\le c)$ and $\mathbb P(U\le t \mid Y>c)$ are continuous on $t\in[0,1]$ whenever the conditioning event has a positive probability; (2) for every $\eta>0$, there exists $t_\eta\in[0,1]$ such that $\mathbb P(Y\le c,U\le t_\eta)<\alpha\mathbb P(U\le t_\eta)$ and $\mathbb P(U\le t_\eta)\ge R_{\mathcal A_2}^*-\eta$, where $R_{\mathcal A_2}^*=\sup_{t\in\mathcal F_{\mathcal A_2}}\mathbb P(U\le t)>0$; (3) the weight $w(\cdot)$ is uniformly bounded. Then we have
\begin{itemize}
\item[(1)] $\limsup_{n,m\to\infty}\mathrm{FDR}\le \alpha$.
\item[(2)] $\lim_{n,m\to\infty}\mathrm{Power}=\sup_{\tau\in\mathcal F_{\mathcal A_2}}\mathbb{P}(Y>c,U\leq \tau)/\mathbb{P}(Y>c)$.
\end{itemize}
\end{corollary}
The proof follows by verifying that the assumptions of Theorem~\ref{theorem2} are satisfied.

\paragraph{Estimation of $w(x)$.} We first consider the class-conditional label contamination model on $(X_i,L_i,\tilde{L}_i)$, under which the strategy in Section~\ref{estimate_w} can be used directly to estimate $w(x)$. Specifically, let $\mathbf{T}^{(L)}\in\mathbb{R}^{2\times2}$ with $T_{jk}^{(L)}=\mathbb{P}(\tilde{L}=k \mid L=j)$. One can use the noisy training dataset to obtain an estimate of $\tilde{\nu}_k(x):=\mathbb{P}(\tilde{L}=k \mid X=x)$, denoted by $\hat{\mathbb{\nu}}^{\text{noisy}}_k$, and then estimate $w(x)$ by 
\[
\hat w(x) = \frac{ 1-\sum_{j=1}^{2} \hat\nu^{\mathrm{noisy}}_j(x)
\bigl[(\hat T^{(L)})^{-1}\bigr]_{j2}
}{
\hat\nu^{\mathrm{noisy}}_1(x)
},
\]
where $\hat{\mathbf{T}}^{(L)}$ is an estimate of $\mathbf{T}^{(L)}$. For general contamination models, we let $a_\ell(x):=\mathbb{P}(\tilde{L}=1 \mid L=\ell,X=x),~\ell \in\{1,2\}$.
By straightforward calculation, we have $w(x)=(1-a_2(x)/\tilde{\nu}_1(x))/(a_1(x)-a_2(x))$. Let $\hat{a}_\ell(x)$ be an estimate of $a_\ell(x)$, $\ell\in\{1,2\}$, which can be obtained through several methods. Then a plug-in estimator of $w(x)$ is
\begin{equation*}
\hat{w}(x)=\frac{\hat{\nu}_1^{\text{noisy}}(x)-\hat{a}_2(x)}{(\hat{a}_1(x)-\hat{a}_2(x))\hat{\nu}_1^{\text{noisy}}(x)}.
\end{equation*}
In Appendix~\ref{task_2_appendix_2}, we also provide a convenient and detailed strategy to estimate $w(x)$.

\section{Numerical Experiments}\label{sec6}
We evaluate the performance of RCS on simulated data for two tasks in Section~\ref{sec:classification} and Section~\ref{sec:regression}, respectively. We aim to show that, under contaminated calibration data, RCS maintains FDR control and can be much more powerful than existing baselines. We also aim to demonstrate that RCS is robust to bounded and estimated contamination models, even when the structure of the contamination model is mis-specified.

\subsection{Simulations: Selecting Correctly Classified Candidates}\label{sec:classification}
We focus on the task of selecting true classified candidates.
We investigate the performance of RCS under known, bounded, and estimated label contamination models in Sections~\ref{task1:case1}, \ref{task1:case2}, and \ref{task1:case3}, respectively.
For all experiments, we set $\mathbb{P}(Y=k)=1/K$, $k\in[K]$.
The conditional distribution follows $X|Y=k \sim N(m_k,I_p)$ with $m_k=3k \times1_{p} /(2p^{1/4})$, $p=20$, $k\in[K]$.
We consider two label contamination models: the randomized response model and the asymmetric label noise model.
For the randomized response model, matrix $\mathbf{T}$ is given by $T_{kk}=1-\epsilon+\epsilon/K$ and $T_{kj}=\epsilon/K$ for $j\neq k,k\in[K]$. For the asymmetric label noise model, the transition matrix $\mathbf{T}$ is specified by $T_{kk}=1-\epsilon$, $T_{k,k^+}=\epsilon/2$, and $T_{k,j}=\epsilon/(2K-4)$ for $j\notin\{k,k^+\}$, where $k^+=+(k\bmod K)$, $k\in[K]$.
For the sizes of $\mathcal{D}_{\text{cal}}$ and $\mathcal{D}_{\text{test}}$, we set $n=m=Kn_{\text{base}}$. For the base classifier $\hat{f}$, we employ a LightGBM~\citep{ke2017lightgbm} trained on $2000K$ independent samples with contaminated labels generated as above.
Specifically, let $\hat{f}(x)=\arg\max_{k\in[K]} \hat{\mu}_k^{\text{noisy}}(x)$, where $\hat{\mu}_k^{\text{noisy}}(x)$ is the class-probability estimate produced by LightGBM.
We compare RCS with PSP~\citep{sun2025unified} under the same classifier $\hat{f}$ and uncertainty score $U(x)=1-\max_{k\in[K]} \hat{\mu}_k^{\text{noisy}}(x)$. For the RCS method, the localized weight is computed as in \eqref{calibrated_localweights}. We also include the RCS method with oracle weights for comparison.
We note that for the randomized response model, the RCS framework is simplified, where $w_k(x)$ is not used. In this case, we only present the RCS method.
All plotted results are averages over 500 independent simulation trials.

\subsubsection{Simulations under known label contamination models}\label{task1:case1}
We first consider the case of known label contamination models. For both the randomized response model and the asymmetric label noise, we set $K=4$ and set the target level of $\text{FDR}_{\text{o}}$ as $0.1$, while varying $n_{\text{base}}$ over $\{500,1000,3000,5000\}$ and the noise level $\epsilon$ over $\{0,0.05,0.1,0.15,0.20\}$.
The results of the randomized response model and the asymmetric noise are presented in Figure \ref{0_uni_overall} and Figure \ref{0_asy_overall}, respectively.

\begin{figure}[ht]
\centering
    \includegraphics[width=0.8\textwidth]{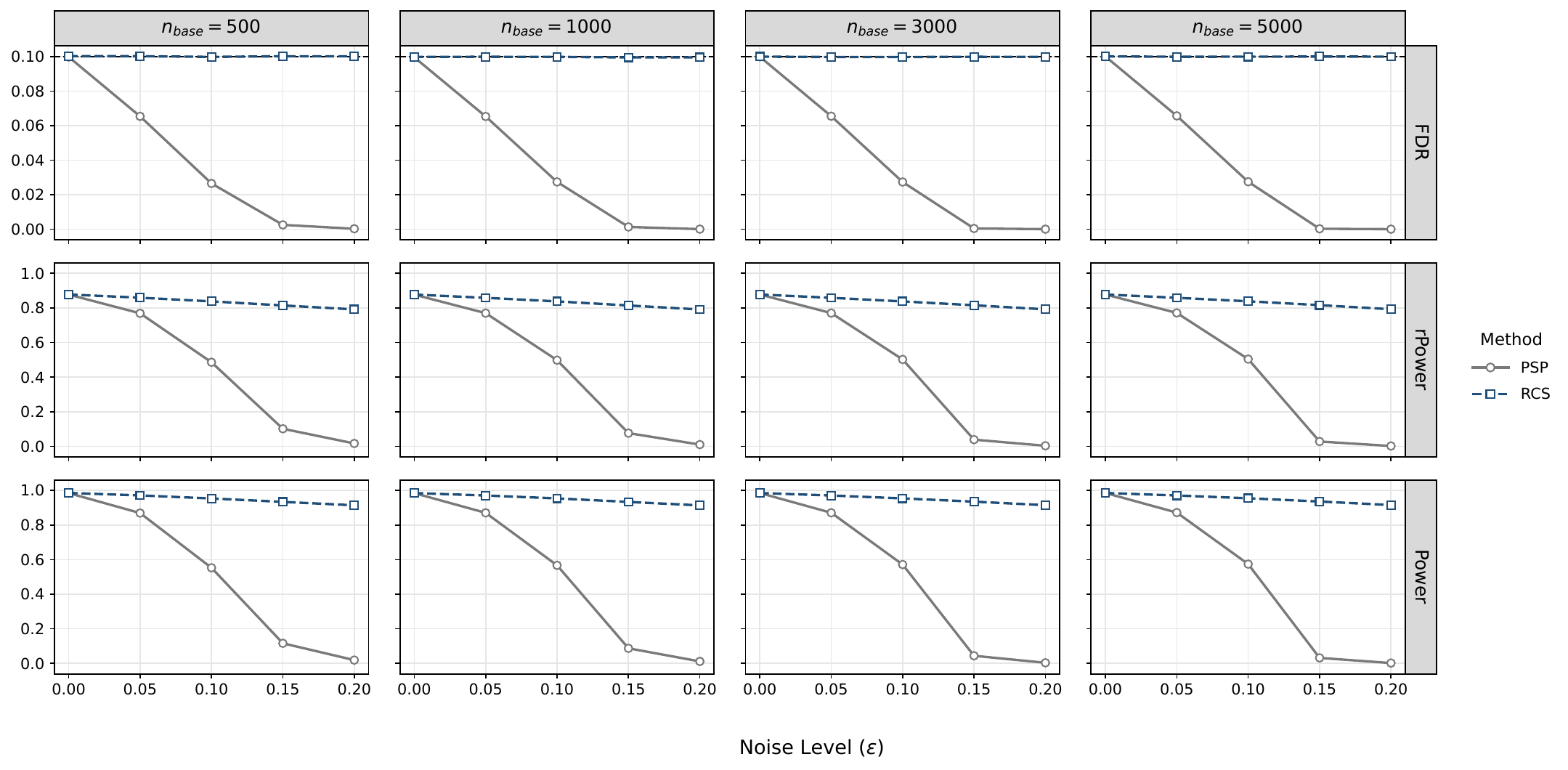}
    \caption{Results under a known randomized response model in overall classification.}
    \label{0_uni_overall}
\end{figure}

\begin{figure}[ht]
\centering
    \includegraphics[width=0.85\textwidth]{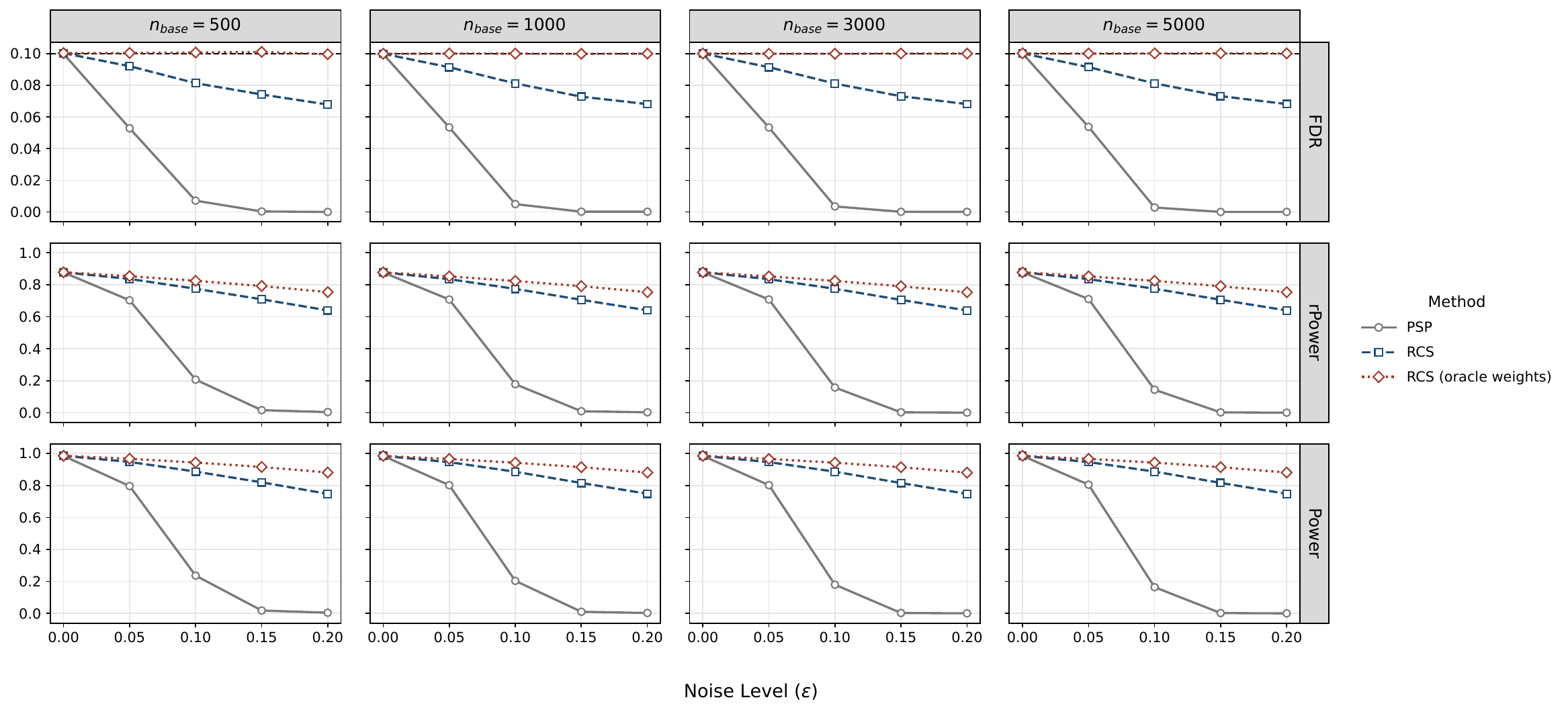}
    \caption{Results under a known asymmetric label noise in overall classification.}
    \label{0_asy_overall}
\end{figure}

By Figure \ref{0_uni_overall}, PSP quickly loses power when the noise level $\epsilon$ of the randomized response model increases. In contrast, the proposed \texttt{RCS} always maintains an FDR close to the nominal level and remains powerful for $\epsilon>0$. For the asymmetric label noise, the above results still hold as shown in Figure \ref{0_asy_overall}. Specifically, \texttt{RCS (oracle weights)} has the highest power in all settings, and \texttt{RCS} is much more powerful than PSP.
These results support the effectiveness of the proposed RCS framework.
It is also observed that \texttt{RCS} performs similarly across different values of $n_{\text{base}}$, implying that \texttt{RCS} can perform well under finite sample sizes, despite its theoretical properties being established in an asymptotic sense.

\begin{figure}[ht]
\centering
    \includegraphics[width=0.85\textwidth]{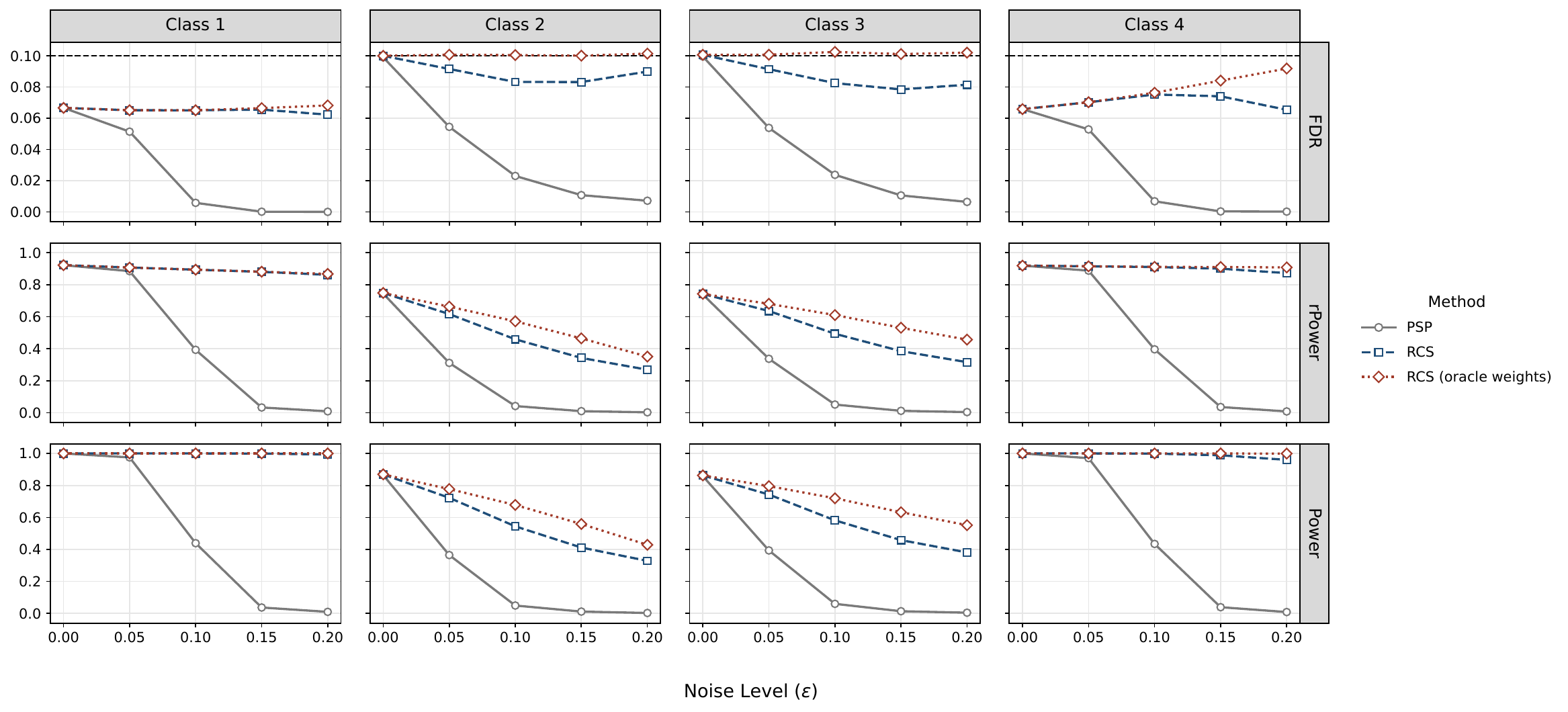}
    \caption{Results under a known asymmetric label noise in class-wise classification.}
    \label{0_asy_classwise}
\end{figure}

To demonstrate that the RCS framework can achieve class-wise FDR control under label noise, we set $\mathcal{A}_k=\{k\}$ and $\alpha_{\mathcal{A}_k}=0.1$ for all $k\in[K]$. Figure~\ref{0_asy_classwise} presents the performance under the asymmetric label noise in the class-wise classification problem with $n_{\text{base}}=1000$. It shows that all methods achieve valid class-wise FDR control across different values of $\epsilon$, and for each class, RCS methods achieve consistently higher power than PSP.

\subsubsection{Simulations under bounded label contamination models}\label{task1:case2}

\begin{figure}[ht]
\centering
    \includegraphics[width=0.8\textwidth]{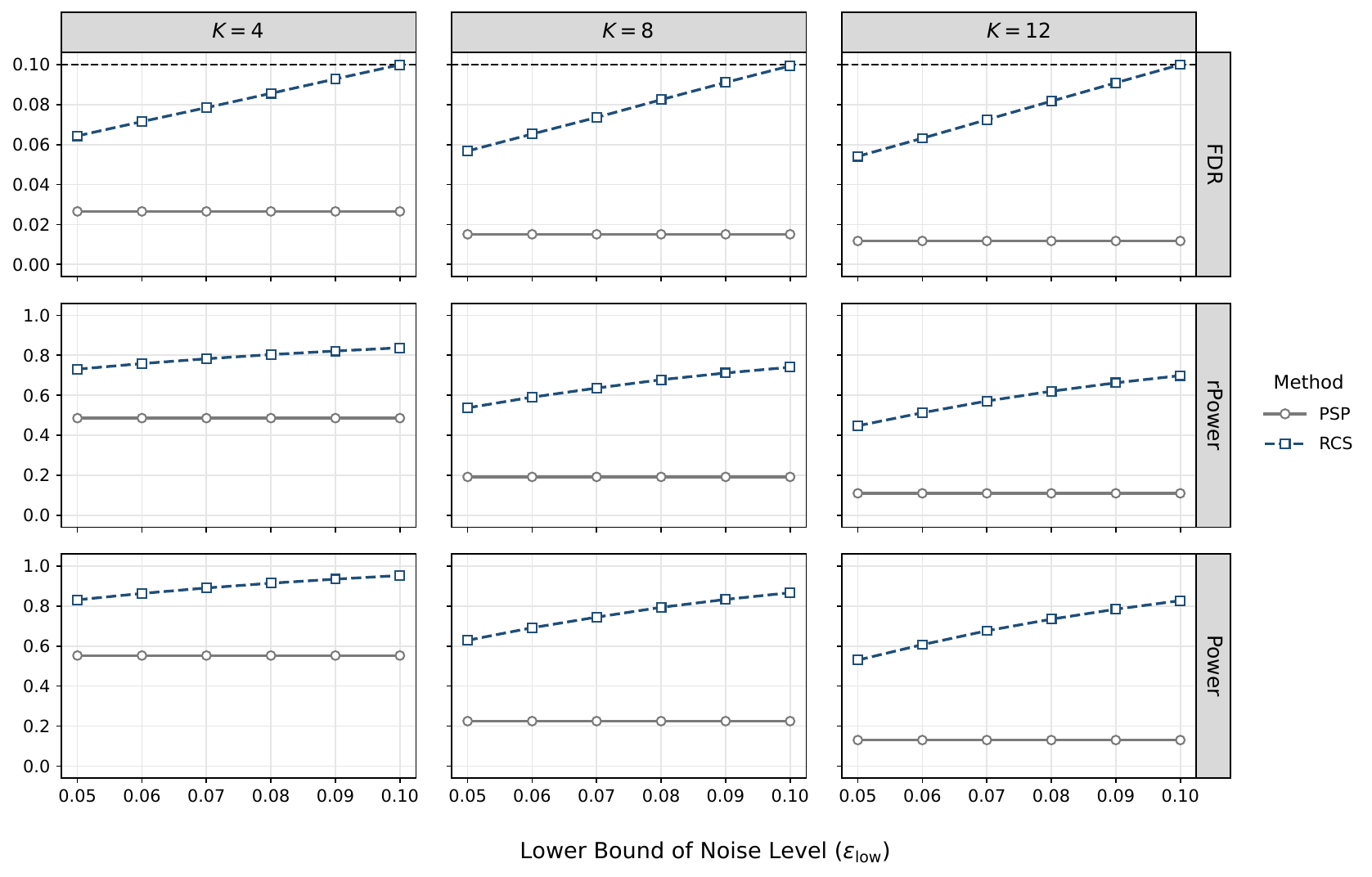}
    \caption{Performance of methods under a randomized response model ($\epsilon=0.1$) with bounded $\epsilon_{\text{low}}$ in overall classification.}
    \label{1_uni_overall}
\end{figure}

In this section, we no longer assume perfect knowledge of the transition matrix $\mathbf{T}$, but focus on the performance of bounded label contamination models. As in \cite{sesia2025adaptive}, we consider the randomized response model, where a lower bound of its noise level $\epsilon$ is known. Specifically, we fix $\epsilon=0.1$ and $n_{\text{base}}=500$, and apply \texttt{RCS} using assumed lower bounds $\epsilon_{\text{low}}\in\{0.05,0.06,0.07,0.08,0.09,0.10\}$. We also vary $K$ in $\{4,8,12\}$. Results are presented in Figure~\ref{1_uni_overall}. We see that \texttt{RCS} maintains FDR control in all settings. The power of \texttt{RCS} increases as $\epsilon_{\text{low}}$ gets closer to the ground truth $\epsilon=0.1$.
We also evaluate the performance under the asymmetric noise model and obtain the same conclusions; see Figure~\ref{1_asym_overall} in Appendix~\ref{simu:bounded_asymmetric}.

\subsubsection{Robustness to model estimation and mis-specification}\label{task1:case3}

We next investigate a more practical setting in which the contamination model is not directly available and must be estimated from an independent model-fitting dataset.
We first consider the model estimation under a correctly specified randomized response model.
Let $\epsilon=0.1$, $K=4$, and $n_{\mathrm{base}}=500$.
Following the setting of \cite{sesia2025adaptive}, we estimate $\epsilon$ by leveraging clean fitting and noisy fitting samples that are independent of the calibration and test datasets, with per-class sample sizes being $n_{\text{clean}}^{\text{fit}}$ and $n_{\text{noisy}}^{\text{fit}}$, respectively.
Therefore, the noisy samples used to train $\hat{f}$ can be reused here.
Let $n_{\text{clean}}^{\text{fit}}\in\{50,100,200,500,1000\}$ and $n_{\text{noisy}}^{\text{fit}}\in\{200,500,2000\}$.
For each combination of sample sizes, we consider two versions of RCS: \texttt{RCS (plug-in $\mathbf{T}$)} and \texttt{RCS (CI-lower $\mathbf{T}$)}, which estimate the matrix $\mathbf{T}$ through a plug-in method and a $95\%$ one-sided bootstrap lower confidence bound, respectively; see Appendix~\ref{simu:estimate} for details. RCS with oracle $\mathbf{T}$ and PSP are also included for comparison.

\begin{figure}[ht]
\centering
\includegraphics[width=0.8\textwidth]{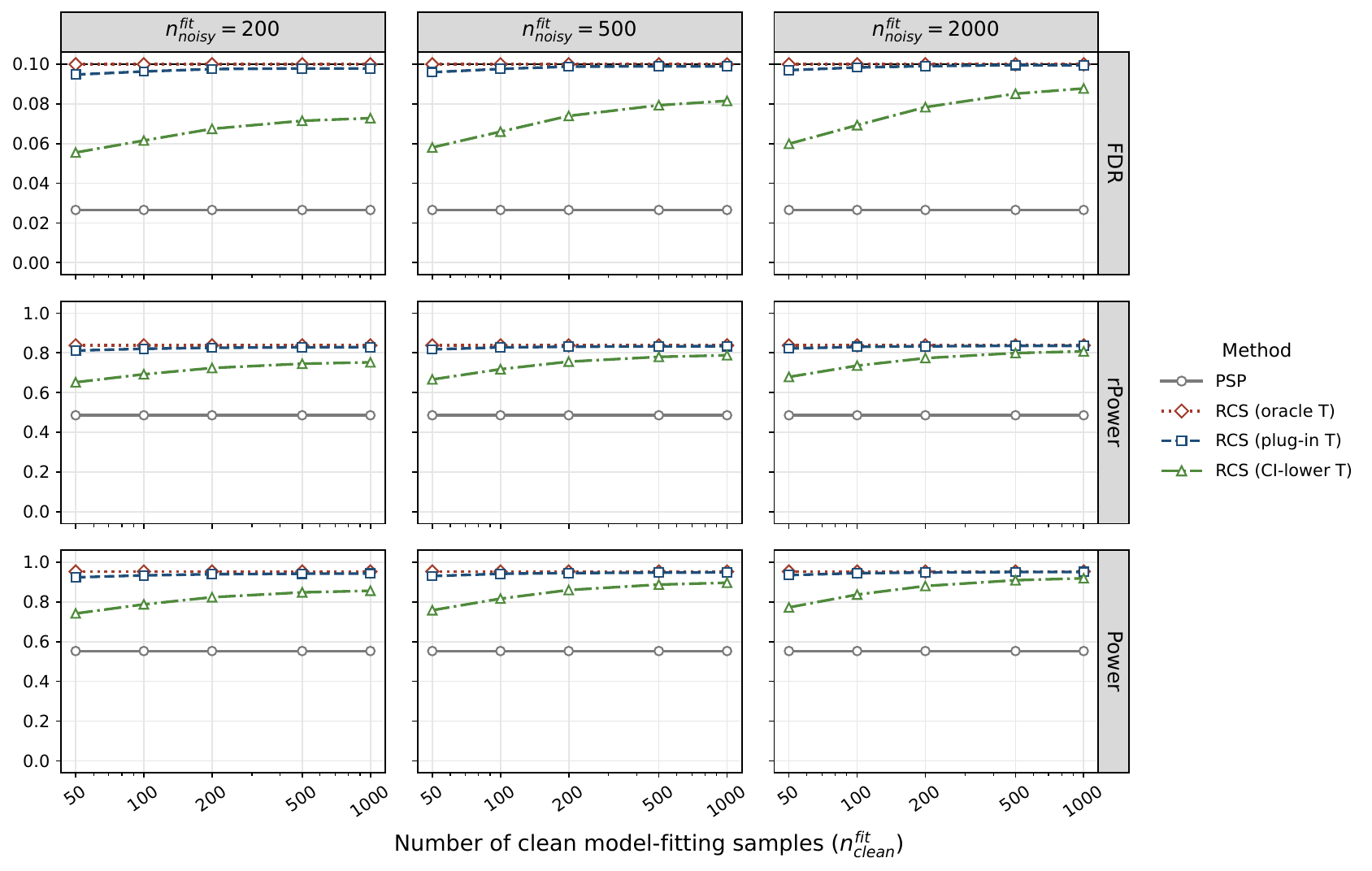}
    \caption{Robustness to estimated randomized response models.}
    \label{es_uni_overall}
\end{figure}

Results are presented in Figure \ref{es_uni_overall}. We observe that all RCS-based methods control FDR below $0.1$ and achieve significantly higher power than PSP, even for small $n_{\text{clean}}^{\text{fit}}$ and $n_{\text{noisy}}^{\text{fit}}$.
The power of \texttt{RCS (CI-lower $\mathbf{T}$)} increases as $n_{\text{clean}}^{\text{fit}}$ increases. This demonstrates the effectiveness of RCS under the estimated label contamination model.

We further consider a more challenging setting, where the structure of the contamination model is unknown or mis-specified. Consider $K=4$ and $n_{\mathrm{base}}=500$.
The four classes are divided into two blocks: $\{1,2\}$ and $\{3,4\}$, and label contamination occurs only within the same block. Let
\begin{equation*}
T_{jk}^{\mathrm{block}}
=
(1-\epsilon)\mathbb I\{j=k\}
+
\frac{\epsilon}{2}
\left(
\mathbb I\{j,k\in\{1,2\}\}
+
\mathbb I\{j,k\in\{3,4\}\}
\right),
\quad j,k\in[4].
\label{eq:block_transition}
\end{equation*}
An additional i.i.d. sample of size $Kn_{\text{pair}}^{\text{fit}}$ containing both noisy and true labels is used to estimate $\mathbf{T}$.
We consider two versions of RCS: \texttt{RCS (plug-in $\mathbf{T}$)}, which estimates the matrix $\mathbf{T}$ through a plug-in method, and \texttt{RCS (mis-specified $\mathbf{T}$)}, which directly estimates $\epsilon$ by assuming the randomized response model; see Appendix~\ref{simu:estimate} for details. The \texttt{RCS (oracle $\mathbf{T}$)} and PSP methods are also compared.

\begin{figure}[ht]
\centering
\includegraphics[width=0.8\textwidth]{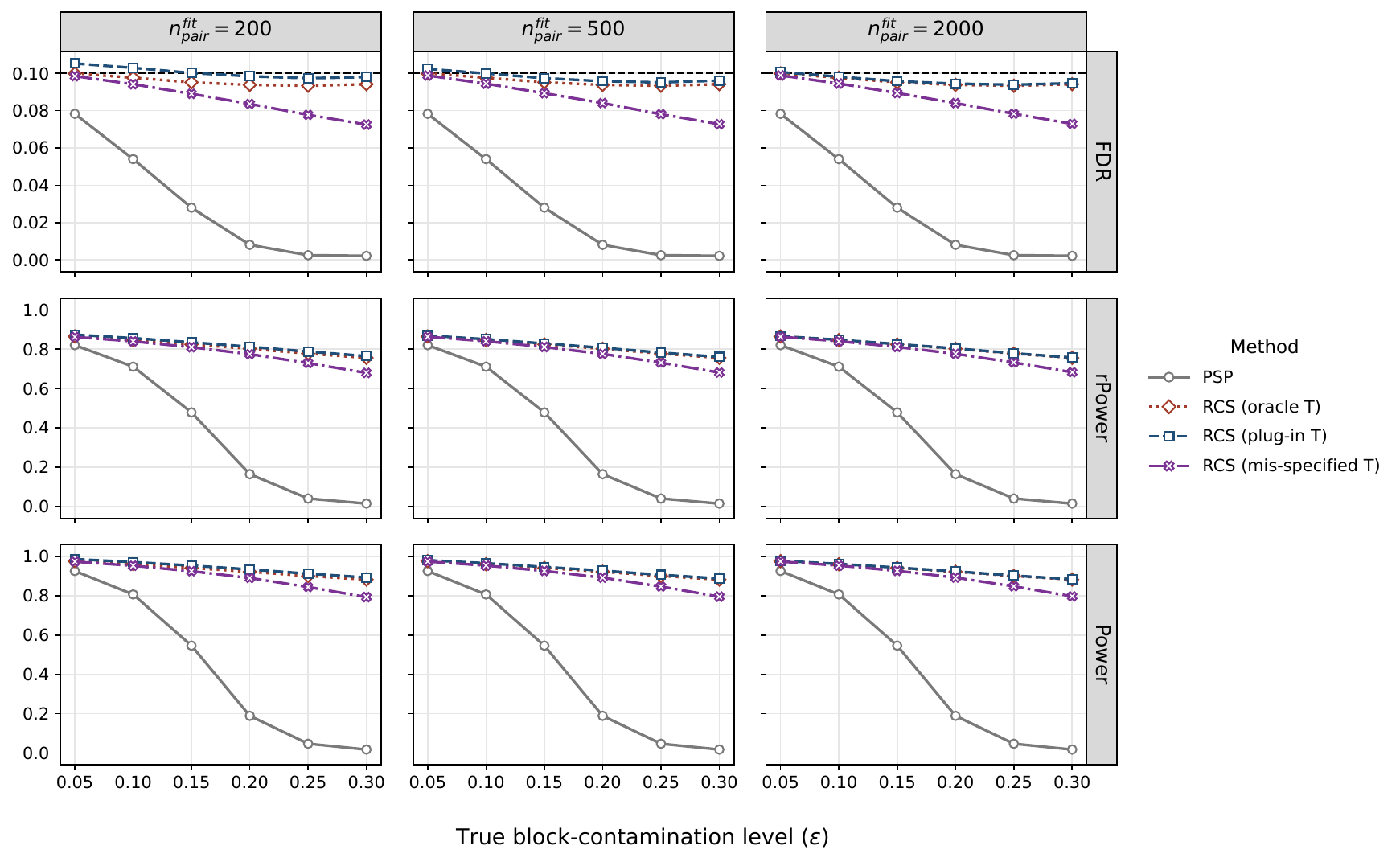}
    \caption{Robustness to structural mis-specification of the contamination model.}
    \label{mis_speci}
\end{figure}

We vary $\epsilon \in \{0.05,0.10,0.15,0.20,0.25,0.30\}$ and $n_{\text{pair}}^\text{fit}\in\{200,500,2000\}$ and present results in Figure~\ref{mis_speci}. We observe that the PSP method quickly loses power as $\epsilon$ increases. In contrast, all RCS methods remain powerful and control FDR well across different values of $\epsilon$. \texttt{RCS (plug-in $\mathbf{T}$)} and \texttt{RCS (oracle $\mathbf{T}$)} exhibit similar performance, while \texttt{RCS (mis-specified $\mathbf{T}$)} only suffers a slight power loss compared with \texttt{RCS (oracle $\mathbf{T}$)}.

 \subsection{Simulations: Selecting Candidates with Responses Exceeding a Threshold}\label{sec:regression}

We evaluate RCS for selecting candidates with large responses.
We generate $X_i\sim \mathrm{Unif}[-1,1]^{20}$ and $Y_i=\mu(X_i)+\epsilon_i$,  
where $\mu(x)=2(x_1x_2+x_3^2 +e^{x_4-1}-1)$ with $X_i=x:=(x_1,\dots,x_{20})$ and $\epsilon_i|X_i=x\sim N(0,\sigma(x)^{2})$ with $\sigma(x)=\max\{0.0625\mu(x)^2\mathbb{I}\{|\mu(x)|<2\}+0.125|\mu(x)|\mathbb{I}\{|\mu(x)|\geq1\},0.05\}$.
We consider two classical contamination models in Sections~\ref{Constant_status_transition_contamination} and \ref{Additive_response_noise}, respectively. For all experiments, the task is to select candidates with $Y_{n+j}>c$ among all $j\in[m]$ under the target FDR level $\alpha=0.1$ and $c=0$, where contaminated data are used for both training and calibration with $n_{\text{train}}=n=m=2000$. 

We take the most powerful version of \texttt{cfBH}, named $\texttt{BH\_clip}$, as a baseline in our experiments, where the nonconformity score satisfies $V(x,y)=100\cdot \mathbb{I}\{y>0\}-\hat{\mu}(x)$ and the regressor $\hat{\mu}(x)$ is trained using gradient boosting. For RCS-type methods, we take the version described in Section~\ref{sec:cf}. To ensure a fair comparison, we directly adopt $U(x)=-\hat{\mu}(x)$ as the uncertainty score. The same training dataset is also used to train $\hat\nu_1^{\text{noisy}}(x)$ for estimating weight $w(x)$.
All results are averages computed over 500 independent trials.

\subsubsection{Constant status-transition contamination}\label{Constant_status_transition_contamination}

We first consider constant status-transition contamination models. Recall that $L=1+\mathbb{I}\{Y>c\}$.
We generate a noisy label $\tilde{L}$ from $L$ through a constant transition matrix $\mathbf{T}\in \mathbb{R}^{2\times 2}$, where we let $T_{ab}$ be $\mathbb{P}(\tilde L=b\mid L=a)$ with a slight abuse of notation.
The contaminated response $\tilde Y$ is then constructed by $\tilde{Y}=c -M_Y \mathbb{I}\{\tilde{L}=1\}+M_Y \mathbb{I}\{\tilde{L}=2\}$, where $M_Y=|Y-c|\exp(\xi/10)+10^{-6}$ and $\xi\sim N(0,1)$. Write $T_{11}=1-\rho_{+}$ and $T_{22}=1-\rho_{-}$. We consider two contamination settings: (1) the false-active dominant model, where $\rho_{+}=\rho$ and $\rho_{-}=\rho/4$, and (2) the false-null dominant model, where $\rho_{+}=\rho/4$ and $\rho_{-}=\rho$.
We vary $\rho\in\{0,0.05,0.1,0.15,0.2,0.25\}$ and compare \texttt{BH\_clip}, \texttt{RCS (oracle weights)}, \texttt{RCS (known $\mathbf{T}$)}, and \texttt{RCS (plug-in $\mathbf{T}$)}, where the plug-in $\hat{\mathbf{T}}$ is obtained from an independent fitting dataset of size $500$ by leveraging the identical estimating procedure in Figure~\ref{mis_speci}.

\begin{figure}[t]
    \centering
    \includegraphics[width=0.8\textwidth]{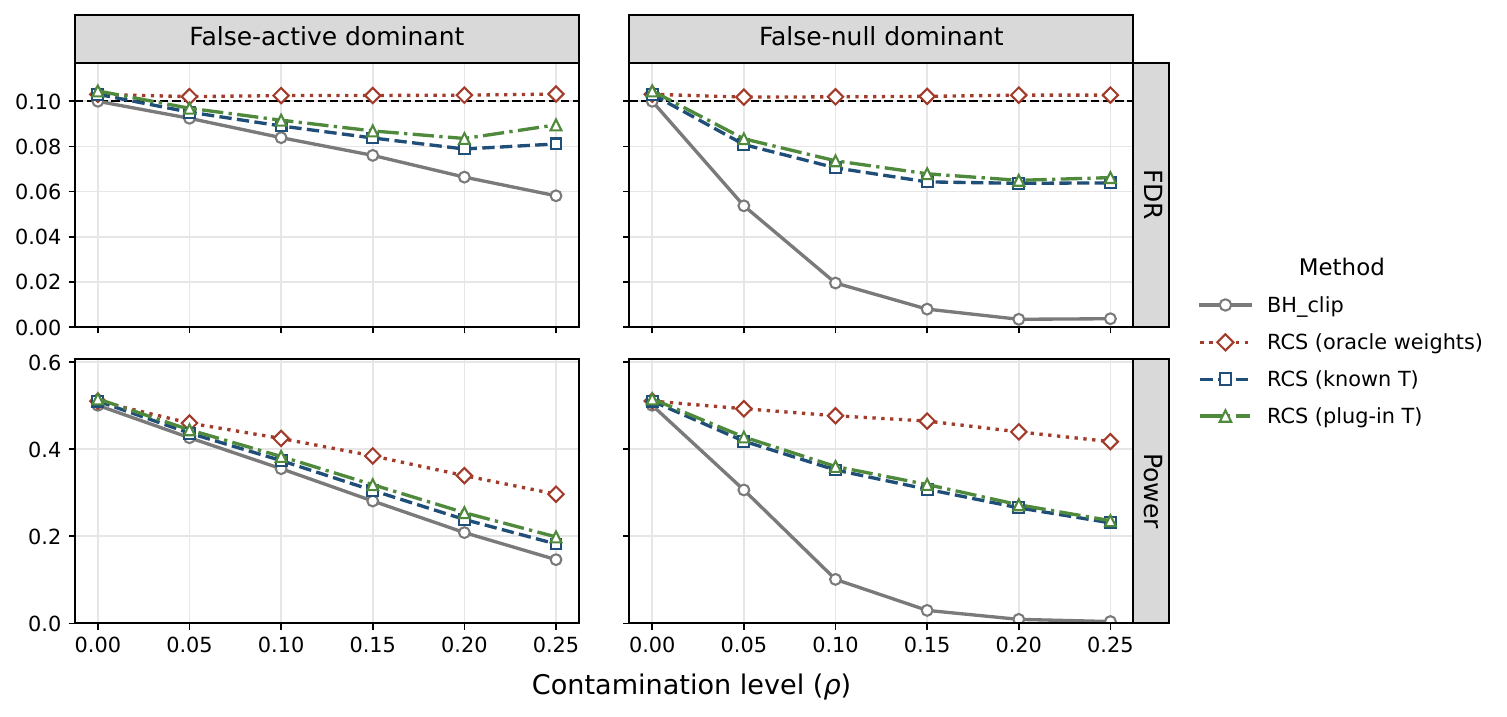}
    \caption{
    Results for Task 2 under two constant threshold-status contamination models. The left and right panels correspond to false-active dominant and false-null dominant contamination, respectively.}
    \label{fig:task2-constant}
\end{figure}

Results are presented in Figure~\ref{fig:task2-constant}. The \texttt{BH\_clip} suffers a severe loss of power as $\rho$ increases. In contrast, \texttt{RCS (oracle weights)} achieves an FDR close to the nominal level $\alpha$ and is more powerful than \texttt{BH\_clip}. \texttt{RCS (known $\mathbf{T}$)} and \texttt{RCS (plug-in $\mathbf{T}$)} also exhibit higher power than \texttt{BH\_clip}, especially for the false-null dominant model.

\subsubsection{Additive response noise}\label{Additive_response_noise}
We next focus on the common measurement-error model, where $\tilde Y=Y+\eta$ and $\eta\mid X=x\sim N(0,\tau_\rho^2(x))$. We consider two choices of $\tau_\rho(x)$: (1) the homoscedastic contamination, where $\tau_\rho(x)$ is a constant for $x\in\mathcal{X}$, and (2) the score-dependent contamination, which models the scenario where the contamination strength is larger in high-response regions; see Appendix~\ref{task_2_appendix_1} for details of $\tau_{\rho}(x)$. We compare \texttt{BH\_clip}, \texttt{RCS (oracle weights)}, and \texttt{RCS (score-localized)} for $\rho \in\{0,0.05,0.1,0.15,0.2,0.25\}$, where \texttt{RCS (score-localized)} refers to the RCS method with score-based estimator $\hat{w}(x)$; see Appendix~\ref{task_2_appendix_2} for details. The results of the two models are presented in Figure~\ref{fig:task2-additive}, which further demonstrate the effectiveness of the RCS framework.

\begin{figure}[t]
    \centering
    \includegraphics[width=0.8\textwidth]{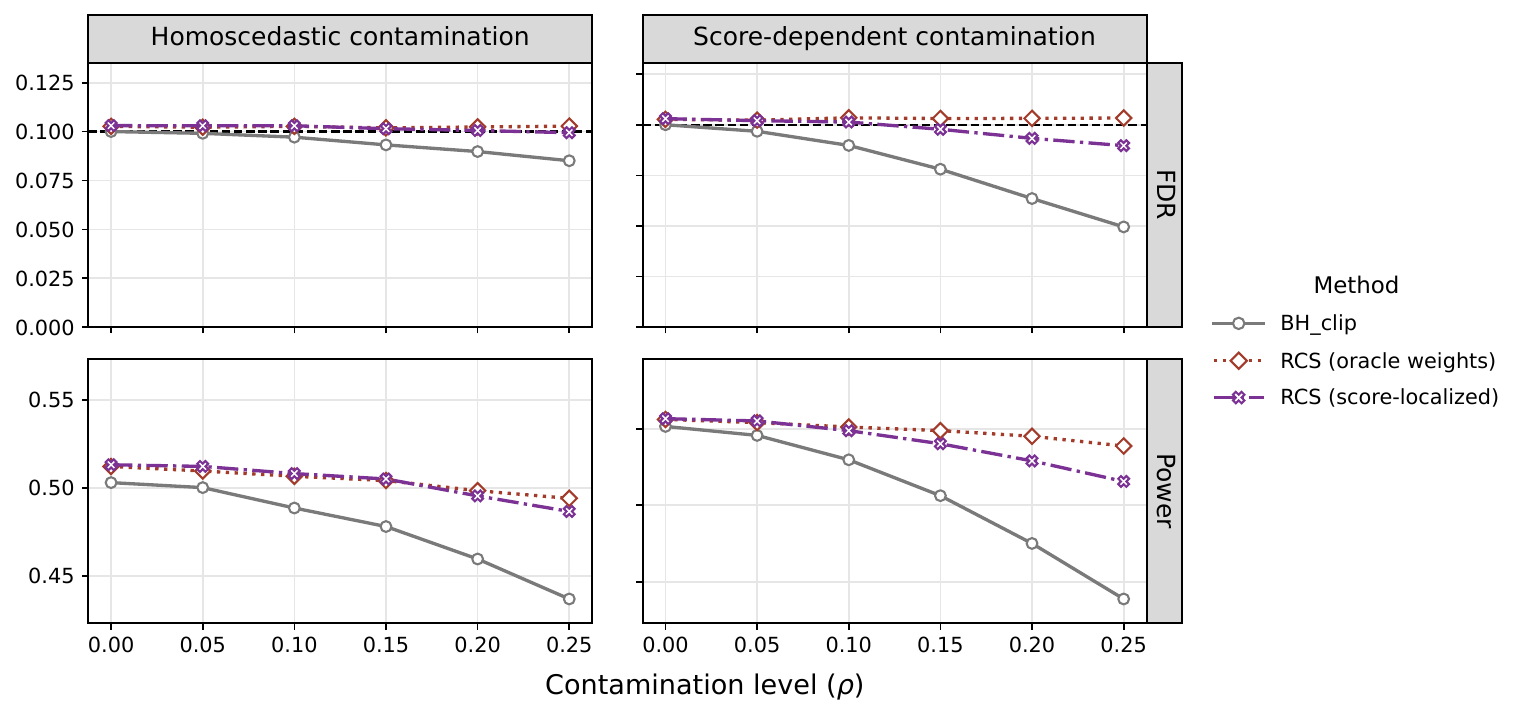}
    \caption{
    Results for Task 2 under additive Gaussian noise. The left and right panels correspond to homoscedastic and score-dependent contamination, respectively.
    }
    \label{fig:task2-additive}
\end{figure}

\section{Real Data Applications}\label{sec7}

\subsection{CIFAR-10H Image Data}\label{case_study:cifar}
We apply RCS to the reliable classification of real-world $32\times32$ color images. We consider the CIFAR-10H dataset~\citep{peterson2019human}, a subset of the CIFAR-10 dataset containing 10,000 images with imperfect labels assigned through Amazon Mechanical Turk, where
each annotator-provided label can be regarded as a noisy version of the underlying true label of the image. Following \cite{sesia2025adaptive}, the working label of each image is randomly sampled from a multinomial distribution whose weights are equal to the relative label frequencies assigned to that image by different human annotators. 
The base classifier is a ResNet-18 convolutional neural network trained on the 50,000 CIFAR-10 images excluded from the CIFAR-10H dataset, and we take the output of the final softmax layer of the network as $\hat{\mu}_k$.
However, different with the task of constructing prediction sets~\citep{sesia2025adaptive}, we aim to directly select truly classified candidates.

We randomly select 500 CIFAR-10H images as the test dataset and use a random calibration sample of size $n_{\text{cal}}\in\{500,1500,4500,9500\}$ from the remaining images. Following \cite{sesia2025adaptive}, we apply RCS by specifying a randomized response model, despite its potential misalignment with the underlying contamination mechanism.
The procedure for estimating $\mathbf{T}$ ($\epsilon$) completely coincides with \cite{sesia2025adaptive}.
For clarity, we denote the simplified version of RCS in Section~\ref{sec:RCS_uniform_noise} with plug-in $\hat\epsilon$ by \texttt{RCS-S (plug-in $\mathbf{T}$)} and denote the original version by \texttt{RCS (plug-in $\mathbf{T}$)}.
The \texttt{PSP} method is taken as a baseline.
For all methods, we set the overall FDR target level as $0.05$.

We plot the results in Figure~\ref{fig:cifar-10}. It can be observed that all methods control the FDR well and achieve higher power as the number of calibration samples increases.
\texttt{RCS-S (plug-in $\mathbf{T}$)} and \texttt{RCS (plug-in $\mathbf{T}$)} perform similarly, while both methods are more powerful than \texttt{PSP} for different calibration sample sizes.

\begin{figure}[t]
    \centering
    \includegraphics[width=0.9\textwidth]{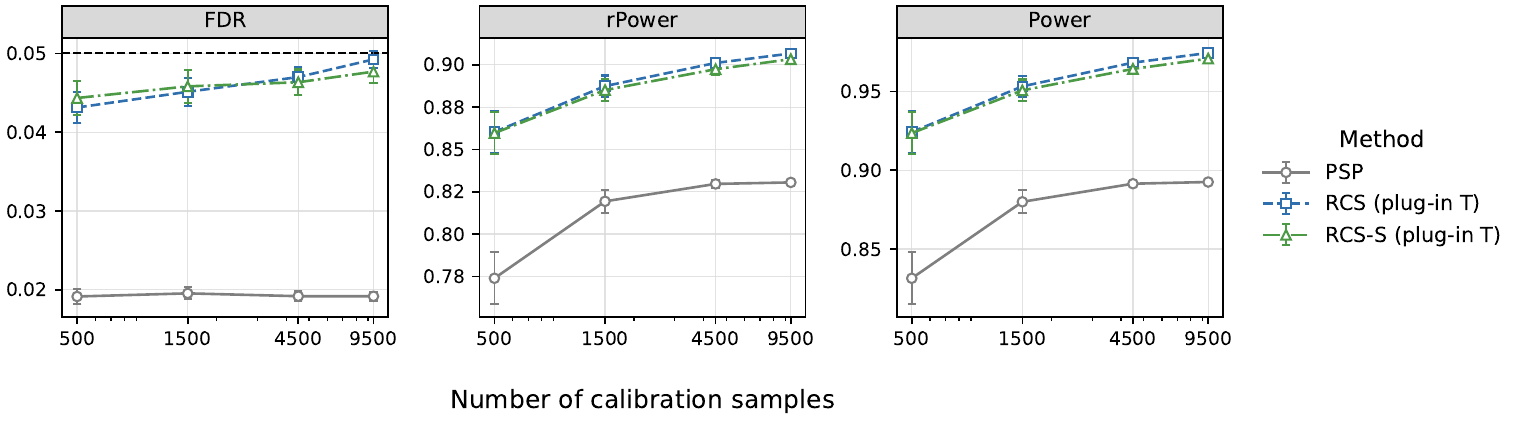}
    \caption{
     Performance on CIFAR-10H image data with imperfect labels under different calibration sizes. Points and error bars represent the mean and one standard deviation over 500 independent repetitions.
    }
    \label{fig:cifar-10}
\end{figure}

\subsection{ACS Income Data with Label Differential Privacy}

We next illustrate the effectiveness of RCS for datasets with label differential privacy.
We consider the 2018 California ACS income data from Folktables~\citep{ding2021retiring}, consisting of 195{,}665 data points, where the binary response indicates whether annual income exceeds 50{,}000 dollars and the feature vector represents the demographic and job information of each individual. In this study, we aim to identify high-income individuals, i.e., individuals with an income exceeding 50,000 dollars.

Following \cite{liu2023need}, we use a subsample consisting of 50,000 observations to train an XGBoost classifier, i.e., $\hat{\mu}(\cdot)$ for $\mathbb{P}(Y=1 \mid X=\cdot)$, and fix it in the subsequent experiments. 
The remaining observations form a pool, from which we repeatedly sample a test set and a calibration set, each of size $10{,}000$.
To mimic a privacy-preserving release of calibration data, we provide the calibration features together with label-private responses, which is natural when the response of interest—income status—is sensitive and should not be revealed directly to practitioners.
Specifically, we construct the binary randomized response with privacy budget $\varepsilon_{\mathrm{DP}}$ such that $\mathbb{P}(\tilde Y=Y)=e^{\varepsilon_{\mathrm{DP}}}/(1+e^{\varepsilon_{\mathrm{DP}}})$, where a smaller $\varepsilon_{\mathrm{DP}}$ corresponds to a stronger privacy.
As in Section~\ref{case_study:cifar}, two versions of RCS, denoted by \texttt{RCS} and \texttt{RCS-S} (the matrix $\mathbf{T}$ is known),
are applied. The \texttt{BH\_clip} method is taken as a baseline, where the nonconformity score satisfies $V(x,y)=100\mathbb{I}\{y=1\}-\hat{\mu}(x)$.

\begin{figure}[t]
    \centering
    \includegraphics[width=0.8\textwidth]{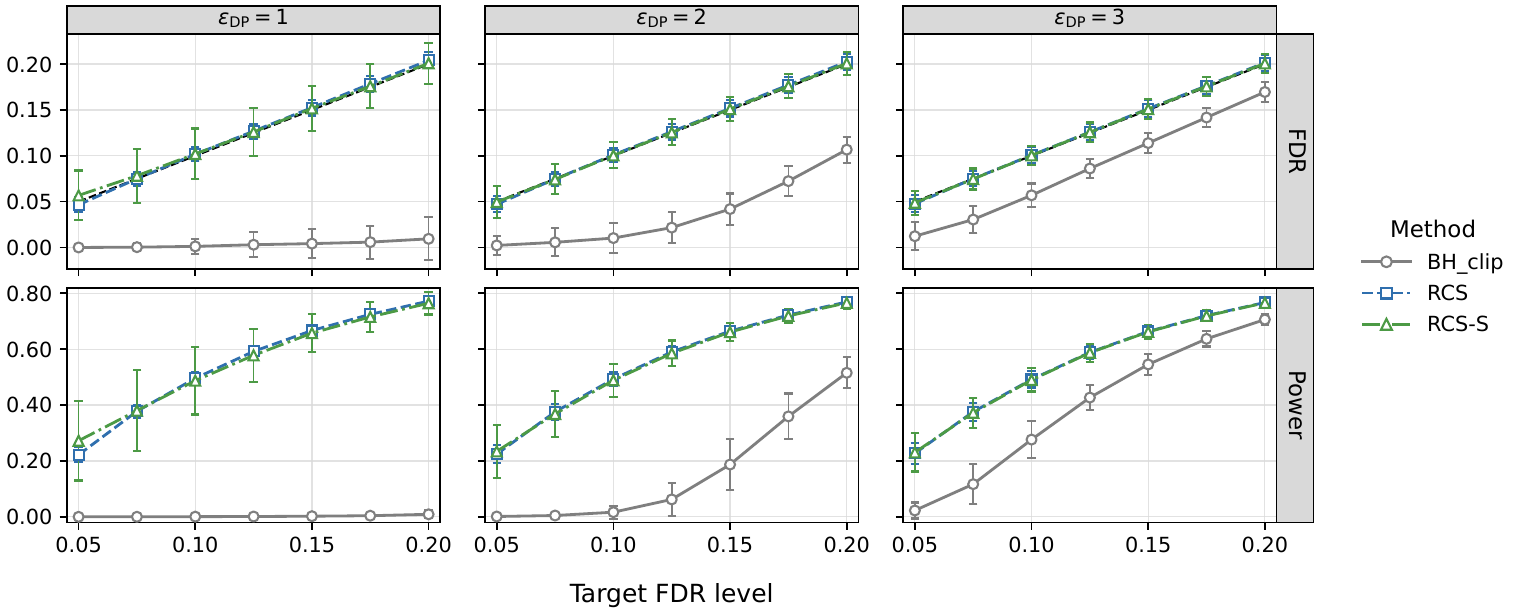}
    \caption{
     Performance on ACS income data with label differential privacy under different target FDR levels. The dotted line corresponds to the line of $y=x$.
    }
    \label{fig:label-dp}
\end{figure}

For each privacy level, we compare the three methods at target FDR levels ranging from $0.05$ to $0.20$; see Figure~\ref{fig:label-dp} for the results of 500 independent repetitions.
Across different settings, both \texttt{RCS} and \texttt{RCS-S} track the nominal FDR level closely while consistently achieving substantially higher power than \texttt{BH\_clip}.
Furthermore, the performance of RCS-based methods remains stable as the privacy protection becomes stronger, indicating that RCS can accommodate substantial label perturbation without an obvious loss of power. Overall, \texttt{RCS} is slightly more stable across different repetitions, whereas \texttt{RCS-S} exhibits somewhat greater variability under the strongest privacy setting.


\section{Discussion}
In this paper, we study conformalized selection in the presence of calibration data with contaminated responses, an important problem in real-world applications that has not been explicitly explored in existing literature. Focusing on selecting candidates whose predicted labels are correct or responses exceed a target threshold, we demonstrate that classical methods can either fail to control the false discovery rate or be too conservative if the responses of calibration data are contaminated.
To tackle this, we develop RCS, a unified framework that can leverage any classification model and a contaminated calibration dataset to achieve large-scale trustworthy labeling. The key insight of RCS lies in reinterpreting the response contamination as a localized covariate-shift problem and estimating the FDP from a covariate-adjusted empirical-Bayes perspective. We establish asymptotic group-wise FDR control, power optimality, and robustness of RCS. We also develop instantiations of RCS under common contamination models and extensions beyond classification.

This research opens several opportunities for future investigation. In fact, since the seminal work of \cite{jin2023selection}, conformalized selection has evolved substantially, with extensions including multivariate selection~\citep{bai2025multivariate}, diversifying selection~\citep{nair2025diversifying}, selection under hierarchical data \citep{lee2025selection}, general risk control~\citep{bai2026conformal}, and many others. 
It would be valuable to extend the basic idea underlying the RCS framework to these settings.


\newpage

\appendix

\section{Further Discussions with Related Work}\label{related_work}

Conformal inference~\citep{vovk2005algorithmic} is an increasingly popular distribution-free framework for quantifying the uncertainty of predictions made by any black-box model. 
The seminal work can be traced back to conformal prediction, which aims to construct a prediction set with marginal coverage for a test point, with the help of calibration samples that are exchangeable with the test point.
By extending the basic idea, numerous conformal methods have been developed for different questions; see \cite{angelopoulos2024theoretical} for a comprehensive review.
This paper builds upon the line of conformalized selective inference, which aims to provide uncertainty quantification for multiple test samples simultaneously, and has received widespread attention for its practicality in drug discovery, large-scale classification, and LLM alignment~\citep{jin2023selection,sun2025unified,gui2024conformal,gui2025acs,rava2021burden}.
We make contributions by clarifying the applicability of existing popular methods with contaminated calibration datasets, and proposing a unified framework towards robust conformalized selection.

A fundamental limitation of conformal inference is its reliance on data exchangeability.
Numerous studies have sought to address this limitation: for the classical conformal prediction, its robustness under covariate shift~\citep{tibshirani2019conformal}, time-series
dependencies~\citep{gibbs2021adaptive}, label shift~\citep{podkopaev2021distribution}, and general unexchangeability~\citep{barber2023conformal} has been well studied.
The most closely related works are \cite{einbinder2024label,sesia2025adaptive}, which study the robustness of conformal prediction with calibration datasets containing noisy labels. However, as discussed in the introduction, their results focus on the prediction of a single test point; the label noise problem within the conformalized selection frameworks remains largely unexplored.
Furthermore, for the selection task considered in \cite{jin2023selection}, the response variable can be continuous, where the problem of response contamination is much more complex than label noise. This paper addresses both of the aforementioned gaps. 
Returning to the line of conformalized selective inference, recent studies have also begun to examine the influence of unexchangeability~\citep{jin2025model,lee2025selection,bai2026conformal}, but these investigations mainly focused on covariate shift.
As discussed in \cite{angelopoulos2024theoretical}, the response contamination (e.g., label noise) differs fundamentally from covariate shift, therefore requiring new techniques and strategies to address. \cite{bashari2025robust} also consider the contaminated reference data, but they focus on outlier detection, and their contamination refers to the reference data containing a small fraction of outliers, which is distinct from our contamination setting.

To the best of our knowledge, this is the first work to explore the properties of current conformal selection methods such as those of \cite{jin2023selection,sun2025unified} under calibration data with general response contamination, and the first to provide a strategy to address the response contamination issue within the conformalized selection framework. At the methodological level, handling different prediction categories separately is similar to \cite{sun2025unified}, but the two approaches serve distinct purposes, and the subsequent treatments are also different. By separately handling different classes, we translate the label contamination problem into the localized covariate shift problem and then adopt a covariate-adjusted empirical-Bayes perspective to estimate the number of false selections.

\section{Technical Proofs}\label{proofs}
\subsection{Proof of Theorem \ref{cs_jingYIN}}

\begin{proof}
Define
\[
V_i:=V(X_i,Y_i),\quad
\tilde{V}_i:=V(X_i,\tilde{Y}_i),\quad i\in[n+m].
\]
Write $\hat{V}_{n+j}=V(X_{n+j},c_j)$, $j\in[m]$.
Since $F_{\tilde V}$ is continuous and
$F_V\ll F_{\tilde V}$, both $\tilde V$ and
$V^\circ$ have atomless distributions. Hence, for every fixed $j\in[m]$,
$\tilde V_1,\ldots,\tilde V_n,V_{n+j}$
has no ties almost surely.Moreover, since $F_{\tilde V}(\hat V)$ has a continuous
distribution,
$F_{\tilde V}(\hat V_{n+1}),\dots,F_{\tilde V}(\hat V_{n+m})$ has no ties almost surely. Consequently, $\hat V_{n+1},\ldots,\hat V_{n+m}$ has no ties almost surely.

Define the noisy oracle p-values by
\begin{equation*}
\tilde{p}_{j}^{*}:=\frac{1+\sum_{i=1}^n \mathbb{I}\{\tilde{V}_i < V_{n+j}\}}{n+1},\quad j=1,\dots m.
\end{equation*}
We fix any $j\in[m]$. Define a set of modified noisy p-values by
\begin{equation*}
\tilde{p}_{l}^{(j)}:=\frac{\sum_{i=1}^n \mathbb{I}\{\tilde{V}_i < \hat{V}_{n+l}\}+\mathbb{I}\{V_{n+j}<\hat{V}_{n+l}\}}{n+1},\quad l=1,\dots,j-1,j+1,\dots,m.
\end{equation*}
Let $\mathcal{R}(a_1,\dots,a_m)\subseteq \{1,\dots,m\}$ be the rejection set obtained by applying BH procedure to p-values that take on the values $a_1,\dots, a_m$.
In this way, the original output is $\mathcal{R}=\mathcal{R}(\tilde{p}_1,\dots,\tilde{p}_{m})$.
We compare it with  
\[
\mathcal{R}_j^*:=\mathcal{R}(\tilde{p}_1^{(j)},\dots,\tilde{p}_{j-1}^{(j)},\tilde{p}_{j}^{*},\tilde{p}_{j+1}^{(j)},\dots,\tilde{p}_{m}^{(j)})
\]
on the event $\{Y_{n+j}\leq c_j,~\delta_{j}=1\}$, where $\delta_j=1\{j\in\mathcal{R}\}$ denotes the decision of hypothesis $H_{j}$ for $j=1,\dots m$. First, on this event, we have $V_{n+j}\leq \hat{V}_{n+j}$ due to the monotonicity of $V$, which then leads to $\tilde{p}_{j}^{*}\leq \tilde{p}_j$ by the definition of noisy p-values. 
Second, for the remaining p-values, since scores have no ties, we consider two cases:
\begin{itemize}
\item[(i)] If $\hat{V}_{n+l}>\hat{V}_{n+j}$, then $\hat{V}_{n+l}>V_{n+j}$ since $\hat{V}_{n+j}\geq V_{n+j}$. Thus, 
\begin{equation*}
\tilde{p}_{l}^{(j)}=\frac{1+\sum_{i=1}^n \mathbb{I}\{\tilde{V}_i < \hat{V}_{n+l}\}}{n+1}=\tilde{p}_{l}.
\end{equation*}
This means that if $\hat{V}_{n+l}>\hat{V}_{n+j}$, replacing $\tilde{p}_{l}$ with $\tilde{p}_{l}^{(j)}$ does not change the rejection set of BH procedure.
\item[(ii)] If $\hat{V}_{n+l}<\hat{V}_{n+j}$, then $\tilde{p}_{l}\leq \tilde{p}_{j}$. Therefore, we have $l \in \mathcal{R}$ on the event $\{\delta_j=1,Y_{n+j}\leq c_j\}$ by the property of BH procedure. Note that on this event, by definition we have
\begin{align*}
\tilde{p}_{l}^{(j)}
&\leq \frac{1+\sum_{i=1}^n \mathbb{I}\{\tilde{V}_i<\hat{V}_{n+l}\}}{n+1}\\
&\leq \frac{1+\sum_{i=1}^n \mathbb{I}\{\tilde{V}_i<\hat{V}_{n+j}\}}{n+1}=\tilde{p}_{j}.
\end{align*}
Since $l\in\mathcal{R}$, by the step-up nature of BH procedure, replacing $\tilde{p}_{l}$ with a smaller value $\tilde{p}_{l}^{(j)}$ still does not change the rejection set.
\end{itemize}
Consequently, on the event $\{\delta_{j}=1,Y_{n+j}\leq c_j\}$, we have
\begin{equation}\label{leave-one-out-same}
\begin{aligned}
\mathcal{R}
&=\mathcal{R}(\tilde{p}_1^{(j)},\dots,\tilde{p}_{j-1}^{(j)},\tilde{p}_j,\tilde{p}_{j+1}^{(j)},\dots,\tilde{p}_{m}^{(j)})\\
&=\mathcal{R}(\tilde{p}_1^{(j)},\dots,\tilde{p}_{j-1}^{(j)},\tilde{p}_j^{*},\tilde{p}_{j+1}^{(j)},\dots,\tilde{p}_{m}^{(j)})\\
&=\mathcal{R}_{j}^{*}.
\end{aligned}
\end{equation}

By the self-consistency of BH procedure, we have the decomposition of FDR that 
\begin{align*}
\mathrm{FDR}
&=\mathbb{E}\left[\frac{\sum_{j=1}^m \mathbb{I}\{Y_{n+j}\leq c_j\}\delta_j}{\max\{1,\sum_{j=1}^m \delta_j\}}\right]\\
&=\sum_{j=1}^{m}\sum_{k=1}^{m} \frac{1}{k}\mathbb{E}\left[\mathbb{I}\{|\mathcal{R}|=k\} \mathbb{I}\{Y_{n+j}\leq c_j\} \mathbb{I}\{\tilde{p}_{j}\leq \alpha k /m\}\right].
\end{align*}
We have
\begin{align*}
\mathbb{I}\{|\mathcal{R}|=k\}\mathbb{I}\{Y_{n+j}\leq c_j\}\mathbb{I}\{\tilde{p}_{j}\leq \alpha k /m\}
&=\mathbb{I}\{|\mathcal{R}_{j}^{*}|=k\}\mathbb{I}\{Y_{n+j}\leq c_j\}\mathbb{I}\{\tilde{p}_{j}\leq \alpha k /m\}\\
&\leq \mathbb{I}\{|\mathcal{R}_{j}^{*}|=k\} \mathbb{I}\{Y_{n+j}\leq c_j\} \mathbb{I}\{\tilde{p}_{j}^{*}\leq \alpha k /m\},
\end{align*}
where the first equation holds by \eqref{leave-one-out-same}.
Thus, 
\begin{align*}
\mathrm{FDR}
&\leq\sum_{j=1}^{m}\sum_{k=1}^{m} \frac{1}{k}\mathbb{E}\left[\mathbb{I}\{|\mathcal{R}_{j}^{*}|=k\}\mathbb{I}\{Y_{n+j}\leq c_j\}\mathbb{I}\{\tilde{p}_{j}^{*}\leq \alpha k /m\}\right]\\
&\leq \sum_{j=1}^{m}\sum_{k=1}^{m} \frac{1}{k}\mathbb{E}\left[\mathbb{I}\{|\mathcal{R}_{j}^{*}|=k\}\mathbb{I}\{\tilde{p}_{j}^{*}\leq \alpha k /m\}\right].
\end{align*}
Define $\delta_{j}^{*}=\mathbb{I}\{j\in\mathcal{R}_{j}^{*}\}$. By the step-up nature of the BH procedure, on the event $\{\delta_{j}^{*}=1\}$, replacing $\tilde{p}_{j}^{*}$ with $0$ does not change the rejection set, say, 
\begin{equation*}
\mathcal{R}_{j\to0}^{*}:=\mathcal{R}(\tilde{p}_1^{(j)},\dots,\tilde{p}_{j-1}^{(j)},0,\tilde{p}_{j+1}^{(j)},\dots,\tilde{p}_{m}^{(j)})=\mathcal{R}_{j}^{*}.
\end{equation*}
Therefore,
\begin{align*}
\mathrm{FDR}
&\leq \sum_{j=1}^{m}\sum_{k=1}^{m} \frac{1}{k}\mathbb{E}\left[\mathbb{I}\{|\mathcal{R}_{j}^{*}|=k\}\mathbb{I}\{\delta_{j}^{*}=1\}\right]\\
&=\sum_{j=1}^{m}\sum_{k=1}^{m} \frac{1}{k}\mathbb{E}\left[\mathbb{I}\{|\mathcal{R}_{j}^{*}|=k\}\mathbb{I}\{\delta_{j\to0}^{*}=1\}\right]\\
&=\sum\limits_{j=1}^{m}\mathbb{E}\left[\frac{\mathbb{I}\left\{\tilde{p}_{j}^{*}\leq \alpha|\mathcal{R}_{j\to0}^{*}|/m\right\}}{\max\{1,|\mathcal{R}_{j}^{*}|\}}\right]\\
&=\sum\limits_{j=1}^{m}\mathbb{E}\left[\frac{\mathbb{I}\left\{\tilde{p}_{j}^{*}\leq \alpha|\mathcal{R}_{j\to0}^{*}|/m\right\}}{\max\{1,|\mathcal{R}_{j\to0}^{*}|\}}\right],
\end{align*}
where $\delta_{j\to0}^{*}=\mathbb{I}\{j\in\mathcal{R}_{j\to0}^{*}\}$. By the tower property, we have
\begin{equation*}
\mathbb{E}\left[\frac{\mathbb{I}\left\{\tilde{p}_{j}^{*}\leq \alpha|\mathcal{R}_{j\to0}^{*}|/m\right\}}{\max\{1,|\mathcal{R}_{j\to0}^{*}|\}}\right]
=\mathbb{E}\left[
\mathbb{E}\left[\frac{\mathbb{I}\left\{\tilde{p}_{j}^{*}\leq \alpha|\mathcal{R}_{j\to0}^{*}|/m\right\}}{\max\{1,|\mathcal{R}_{j\to0}^{*}|\}} \Bigg| [\tilde{V}_1,\dots,\tilde{V}_{n},V_{n+j}],\{\hat{V}_{n+l}\}_{l\ne j}\right]
\right].
\end{equation*}
The set $\mathcal{R}_{j\to 0}^{*}$ is determined by $\{\tilde{p}_l^{(j)}\}_{l\ne j}$. Furthermore, $\tilde{p}_{l}^{(j)}$ is determined by $\hat{V}_{n+l}$ and the unordered set $[\tilde{V}_1,\dots,\tilde{V}_{n},V_{n+j}]$, since $\{\tilde{p}_{l}^{(j)}\}_{l\ne j}$ is invariant after permuting $\{\tilde{V}_{i}\}_{i=1}^n \cup \{V_{n+j}\}$ by definition.
As a result, the set $\mathcal{R}_{j\to 0}^{*}$ is determined by $\{\hat{V}_{n+l}\}_{l\ne j}$ and $[\tilde{V}_1,\dots,\tilde{V}_{n},V_{n+j}]$.
Therefore, conditional on the unordered set $[\tilde{V}_1,\dots,\tilde{V}_n,V_{n+j}]$ and $\{\hat{V}_{n+l}\}_{l\ne j}$, $\mathcal{R}_{j\to0}^{*}$ is a determined set.
Therefore, we have
\begin{align*}
&\quad~\mathbb{E}\left[\frac{\mathbb{I}\left\{\tilde{p}_{j}^{*}\leq \alpha|\mathcal{R}_{j\to0}^{*}|/m\right\}}{\max\{1,|\mathcal{R}_{j\to0}^{*}|\}}\right]\\
&=\mathbb{E}\left[\frac{1}{\max\{1,|\mathcal{R}_{j\to0}^{*}|\}}\mathbb{P}\left(\tilde{p}_j^*\leq \frac{\alpha|\mathcal{R}^*_{j\to0}|}{m}~\big{|}~[\tilde{V}_1,\dots,\tilde{V}_n,V_{n+j}],\{\hat{V}_{n+l}\}_{l\neq j}
\right)\right].
\end{align*}

For any fixed $j\in[m]$, we rewrite the random vector $(\tilde{V}_1,\dots,\tilde{V}_n,V_{n+j})$ by $(Z_1^{j},\dots,Z_{n}^{j},Z_{n+1}^{j})$.
Let $Z_{(1)}^{j}<Z_{(2)}^{j}<\dots<Z_{(n+1)}^{j}$ be the order statistics of $Z_1^{j},\dots,Z_{n}^{j},Z_{n+1}^{j}$.
Define
\begin{equation*}
w(z):=\dfrac{dF_V(z)}{dF_{\tilde{V}}(z)}.
\end{equation*}
We have
\[
\mathbb{P}(Z_{n+1}^j=Z_{(l)}^{j}~|~[\tilde{V}_1,\dots,\tilde{V}_n,V_{n+j}])=\frac{w(Z_{(l)}^{j})}{\sum_{i=1}^{n+1}w(Z_i^{j})}.
\]
Therefore, 
\begin{align*}
\mathbb{P}(\tilde{p}_{j}^*\leq t~|~[\tilde{V}_1,\dots,\tilde{V}_n,V_{n+j}])
&=\mathbb{P}(Z_{n+1}^j\leq Z_{(l)}^{j}~\text{for } l \leq \lfloor t(n+1)\rfloor ~ | ~[\tilde{V}_1,\dots,\tilde{V}_n,V_{n+j}])\\
&=\frac{\sum_{l=1}^{\lfloor t(n+1)\rfloor}w (Z_{(l)}^{j})}{\sum_{i=1}^{n+1}w(Z_i^{j})}.
\end{align*}
By the strong law of large numbers, Slutsky's lemma, and the boundeness of $w(\cdot)$, we have 
\begin{align*}
\frac{1}{n+1}\sum_{i=1}^{n+1}w(Z_i^j)
&=\frac{n}{n+1}\frac{1}{n}\sum_{i=1}^nw(\tilde{V}_i)+\frac{1}{n+1}w(V_{n+j})\\
&\xrightarrow{\mathrm{a.s.}}\mathbb{E}_{\tilde{V}\sim F_{\tilde{V}}}[w(\tilde{V})]\quad \text{as } n\to\infty.
\end{align*}
Note that
\begin{align*}
\mathbb{E}_{\tilde{V} \sim F_{\tilde{V}}}[w(\tilde{V})] = \int w(z) dF_{\tilde{V}}(z) = \int \frac{dF_V(z)}{dF_{\tilde{V}}(z)} dF_{\tilde{V}}(z) = \int dF_V(z) = 1.
\end{align*}
Then we have $\sum_{i=1}^{n+1}w(Z_i^{j})/(n+1) \xrightarrow{\mathrm{a.s.}} 1$ as $n\to \infty$.
In addition, by straightforward calculations, we have
\[
Z_{(\lfloor t(n+1) \rfloor)}^{j} \xrightarrow{\mathrm{a.s.}} F^{-1}_{\tilde{V}}(t),\quad \text{as }n\to\infty.
\]
For numerator, we have
\begin{align*}
&\quad \sup_{t\in(0,\alpha)}\left|\frac{1}{n+1}\sum_{l=1}^{\lfloor t(n+1)\rfloor}w(Z_{(l)}^{j})
-\mathbb{E}_{\tilde{V}\sim F_{\tilde{V}}}\left[w(\tilde{V})\mathbb{I}\left\{\tilde{V}\leq F_{\tilde{V}}^{-1}(t)\right\}\right] \right|\\
&=\sup_{t\in(0,\alpha)}\left| \frac{1}{n+1}\sum_{l=1}^{n+1}w(Z_l^{j})\mathbb{I}\{Z_l^{j}\leq Z^{j}_{(\lfloor t(n+1) \rfloor)} \}
-\mathbb{E}_{\tilde{V}\sim F_{\tilde{V}}}\left[w(\tilde{V})\mathbb{I}\left\{\tilde{V}\leq F_{\tilde{V}}^{-1}(t)\right\}\right] \right|\\
&\leq 
\underbrace{
\sup_{t\in(0,\alpha)}\left|\frac{1}{n+1}\sum_{l=1}^{n+1}w(Z_l^j)\left(\mathbb{I}\{Z_l^j \leq z_{(\lfloor t(n+1)\rfloor)}\}-\mathbb{I}\left\{Z_l^j \leq F_{\tilde{V}}^{-1}(t)\right\}\right)\right|
}_{\text{(I)}}\\
&\quad +
\underbrace{
\sup_{t\in(0,\alpha)}\left|\frac{1}{n+1}\sum_{l=1}^{n+1}w(Z_l^j)\mathbb{I}\left\{Z_l^{j}\leq F^{-1}_{\tilde{V}}(t)\right\}-\mathbb{E}_{\tilde{V}\sim F_{\tilde{V}}}\left[w(\tilde{V})\mathbb{I}\left\{\tilde{V}\leq F_{\tilde{V}}^{-1}(t)\right\}\right]\right|
}_{\text{(II)}}.
\end{align*}
For term (I), by Glivenko-Cantelli theorem,
we have 
\begin{align*}
\text{(I)}
&\overset{(i)}{\leq} M\sup_{t\in(0,\alpha)}\left|\frac{1}{n+1}\sum_{l=1}^{n+1}\left(\mathbb{I}\{Z_l^j\leq z_{(\lfloor t(n+1)\rfloor)}\}-\mathbb{I}\left\{Z_l^j\leq F_{\tilde{V}}^{-1}(t)\right\}\right)\right|\\
&=M\sup_{t\in(0,\alpha)}\left|\frac{\lfloor t(n+1) \rfloor}{n+1}  
-\frac{1}{n+1}\sum_{l=1}^{n+1}
\mathbb{I}\{Z_l\leq F^{-1}_{\tilde{V}}(t)\}\right| \xrightarrow{\mathrm{a.s.}} 0 \quad \text{as } n\to\infty,
\end{align*}
where $(i)$ holds by $\sup w(z)\leq M <\infty$. For term (II), we have
\begin{align*}
\text{(II)} & \leq\frac{1}{n+1}w(Z_{n+1}^j)\mathbb{I}\left\{Z_{n+1}^{j}\leq F_{\tilde{V}}^{-1}(t)\right\}\\
&\quad~+
\sup_{t\in(0,\alpha)}\left|\frac{n}{n+1}\frac{1}{n}\sum_{l=1}^{n}w(Z_l^j)\mathbb{I}\left\{Z_l^{j}\leq F^{-1}_{\tilde{V}}(t)\right\}-\mathbb{E}_{\tilde{V}\sim F_{\tilde{V}}}\left[w(\tilde{V})\mathbb{I}\left\{\tilde{V}\leq F_{\tilde{V}}^{-1}(t)\right\}\right]\right|.
\end{align*}
By the boundedness of $w(\cdot)$ and the Glivenko-Cantelli theorem, we have $\text{(II)}\xrightarrow{\mathrm{a.s.}} 0 $ as $n\to\infty$.
Furthermore, it holds that
\begin{align*}
\mathbb{E}_{\tilde{V}\sim F_{\tilde{V}}}[w(\tilde{V})\mathbb{I}\{\tilde{V}\leq F_{\tilde{V}}^{-1}(t)\}]
&=\int_{-\infty}^{F_{\tilde{V}}^{-1}(t)} w(z) dF_{\tilde{V}}(z) \\
&= \int_{-\infty}^{F_{\tilde{V}}^{-1}(t)} \frac{dF_V(z)}{dF_{\tilde{V}}(z)} dF_{\tilde{V}}(z) \\
&= \int_{-\infty}^{F_{\tilde{V}}^{-1}(t)} dF_V(z) = F_V \left( F_{\tilde{V}}^{-1}(t) \right).
\end{align*}
Thus, we have
\begin{equation*}
\sup_{t\in(0,\alpha)}\left|\frac{1}{n+1}\sum_{l=1}^{\lfloor t(n+1)\rfloor}w(Z_{(l)}^j)
-F_V\left(F_{\tilde{V}}^{-1}(t)\right)\right|\xrightarrow{\mathrm{a.s.}} 0\quad \text{as }n\to\infty.
\end{equation*}
Combining the above results, we have
\begin{equation*}
\sup_{t\in(0,\alpha)}\left|\mathbb{P}\left(\tilde{p}_{j}^*\leq t~|~[\tilde{V}_1,\dots,\tilde{V}_n,V_{n+j}]\right) - F_V \left( F_{\tilde{V}}^{-1}(t) \right) \right| \xrightarrow{\mathrm{a.s.}} 0 \quad \text{as }n \to \infty.
\end{equation*}
Since $\tilde{p}_{j}^*$ is independent of $\{(X_{n+l},Y_{n+l})\}_{l\neq j}$, we have
\begin{equation}\label{asy_leave_p_values}
\sup_{t\in(0,\alpha)}\left| \mathbb{P}\left(\tilde{p}_{j}^*\leq t~|~[\tilde{V}_1,\dots,\tilde{V}_n,V_{n+j}], \{\hat{V}_{n+l}\}_{l\neq j}\right) - F_V \left( F_{\tilde{V}}^{-1}(t) \right)\right|\xrightarrow{\mathrm{a.s.}} 0\quad \text{as }n \to \infty.
\end{equation}

Now we establish the asymptotic behavior of $|\mathcal{R}_{j\to0}^*|$ as $n\to \infty$.
To avoid always introducing the symbol $n$, we rewrite the test data $(X_{n+j},Y_{n+j})$ by $(X^{\text{test}}_j,Y^{\text{test}}_j)$, $j\in[m]$.
Thus, for $l\neq j$, we have
\begin{equation}\label{asy_leave_selection_Set}
\begin{aligned}
\tilde{p}_{l}^{(j)}
&=\frac{\sum_{i=1}^n \mathbb{I}\{\tilde{V}_i < V(X^{\text{test}}_{l},c_l)\}+\mathbb{I}\{V_{n+j}<V(X^{\text{test}}_{l},c_l)\}}{n+1}\\
&=\frac{n}{n+1}\frac{1}{n}\sum_{i=1}^{n}\mathbb{I}\{\tilde{V}_i<V(X^{\text{test}}_{l},c_l)\}+\frac{1}{n+1}\mathbb{I}\{V_{n+j} < V(X^{\text{test}}_{l},c_l)\} \\
&\xrightarrow{\mathrm{a.s.}} F_{\tilde{V}}(V(X^{\text{test}}_{l},c_l)) \quad
\text{as } n\to \infty,
\end{aligned}
\end{equation}
where the last equality holds by the Glivenko-Cantelli theorem, the Slutsky theorem, and the independence between $V(X_{l}^{\text{test}},c_l)$ and $\{\tilde{V}_i\}_{i=1}^{n}$.
This means that for $l\neq j$, the asymptotic behavior of $\tilde{p}_l^{(j)}$ is independent of $j$ and is only related to the random variable $X_{l}^{\text{test}}$ and $c_l$.
Consider $c_j\equiv c$ for all $j\in [m]$.
Define $q_l:=F_{\tilde V}\left(V(X_l^{\mathrm{test}},c)\right)$, $l\in[m]$, and
\begin{equation*}
\bar{\mathcal{R}}_{j\to0}^{*}:=\mathcal{R}(q_1,\ldots,q_{j-1},0,q_{j+1},\ldots,q_m).
\end{equation*}
We first claim that
\begin{equation*}
\bar{\mathcal{R}}_{j\to0}^{*}\overset{d}{=}\lim_{n\to\infty}\mathcal{R}_{j\to0}^*\quad \text{almost surely}.
\end{equation*}
Specifically, since $F_{\tilde V}(\hat V)$ has a continuous
distribution, we have
\begin{equation*}
\mathbb P\left( q_l=\frac{\alpha r}{m}\text{ for some }l,r\in[m]\right)=0.
\end{equation*}
For fixed $m$, it follows almost surely that
\begin{equation*}
\Delta_j:= \min_{\substack{l\neq j,~r\in[m]}} \left| q_l-\frac{\alpha r}{m}
\right|>0.
\end{equation*}
Note that by the finiteness of $m$ and \eqref{asy_leave_selection_Set}, we have
\begin{equation*}
\max_{l\neq j} \left| \tilde p_l^{(j)}-q_l \right| \xrightarrow{\mathrm{a.s.}}0.
\end{equation*}
Hence, eventually almost surely, we have
\[
\mathbb I\left\{ \tilde p_l^{(j)} \le\frac{\alpha r}{m} \right\}
=
\mathbb I\left\{ q_l\le\frac{\alpha r}{m} \right\},
\quad l\neq j,\ r\in[m].
\]
Therefore, we have $\mathcal R_{j\to0}^{*}=\bar{\mathcal R}_{j\to0}^{*}$ eventually almost surely.

The randomness of $\bar{\mathcal{R}}_{j\to0}^*$ comes entirely from the vector $(X^{\text{test}}_1,\dots,X_{j-1}^{\text{test}},X_{j+1}^{\text{test}},\dots,X_m^{\text{test}})$.
For $l\in\{1,\dots,j-1,j+1,\dots,m\}$, we have
\begin{align*}
|\bar{\mathcal{R}}_{l\to0}^{*}|
&=|\mathcal{R}(q_1,\dots,q_{l-1},0,q_{l+1},\dots,q_{j-1},q_{j},q_{j+1},\dots,q_m)|\\
&\overset{(i)}{=}|\mathcal{R}(q_1,\dots,q_{l-1},q_j,q_{l+1},\dots,q_{j-1},0,q_{j+1},\dots,q_{m})|,
\end{align*}
where $(i)$ holds by the property that the size of the rejection set of the BH procedure remains unchanged for any permutation of p-values.
Furthermore, we have
\begin{align*}
&\quad~\mathcal{R}(q_1,\dots,q_{l-1},q_j,q_{l+1},\dots,q_{j-1},0,q_{j+1},\dots,q_{m})\\
&\overset{d}{=}
\mathcal{R}(q_1,\dots,q_{l-1},q_l,q_{l+1},\dots,q_{j-1},0,q_{j+1},\dots,q_{m})\\
&=\bar{\mathcal{R}}^{*}_{j\to0},
\end{align*}
since
\begin{align*}
&\quad~(X^{\text{test}}_1,\dots,X_{l-1}^{\text{test}},X_{j}^{\text{test}},X_{l+1}^{\text{test}},\dots,X_{j-1}^{\text{test}},X_{j+1}^{\text{test}},\dotsm X_m^{\text{test}})\\
&\overset{d}{=}(X^{\text{test}}_1,\dots,X_{l-1}^{\text{test}},X_{l}^{\text{test}},X_{l+1}^{\text{test}},\dots,X_{j-1}^{\text{test}},X_{j+1}^{\text{test}},\dotsm X_m^{\text{test}}).
\end{align*}
Therefore, we have
\begin{equation}\label{same_random_selection_set}
|\bar{\mathcal{R}}^*_{j\to0}|\overset{d}{=}|\bar{\mathcal{R}}_{l\to0}^*|
\end{equation}

For the given $\alpha\in(0,1)$, let $\bar{\mathcal{R}}_{\text{cs}}^{*}$ be the random selection set of BH procedure with inputs $(0,q_2,\dots,q_m)$.
Based on the above results, we have 
\begin{align*}
\limsup_{n\to\infty}\text{FDR}
&\leq \limsup_{n\to \infty}\sum_{j=1}^{m} \mathbb{E}\left[\frac{1}{\max\{1,|\mathcal{R}_{j\to0}^{*}|\}}\mathbb{P}\left(\tilde{p}_j^*\leq \frac{\alpha|\mathcal{R}^*_{j\to0}|}{m}~\big{|}~[\tilde{V}_1,\dots,\tilde{V}_n,V_{n+j}],\{\hat{V}_{n+l}\}_{l\neq j}
\right)\right]\\
&\overset{(i)}{\leq} \sum_{j=1}^{m} \mathbb{E}\left[\limsup_{n\to \infty}\frac{1}{\max\{1,|\mathcal{R}_{j\to0}^{*}|\}}\mathbb{P}\left(\tilde{p}_j^*\leq \frac{\alpha|\mathcal{R}^*_{j\to0}|}{m}~\big{|}~[\tilde{V}_1,\dots,\tilde{V}_n,V_{n+j}],\{\hat{V}_{n+l}\}_{l\neq j}
\right)\right]\\
&\overset{(ii)}{=}\sum_{j=1}^{m}\mathbb{E}\left[\frac{1}{\max\{1,|\bar{\mathcal{R}}^*_{j\to0}|\}} F_V F^{-1}_{\tilde{V}}\left(\frac{\alpha |\bar{\mathcal{R}}^*_{j\to0}|}{m}\right) \right]\\
&\overset{(iii)}{=}m\mathbb{E}\left[\frac{1}{|\bar{\mathcal{R}}^*_{\text{cs}}|} F_V F^{-1}_{\tilde{V}}\left(\frac{\alpha |\bar{\mathcal{R}}^*_{\text{cs}}|}{m}\right) \right],
\end{align*}
where $(i)$ holds by the reverse Fatou lemma, $(ii)$ holds by \eqref{asy_leave_p_values}, \eqref{asy_leave_selection_Set}, and the Slutsky theorem, and $(iii)$ holds by \eqref{same_random_selection_set} and the fact that $|\bar{\mathcal{R}}^*_{\text{cs}}|\geq1$ (since one of the p-values equals to 0). This completes the proof.

\end{proof}

\subsection{Proof of Proposition~\ref{additive_noise_cs}}

\begin{proof}
We fix any $v$ satisfying $v\leq F^{-1}_{\tilde{V}}(\alpha)$ and $\alpha$ satisfying $\alpha < F_{\tilde{V}}(V(x, m(x)))$ for all $x \in \mathcal{X}$.
We define the generalized boundary by
\begin{equation*}
y_v(x)=\sup\{y\in\mathcal{Y}: V(x,y)\leq v\}, \quad x\in\mathcal{X}.
\end{equation*}
By the monotonicity of $V$, we then have 
\begin{align*}
y_v(x)
\leq \sup\{{y\in\mathcal{Y}}:V(x,y)\leq F_{\tilde{V}}^{-1}(\alpha)\}=\sup\{y\in\mathcal{Y}: F_{\tilde{V}}(V(x,y))\leq \alpha\}.
\end{align*}
We note that, for $y=m(x)$, we have $
F_{\tilde{V}}(V(x,m(x)))>\alpha$.
Since $F_{\tilde{V}}$ is non-decreasing, we have
\begin{equation*}
y_v(x)<m(x),\quad x\in\mathcal{X}.
\end{equation*}

Since $f_{Y|X}$ is continuous, $\mathbb{P}(Y = y_v(x) \mid X=x) = 0$. Therefore, we have
\begin{equation}\label{prop:equ}
\mathbb{P}(V(X,Y)\leq v ~|~X=x)=\mathbb{P}(Y\leq y_v(x)~|~ X=x)=F_{Y|X=x}(y_v(x)).
\end{equation}
By the definition of additive noise, we have
\begin{align*}
\mathbb{P}(\tilde{V}\leq v~|~X=x, \eta=\xi)&=\mathbb{P}(V(x,Y+\xi)\leq v~|~X=x,\eta=\xi)\\
&\overset{(i)}=\mathbb{P}(V(x,Y+\xi)\leq v~|~X=x)\\
&\overset{(ii)}{=}F_{Y|X=x}(y_v(x)-\xi),
\end{align*}
where $(i)$ holds by the independence of the noise and $(ii)$ holds by \eqref{prop:equ} and the definition of $y_v(x)$. Marginalizing over $\eta$ yields that
\begin{align*}
\mathbb{P}(\tilde{V}\leq v~|~X=x)
&=\int_{-\infty}^{\infty} F_{Y|X=x}(y_v(x)-\xi)f_{\eta}(\xi) d\xi \\
&\overset{(i)}=\int_{0}^{\infty} \left[F_{Y|X=x}(y_v(x)-\xi)+ F_{Y|X=x}(y_v(x)+\xi)\right]f_{\eta}(\xi) d\xi,
\end{align*}
where $(i)$ holds by the symmetry of $\eta$. The conditional CDF of the clean score can be rewritten as
\[
\mathbb{P}(V \le v \mid X=x) =F_{Y|X=x}(y_v(x))= \int_{0}^{\infty} 2F_{Y|X=x}(y_v(x)) f_{\eta}(\xi) d\xi.
\]
Let
\begin{equation*}
\Delta(\xi)=F_{Y|X=x}(y_v(x)-\xi)+ F_{Y|X=x}(y_v(x)+\xi)-2F_{Y|X=x}(y_v(x)).
\end{equation*}
Then we have
\[
\Delta(\xi) = \int_{0}^{\xi} \left[ f_{Y|X}(y_v(x) + t) - f_{Y|X}(y_v(x) - t) \right] dt.
\]
For $t\in(0,\xi]$, we compare the distances from the mode as follows.
We have
\begin{align*}
(y_v(x)-t-m(x))^2-(y_v(x)+t-m(x))^2
=-4t(y_v(x)-m(x))>0,
\end{align*}
since $y_v(x)<m(x)$ for $v\leq F^{-1}_{\tilde{V}}(\alpha)$ and $t>0$.
Therefore, we have $|y_v(x) + t - m(x)| < |y_v(x) - t - m(x)|$. By the strict unimodality of $f_{Y|X=x}$ around $m(x)$, we have $f_{Y|X=x}(y_v(x) + t) > f_{Y|X=x}(y_v(x) - t)$. Consequently, we have $\Delta (\xi)>0$ for all $\xi>0$. Therefore,
\begin{equation*}
\mathbb{P}(\tilde{V}\leq v~|~X=x)>\mathbb{P}(V\leq v ~|~X=x),
\end{equation*}
which further implies that $F_{\tilde{V}}(v) >F_V(v)$ for any $v\leq F^{-1}_{\tilde{V}}(\alpha)$. Since $\alpha|\bar{\mathcal{R}}^*_{\text{cs}}|/m\leq\alpha$, we have $F^{-1}_{\tilde{V}}(\alpha|\bar{\mathcal{R}}^*_{\text{cs}}|/m)\leq F^{-1}_{\tilde{V}}(\alpha)$.
Therefore, we have
\begin{equation*}
F_VF_{\tilde{V}}^{-1}\left(\frac{\alpha|\bar{\mathcal{R}}^*_{\text{cs}}|}{m}\right)<F_{\tilde{V}}F_{\tilde{V}}^{-1}\left(\frac{\alpha|\bar{\mathcal{R}}^*_{\text{cs}}|}{m}\right)=\frac{\alpha|\bar{\mathcal{R}}^*_{\text{cs}}|}{m}.
\end{equation*}
This completes the proof by
\begin{equation*}
\limsup_{n\to \infty}\text{FDR}\leq  m\mathbb{E}\left[\frac{1}{|\bar{\mathcal{R}}^*_{\text{cs}}|} F_V F^{-1}_{\tilde{V}}\left(\frac{\alpha |\bar{\mathcal{R}}^*_{\text{cs}}|}{m}\right) \right]<\alpha.
\end{equation*}
\end{proof}
\subsection{Proof of Theorem \ref{thm:psp_general}}
\begin{proof}
Since $\hat{f}(\cdot)$ and $\{\mu^{\text{psp}}_k(\cdot)\}_{k\in[K]}$ are viewed as fixed functions, the score $\mu_{\hat{Y}_i}^{\text{psp}}(X_i)$ is a function of the random variable $X_i$.
Let $V_i=V(X_i)=-\mu_{\hat{Y}_i}^{\text{psp}}(X_i)$, $i\in[n+m]$. Denote the random variable $E$ by $E=\mathbb{I}\{\hat{f}(X)\neq Y\}$ with $\mathbb{P}(E=1)=\theta$. Correspondingly, let $\tilde{E}=\mathbb{I}\{\hat{f}(X)\neq \tilde{Y}\}$ with $\mathbb{P}(\tilde{E}=1)=\tilde{\theta}>0$. We have
\begin{equation*}
F_{V|E=1}(v)=\mathbb{P}(V(X)\leq v\mid E=1),\quad
F_{V|\tilde{E}=1}(v)=\mathbb{P}(V(X)\leq v\mid \tilde{E}=1).
\end{equation*}
Let $N_n=\sum_{i=1}^{n}\tilde{E}_i$ and rewrite $\{V_i:\tilde{E}_i=1,i\in[n]\}$ as $\{\tilde{Z}_1,\dots,\tilde{Z}_{N_n}\}$. Then $\{\tilde{p}_j^{\text{psp}}\}_{j=1}^{m}$ and $\tilde{\theta}_n$ can be rewritten as
\begin{equation*}
\tilde{p}_j^{\text{psp}}
=
\frac{1+\sum_{i=1}^{N_n}\mathbb{I}\{\tilde{Z}_i\leq V_{n+j}\}}{1+N_n},
\quad j\in[m],
\quad
\tilde{\theta}_n=\frac{1+N_n}{1+n}.
\end{equation*}
The PSP procedure is equivalent to running the BH procedure with inputs $\{\tilde{p}_j^{\text{psp}}\}_{j=1}^{m}$ at level $\alpha_n^*=\alpha/\tilde{\theta}_n$. Let $\mathcal{R}_{\gamma}(a_1,\dots,a_m)$ denote the rejection set obtained by applying the BH procedure at level $\gamma$ to $(a_1,\dots,a_m)$, and write
\begin{equation*}
\mathcal{R}
:=
\mathcal{R}_{\alpha_n^*}
\bigl(\tilde{p}_1^{\text{psp}},\dots,\tilde{p}_m^{\text{psp}}\bigr).
\end{equation*}
We have
\begin{align*}
\text{FDR}_{\text{o}}
&=\mathbb{E}\left[ \frac{\sum_{j=1}^m E_{n+j} \mathbb{I}\{j \in \mathcal{R}\}}{\max\{1, |\mathcal{R}|\}} \right]\\
&=\sum_{j=1}^{m}\mathbb{E}\left[\frac{\mathbb{I}\{j\in\mathcal{R}\}}{\max\{1, |\mathcal{R}|\}} \Big{|}~E_{n+j}=1\right]\mathbb{P}(E_{n+j}=1)+\sum_{j=1}^m 0\cdot \mathbb{P}(E_{n+j}=0)\\
&=\theta\sum_{j=1}^{m}\mathbb{E}\left[\frac{\mathbb{I}\{j\in\mathcal{R}\}}{\max\{1, |\mathcal{R}|\}} \Big{|}~E_{n+j}=1\right].
\end{align*}
Define the $\sigma$-algebra
\begin{equation*}
\mathcal{E}_{\text{cal}}
:=
\sigma(\tilde{E}_1,\dots,\tilde{E}_n).
\end{equation*}
Conditional on $\mathcal{E}_{\text{cal}}$, $N_n$ is fixed and
$\tilde{Z}_1,\dots,\tilde{Z}_{N_n}\overset{\mathrm{i.i.d.}}{\sim}F_{V|\tilde{E}=1}$.
For each $j\in[m]$, let
\begin{equation*}
\mathcal{R}_{j\to0}
:=
\mathcal{R}_{\alpha_n^*}
\bigl(
\tilde{p}_1^{\text{psp}},\dots,
\tilde{p}_{j-1}^{\text{psp}},0,
\tilde{p}_{j+1}^{\text{psp}},\dots,
\tilde{p}_m^{\text{psp}}
\bigr).
\end{equation*}
Conditional on $\mathcal{E}_{\text{cal}}$, the level $\alpha_n^*$ is fixed. On the event $\{j\in\mathcal{R}\}$, the step-up property of the BH procedure implies that replacing $\tilde{p}_j^{\text{psp}}$ by $0$ does not change the rejection set, so $\mathcal{R}_{j\to0}=\mathcal{R}$. Consequently,
\begin{equation*}
\frac{\mathbb{I}\{j\in\mathcal{R}\}}
{\max\{1,|\mathcal{R}|\}}
\leq
\frac{
\mathbb{I}\left\{
\tilde{p}_j^{\text{psp}}
\leq
|\mathcal{R}_{j\to0}|\alpha_n^*/m
\right\}}
{|\mathcal{R}_{j\to0}|},
\end{equation*}
where $|\mathcal{R}_{j\to0}|\geq1$ because its $j$th input equals zero. Therefore, we have
\begin{equation}\label{psp_fdr_reduction}
\text{FDR}_{\text{o}}
\leq
\theta\sum_{j=1}^{m}
\mathbb{E}\left[
\frac{
\mathbb{I}\left\{
\tilde{p}_j^{\text{psp}}
\leq
|\mathcal{R}_{j\to0}|\alpha_n^*/m
\right\}}
{|\mathcal{R}_{j\to0}|}
\,\Big|\,
E_{n+j}=1
\right].
\end{equation}

Define the empirical distribution function
\begin{equation*}
\tilde{F}_n(v)
:=
\frac{1+\sum_{k=1}^{N_n}\mathbb{I}\{\tilde{Z}_k\leq v\}}
{N_n+1},
\end{equation*}
so that $\tilde{p}_j^{\text{psp}}=\tilde{F}_n(V_{n+j})$. Let
\begin{equation*}
\mathcal{V}_{-j}:=\{V_{n+l}\}_{l\neq j},
\quad
\tilde{\mathcal{Z}}_{\text{cal}}
:=
\{\tilde{Z}_1,\dots,\tilde{Z}_{N_n}\},
\end{equation*}
and define
\begin{equation*}
\mathcal{G}_j
:=
\sigma\bigl(
\mathcal{E}_{\text{cal}},
\tilde{\mathcal{Z}}_{\text{cal}},
\mathcal{V}_{-j}
\bigr).
\end{equation*}
For every $l\neq j$, the value of $\tilde{p}_l^{\text{psp}}$ is determined by $V_{n+l}$ and $\tilde{\mathcal{Z}}_{\text{cal}}$. Hence $\mathcal{R}_{j\to0}$, $\alpha_n^*$, and $\tilde{F}_n$ are $\mathcal{G}_j$-measurable.
For $t\in\mathbb{R}$, define the upper generalized inverse
\begin{equation*}
\tilde{F}_{n,+}^{-1}(t)
:=
\inf\{v\in\mathbb{R}:\tilde{F}_n(v)>t\},
\end{equation*}
with $\inf\varnothing=+\infty$. Since $F_{V|E=1}$ is continuous, conditional on $\mathcal{G}_j$ and $E_{n+j}=1$ we have
\begin{align*}
&\mathbb{P}\left(
\tilde{p}_j^{\text{psp}}
\leq
\frac{|\mathcal{R}_{j\to0}|\alpha_n^*}{m}
\,\Big|\,
\mathcal{G}_j,E_{n+j}=1
\right)\\
&=
\mathbb{P}\left(
V_{n+j}
<
\tilde{F}_{n,+}^{-1}
\left(
\frac{|\mathcal{R}_{j\to0}|\alpha_n^*}{m}
\right)
\,\Big|\,
\mathcal{G}_j,E_{n+j}=1
\right)\\
&=
F_{V|E=1}\left(
\tilde{F}_{n,+}^{-1}
\left(
\frac{|\mathcal{R}_{j\to0}|\alpha_n^*}{m}
\right)
\right),
\end{align*}
where the last equality holds because, conditional on $E_{n+j}=1$, the variable $V_{n+j}$ is independent of $\mathcal{G}_j$ and has distribution function $F_{V|E=1}$. By applying the tower property to \eqref{psp_fdr_reduction}, we have
\begin{equation*}
\text{FDR}_{\text{o}}
\leq
\theta\sum_{j=1}^{m}
\mathbb{E}\left[
\frac{1}{|\mathcal{R}_{j\to0}|}
F_{V|E=1}\left(
\tilde{F}_{n,+}^{-1}
\left(
\frac{|\mathcal{R}_{j\to0}|\alpha_n^*}{m}
\right)
\right)
\Big|E_{n+j}=1\right].
\end{equation*}
Since the integrand is $\mathcal{G}_j$-measurable and $\mathcal{G}_j$ is independent of $E_{n+j}$, we have
\begin{equation*}\label{decop}
\text{FDR}_{\text{o}}
\leq
\theta\sum_{j=1}^{m}
\mathbb{E}\left[
\frac{1}{|\mathcal{R}_{j\to0}|}
F_{V|E=1}\left(
\tilde{F}_{n,+}^{-1}
\left(
\frac{|\mathcal{R}_{j\to0}|\alpha_n^*}{m}
\right)
\right)
\right].
\end{equation*}

By the strong law of large numbers, Slutsky's theorem, and the continuous mapping theorem, we have
\begin{equation}\label{psp_alpha_limit}
\alpha_n^*
=
\alpha\frac{n+1}{N_n+1}=\alpha \frac{1}{N_n/(n+1)+1/(n+1)}
\xrightarrow{\mathrm{a.s.}}
\frac{\alpha}{\tilde{\theta}}
=:\tilde{\alpha},
\qquad n\to\infty.
\end{equation}
Moreover, the Glivenko--Cantelli theorem gives
\begin{equation}\label{psp_cdf_limit}
\sup_{v\in\mathbb{R}}
\left|
\tilde{F}_n(v)-F_{V|\tilde{E}=1}(v)
\right|
\xrightarrow{\mathrm{a.s.}}0.
\end{equation}
In particular, for every $l\in[m]$,
\begin{equation}\label{psp_pvalue_limit}
\tilde{p}_l^{\text{psp}}
=
\tilde{F}_n(V_{n+l})
\xrightarrow{\mathrm{a.s.}}
q_l
:=
F_{V|\tilde{E}=1}(V_{n+l}).
\end{equation}
For each $j\in[m]$, let
\begin{equation*}
\bar{\mathcal{R}}_{j\to0}
:=
\mathcal{R}_{\tilde{\alpha}}
\bigl(
q_1,\dots,q_{j-1},0,q_{j+1},\dots,q_m
\bigr).
\end{equation*}
Since $F_{V|\tilde{E}=1}(V)$ has a continuous distribution,
\begin{equation*}
\mathbb{P}\left(
q_l=\frac{r\tilde{\alpha}}{m}
\text{ for some }l,r\in[m]
\right)=0.
\end{equation*}
Hence, for fixed $m$, almost surely,
\begin{equation*}
\Delta_j
:=
\min_{\substack{l\neq j\\r\in[m]}}
\left|
q_l-\frac{r\tilde{\alpha}}{m}
\right|>0.
\end{equation*}
Combining \eqref{psp_alpha_limit}, \eqref{psp_pvalue_limit}, and the finiteness of $m$, we obtain, eventually almost surely,
\begin{equation*}
\mathbb{I}\left\{
\tilde{p}_l^{\text{psp}}
\leq
\frac{r\alpha_n^*}{m}
\right\}
=
\mathbb{I}\left\{
q_l
\leq
\frac{r\tilde{\alpha}}{m}
\right\},
\qquad l\neq j,\quad r\in[m].
\end{equation*}
Therefore,
\begin{equation*}\label{psp_rejection_stability}
\mathcal{R}_{j\to0}
=
\bar{\mathcal{R}}_{j\to0}
\qquad\text{eventually almost surely}.
\end{equation*}
In particular,
\begin{equation}\label{threshold_psp}
\frac{\alpha_n^*|\mathcal{R}_{j\to0}|}{m}
\xrightarrow{\mathrm{a.s.}}
\frac{\alpha|\bar{\mathcal{R}}_{j\to0}|}{m\tilde{\theta}}.
\end{equation}

Since $|\bar{\mathcal{R}}_{j\to0}|\geq1$ and $\alpha<\tilde{\theta}$, we have
\begin{equation*}
0<
\frac{\alpha}{m\tilde{\theta}}
\leq
\frac{\alpha|\bar{\mathcal{R}}_{j\to0}|}{m\tilde{\theta}}
\leq
\frac{\alpha}{\tilde{\theta}}
<1.
\end{equation*}
Let
\begin{equation*}
a:=\frac{\alpha}{2m\tilde{\theta}},
\qquad
b:=\frac{1+\alpha/\tilde{\theta}}{2}.
\end{equation*}
By \eqref{threshold_psp}, for all sufficiently large $n$,
$\alpha_n^*|\mathcal{R}_{j\to0}|/m\in[a,b]$ almost surely. Since $F_{V|\tilde{E}=1}$ is continuous and strictly increasing on its support, \eqref{psp_cdf_limit} implies
\begin{equation}\label{psp_quantile_limit}
\sup_{t\in[a,b]}
\left|
\tilde{F}_{n,+}^{-1}(t)
-
F_{V|\tilde{E}=1}^{-1}(t)
\right|
\xrightarrow{\mathrm{a.s.}}0.
\end{equation}
Combining \eqref{threshold_psp}, \eqref{psp_quantile_limit}, and the continuity of $F_{V|\tilde{E}=1}^{-1}$ on $[a,b]$,
by continuous mapping theorem, we obtain
\begin{equation*}
\tilde{F}_{n,+}^{-1}
\left(
\frac{|\mathcal{R}_{j\to0}|\alpha_n^*}{m}
\right)
\xrightarrow{\mathrm{a.s.}}
F_{V|\tilde{E}=1}^{-1}
\left(
\frac{\alpha|\bar{\mathcal{R}}_{j\to0}|}{m\tilde{\theta}}
\right).
\end{equation*}
Since $F_{V|E=1}$ is continuous, the continuous mapping theorem yields
\begin{equation}\label{psp_integrand_limit}
\frac{1}{|\mathcal{R}_{j\to0}|}
F_{V|E=1}\left(
\tilde{F}_{n,+}^{-1}
\left(
\frac{|\mathcal{R}_{j\to0}|\alpha_n^*}{m}
\right)
\right)
\xrightarrow{\mathrm{a.s.}}
\frac{1}{|\bar{\mathcal{R}}_{j\to0}|}
F_{V|E=1}\left(
F_{V|\tilde{E}=1}^{-1}
\left(
\frac{\alpha|\bar{\mathcal{R}}_{j\to0}|}{m\tilde{\theta}}
\right)
\right).
\end{equation}
The random variables on the left-hand side of \eqref{psp_integrand_limit} are bounded by one. Therefore, by \eqref{decop} and the dominated convergence theorem,
\begin{align*}
\limsup_{n\to\infty}\text{FDR}_{\text{o}}
&\leq
\theta\sum_{j=1}^{m}
\mathbb{E}\left[
\frac{1}{|\bar{\mathcal{R}}_{j\to0}|}
F_{V|E=1}\left(
F_{V|\tilde{E}=1}^{-1}
\left(
\frac{\alpha|\bar{\mathcal{R}}_{j\to0}|}{m\tilde{\theta}}
\right)
\right)
\right].
\end{align*}
Let $\bar{\mathcal{R}}_{\text{psp}}$ be the rejection set obtained by applying the BH procedure at level $\alpha/\tilde{\theta}$ to $(0,q_2,\dots,q_m)$.
Since $q_1,\dots,q_m$ are i.i.d. and the size of the BH rejection set is invariant under permutations of its inputs, we have
\begin{equation*}
|\bar{\mathcal{R}}_{j\to0}|
\overset{d}{=}
|\bar{\mathcal{R}}_{\text{psp}}|,
\qquad j\in[m].
\end{equation*}
Consequently,
\begin{equation*}
\limsup_{n\to\infty}\text{FDR}_{\text{o}}
\leq
m\theta\mathbb{E}\left[
\frac{1}{|\bar{\mathcal{R}}_{\text{psp}}|}
F_{V|E=1}\left(
F_{V|\tilde{E}=1}^{-1}
\left(
\frac{\alpha|\bar{\mathcal{R}}_{\text{psp}}|}{m\tilde{\theta}}
\right)
\right)
\right].
\end{equation*}
This completes the proof.
\end{proof}

\subsection{Proof of Proposition~\ref{random_flip_psp}}
\begin{proof}
For $\tilde{\theta}$, by definition of the random flip and the condition that $\theta<1/2$, we have
\begin{align*}
\tilde{\theta}&=\mathbb{P}(\tilde{E}=1|E=1)\mathbb{P}(E=1)+\mathbb{P}(\tilde{E}=1|E=0)\mathbb{P}(E=0)\\
&=(1-\epsilon)\theta+\epsilon(1-\theta)\\
&=\theta+\epsilon(1-2\theta)>\theta.
\end{align*}
For any $v\leq F_{V|\tilde{E}=1}^{-1}(\alpha/\tilde{\theta})$, we have 
\begin{align*}
F_{V|\tilde{E}=1}(v)
&=\frac{\mathbb{P}(V\leq v, \tilde{E}=1,E=1)+\mathbb{P}(V\leq v,\tilde{E}=1,E=0)}{\mathbb{P}(\tilde{E}=1)}\\
&\overset{(i)}{=}\frac{\mathbb{P}(V\leq v~|~E=1)\mathbb{P}(E=\tilde{E}=1)+\mathbb{P}(V\leq v~|~E=0)\mathbb{P}(\tilde{E}=1,E=0)}{\mathbb{P}(\tilde{E}=1)}\\
&=\frac{(1-\epsilon)\theta}{\tilde{\theta}}\mathbb{P}(V\leq v~|~E=1)+\frac{\epsilon(1-\theta)}{\tilde{\theta}}\mathbb{P}(V\leq v~|~E=0)\\
&=\frac{(1-\epsilon)\theta}{\tilde{\theta}} F_{V|E=1}(v)+\frac{\epsilon(1-\theta)}{\tilde{\theta}}F_{V|E=0}(v)\\
&\overset{(ii)}{>}\frac{(1-\epsilon)\theta+\epsilon(1-\theta)}{\tilde{\theta}}F_{V|E=1}(v)\\
&=F_{V|E=1}(v),
\end{align*}
where $(i)$ holds by Bayes' rule and the fact that $V\perp\!\!\!\perp \tilde{E} ~|~E$, and $(ii)$ holds by the assumption. Therefore, for any $t\leq \alpha/\tilde{\theta}$, we have
\begin{equation*}
F_{V|E=1}\left(F^{-1}_{V|\tilde{E}=1}(t)\right)<F_{V|\tilde{E}=1}\left(F^{-1}_{V|\tilde{E}=1}(t)\right)=t.
\end{equation*}
By Theorem~\ref{thm:psp_general}, we have
\begin{equation*}
\limsup_{n\to\infty}\text{FDR}_{\text{o}}\leq m\theta\mathbb{E}\left[ \frac{1}{|\bar{\mathcal{R}}_{\text{psp}}|}F_{V|E=1}\left(F_{V|\tilde{E}=1}^{-1}\left(\frac{\alpha|\bar{\mathcal{R}}_{\text{psp}}|}{m\tilde{\theta}}\right)\right)\right]<\alpha\frac{\theta}{\tilde{\theta}}.
\end{equation*}
\end{proof}
\subsection{Proof of Theorem~\ref{theorem2}}
We fix $g\in[G]$ and write $\mathcal A=\mathcal A_g$ and $\alpha_{\mathcal A}=\alpha_{\mathcal A_g}$. Since $\hat f$ and $U$ are fixed or trained on an independent dataset, we condition on the training stage and treat them as deterministic.
For $t\in[0,1]$, define
\begin{align*}
R^\infty_{\mathcal A}(t)
&=\sum_{k\in\mathcal A}\mathbb P(\hat Y=k,U\le t),\\
V^\infty_{\mathcal A}(t)
&=\sum_{k\in\mathcal A}\mathbb P(\hat Y=k,Y\ne k,U\le t),\\
T^\infty_{\mathcal A}(t)
&=\sum_{k\in\mathcal A}\mathbb P(\hat Y=Y=k,U\le t).
\end{align*}
Then $R^\infty_{\mathcal A}(t)=V^\infty_{\mathcal A}(t)+T^\infty_{\mathcal A}(t)$. The oracle feasible set for the threshold is rewritten as
\[
\mathcal F_{\mathcal A}
=
\left\{
t\in[0,1]:
R^\infty_{\mathcal A}(t)>0,\ 
\frac{V^\infty_{\mathcal A}(t)}{R^\infty_{\mathcal A}(t)}
\le \alpha_{\mathcal A}
\right\}.
\]
For the test sample, define
\begin{align}
R_{m,\mathcal A}(t)
&=\frac1m\sum_{k\in\mathcal A}\sum_{j=1}^m
\mathbb I\{\hat Y_{n+j}=k,U_{n+j}\le t\},\label{equ:R_m_A}\\
V_{m,\mathcal A}(t)
&=\frac1m\sum_{k\in\mathcal A}\sum_{j=1}^m
\mathbb I\{\hat Y_{n+j}=k,Y_{n+j}\ne k,U_{n+j}\le t\},\label{equ:V_m_A}\\
T_{m,\mathcal A}(t)
&=\frac1m\sum_{k\in\mathcal A}\sum_{j=1}^m
\mathbb I\{\hat Y_{n+j}=Y_{n+j}=k,U_{n+j}\le t\}.\label{equ:T_m_A}
\end{align}
The calibration estimate of the false-selection mass under threshold $t$ is
\[
\widehat V_{n,m,\mathcal A}(t)
=
\sum_{k\in\mathcal A}
\frac{|\mathcal S_k^{\rm test}|}{m}
\left[
\frac{1}{|\mathcal S_k^{\rm cal}|}
\sum_{i\in\mathcal S_k^{\rm cal}}
w_k(X_i)\mathbb I\{\tilde Y_i\ne k,U_i\le t\}
\right].
\]
If $|\mathcal S_k^{\rm cal}|=0$, we let $\hat{V}_{n,m,\mathcal{A}}(t)=0$ for simplicity, since the event $\{|\mathcal S_k^{\rm cal}|=0\}$ occurs only finitely often almost surely due to $\mathbb P(\hat Y=k)>0$, $k\in\mathcal A$ (Assumption \ref{techniqe_condition}). We rewrite the
$\widehat{\mathrm{FDP}}_{\mathcal{A}_g}(t)$ in \eqref{estimate_fdp} as
\[
\widehat{\mathrm{FDP}}_{\mathcal A}^{(n,m)}(t)=
\frac{\widehat V_{n,m,\mathcal A}(t)}
{R_{m,\mathcal A}(t)\vee m^{-1}}.
\]
Define the empirical feasible set as
\[
\widehat{\mathcal F}_{n,m}
:=
\left\{
t\in[0,1]:
\widehat V_{n,m,\mathcal A}(t)
\le
\alpha_{\mathcal A}\{R_{m,\mathcal A}(t)\vee m^{-1}\}
\right\}.
\]
Therefore, the selected threshold $\hat t_{\mathcal A}$ maximizes $R_{m,\mathcal A}(t)$ over $\widehat{\mathcal F}_{n,m}$ whenever this set is nonempty; we will show below that it is eventually nonempty.
We first establish the following lemma.

\begin{lemma}\label{convergence_lemma}
Under the same assumptions of Theorem~\ref{theorem2}, as $n,m\to\infty$, we have
\begin{align}
\sup_{t\in[0,1]}|R_{m,\mathcal A}(t)-R^\infty_{\mathcal A}(t)|&\xrightarrow{\rm a.s.}0,\label{eq:unif_R}\\
\sup_{t\in[0,1]}|V_{m,\mathcal A}(t)-V^\infty_{\mathcal A}(t)|&\xrightarrow{\rm a.s.}0,\label{eq:unif_V}\\
\sup_{t\in[0,1]}|T_{m,\mathcal A}(t)-T^\infty_{\mathcal A}(t)|&\xrightarrow{\rm a.s.}0,\label{eq:unif_T}\\
\sup_{t\in[0,1]}|\widehat V_{n,m,\mathcal A}(t)-V^\infty_{\mathcal A}(t)|&\xrightarrow{\rm a.s.}0.\label{eq:unif_Vhat}
\end{align}
\end{lemma}

\begin{proof}
Consider the function class that
\begin{equation*}
\mathcal{G}_R = \left\{ x \mapsto \sum_{k \in \mathcal{A}} \mathbb{I}\left\{\hat{f}(x) = k, U(x) \le t\right\} : t \in [0,1] \right\}.
\end{equation*}
Since each function in $\mathcal{G}_R$ is a finite sum of a one-dimensional threshold indicator of the form $\mathbb{I}\{U(x)\leq t\}$ multiplied by fixed indicators $\mathbb{I}\{\hat{f}(x)=k\}$, its Vapnik-Chervonenkis (VC) dimension is finite.
Therefore, by the Glivenko-Cantelli theorem, we have
\begin{equation}\label{concergence_1}
\sup_{t \in [0,1]} \left| R_{m,\mathcal{A}}(t) - R_{\mathcal{A}}^\infty(t) \right| \xrightarrow{a.s.} 0 \quad \text{as } m \to \infty.
\end{equation}
Similarly, \eqref{eq:unif_V} and \eqref{eq:unif_T} can also be proved.
It remains only to justify \eqref{eq:unif_Vhat}.
For $\widehat{V}_{n,m,\mathcal{A}}(t)$, we first fix any $k\in\mathcal A$. 
Define
\begin{equation*}
\widehat{V}_{n,m,k}(t) := \frac{\hat{\pi}_k^{\text{test}}}{\hat{\pi}_k^{\text{cal}}} \cdot C_{n,k}(t),
\end{equation*}
where $\hat{\pi}_k^{\text{test}} := |\mathcal{S}_k^{\text{test}}|/m$, $\hat{\pi}_k^{\text{cal}} := |\mathcal{S}_k^{\text{cal}}|/n$, and
\begin{equation*}
C_{n,k}(t) := \frac{1}{n} \sum_{i=1}^n w_k(X_i)\mathbb{I}\{\hat{Y}_i=k\} \mathbb{I}\{\tilde{Y}_i \neq k\} \mathbb{I}\{U_i\leq t\}.
\end{equation*}
Then $\widehat{V}_{n,m,\mathcal{A}}(t)=\sum_{k\in\mathcal{A}}\widehat{V}_{n,m,k}(t)$. 
We first show that for each $k\in\mathcal{A}$, as $n,m\to\infty$, we have
$\sup_{t} |\widehat{V}_{n,m,k}(t) - V_k^{\infty}(t)| \xrightarrow{a.s.} 0$, where 
$V_k^{\infty}(t)=\mathbb{P}(\hat{Y}=k,Y\neq k, U\leq t)$.
For $V_k^\infty(t)$, 
\begin{align*}
\mathbb{E} \left[ w_k(X) \mathbb{I}\{\hat{Y}=k\} \mathbb{I}\{\tilde{Y} \neq k\} \mathbb{I}\{U \le t\} \right]
&\overset{(i)}{=}\mathbb{E}_{X}\left[w_k(X)\mathbb{I}\{\hat{Y}=k\}\mathbb{I}\{U(X)\leq t\}\mathbb{E}\left[\mathbb{I}\{\tilde{Y}\neq k\}|X\right]\right]\\
&\overset{(ii)}{=}\mathbb{E}_{X}\left[\frac{\mathbb{P}(Y\neq k|X)}{\mathbb{P}(\tilde{Y}\neq k|X)}\mathbb{I}\{\hat{Y}=k\}\mathbb{I}\{U(X)\leq t\}\mathbb{P}(\tilde{Y}\neq k|X)\right]\\
&=\mathbb{E}_X\left[\mathbb{P}(Y\neq k|X)\mathbb{I}\{\hat{Y}=k\}\mathbb{I}\{U(X)\leq t\}\right]\\
&\overset{(iii)}{=} \mathbb{P}\left(\hat{Y}=k, Y \neq k, U \le t\right)\\
&=V_k^{\infty}(t),
\end{align*}
where $(i)$ and $(iii)$ hold by the tower property and the fact that $\hat{Y}$ and $U$ are deterministic functions of $X$, and $(ii)$ holds by the definition of $w_k(X)$.
Define the function class
\begin{equation*}
\mathcal{G}_C = \{(x,\tilde{y}) \mapsto w_k(x) \mathbb{I}\{\hat{f}(x)=k\} \mathbb{I}\{\tilde{y} \neq k\} \mathbb{I}\{U(x) \le t\}:t\in[0,1] \},
\end{equation*}
which is formed by taking the product of a VC-subgraph indicator class and a deterministic weight function. Since $w_k(\cdot)$ is uniformly bounded, $\mathcal{G}_C$ is bounded by an envelope. By the Glivenko-Cantelli theorem, as $n\to\infty$, we have
\begin{align}\label{each_class_convergence}
\sup_{t} \left|C_{n,k}(t) - V_k^\infty(t) \right| \xrightarrow{\mathrm{a.s.}}0.
\end{align}
For $\hat{\pi}_k^{\text{test}}$, by the Strong Law of Large Numbers, we have
\begin{align*}
\hat{\pi}_k^{\text{test}}=\frac{1}{m}\sum_{j=1}^{m}\mathbb{I}\{\hat{Y}_{n+j}=k\}\xrightarrow{\mathrm{a.s.}} \mathbb{P}(\hat{Y}=k)=:\pi_k>0, \quad \text{as } m\to \infty.
\end{align*}
Similarly, we have $\hat{\pi}_k^{\text{cal}} \xrightarrow{\mathrm{a.s.}} \pi_k$ as $n\to \infty$. By the Continuous Mapping Theorem, we have
\begin{equation}\label{prop_to_1}
\frac{\hat{\pi}_k^{\text{test}}}{\hat{\pi}_k^{\text{cal}}}\xrightarrow{\mathrm{a.s.}} 1, \quad \text{as } n\to\infty,~m\to \infty.
\end{equation}
Therefore, we have
\begin{align*}
\sup_{t} \left| \widehat{V}_{n,m,k}(t) - V_k^\infty(t) \right|
&=\sup_{t} \left| \frac{\hat{\pi}_k^{\text{test}}}{\hat{\pi}_k^{\text{cal}}} \left( C_{n,k}(t) -V_k^{\infty}(t)\right) + \left(\frac{\hat{\pi}_k^{\text{test}}}{\hat{\pi}_k^{\text{cal}}}-1\right)V_k^{\infty}(t)
\right|\\
&\le \left| \frac{\hat{\pi}_k^{\text{test}}}{\hat{\pi}_k^{\text{cal}}} \right| \cdot \sup_{t} \left| C_{n,k}(t) - V_k^\infty(t) \right| + \left| \frac{\hat{\pi}_k^{\text{test}}}{\hat{\pi}_k^{\text{cal}}} - 1 \right| \cdot \sup_{t} \left| V_k^\infty(t) \right|\\
&\xrightarrow{\mathrm{a.s.}} 0 \quad \text{as } n\to\infty, m\to\infty,
\end{align*}
where the last step follows from \eqref{prop_to_1}, \eqref{each_class_convergence}, and the fact that $\sup_{t} \left| V_k^\infty(t) \right|\leq \pi_k<1$.
Thus,
\begin{align}\label{concergence_2}
\sup_{t} \left| \widehat{V}_{n,m,\mathcal{A}}(t) - V_{\mathcal{A}}^\infty(t) \right|
&=\sup_{t} \left| \sum_{k\in\mathcal{A}}\left(\widehat{V}_{n,m,k}(t) - V_{k}^\infty(t)\right) \right|  \nonumber\\
&\leq \sum_{k\in\mathcal{A}} \sup_{t} \left|\widehat{V}_{n,m,k}(t) - V_{k}^\infty(t) \right| \xrightarrow{\mathrm{a.s.}} 0, \quad \text{as }  n\to\infty, m\to\infty.
\end{align}
\end{proof}

We now prove FDR control. By Assumption~\ref{techniqe_condition}, there exists $t_0\in[0,1]$ such that
\[
R^\infty_{\mathcal A}(t_0)>0,
\qquad
V^\infty_{\mathcal A}(t_0)<\alpha_{\mathcal A}R^\infty_{\mathcal A}(t_0).
\]
Thus, by Lemma~\ref{convergence_lemma},
\[
\frac{\widehat V_{n,m,\mathcal A}(t_0)}
{R_{m,\mathcal A}(t_0)\vee m^{-1}}
\xrightarrow{\rm a.s.}
\frac{V^\infty_{\mathcal A}(t_0)}
{R^\infty_{\mathcal A}(t_0)}
<\alpha_{\mathcal A},\quad \text{as }n\to\infty,m\to\infty.
\]
Hence $t_0\in\widehat{\mathcal F}_{n,m}$ eventually almost surely, which proves that $\widehat{\mathcal F}_{n,m}$ is eventually nonempty. By the definition of $\hat t_{\mathcal{A}}$ and Lemma~\ref{convergence_lemma}, we have
\[
R_{m,\mathcal A}(\hat t_{\mathcal A})
\geq R_{m,\mathcal A}(t_0)
\xrightarrow{\rm a.s.}
R^\infty_{\mathcal A}(t_0)>0.
\]
Consequently, there exists a constant $c_0>0$ such that
\begin{equation}\label{eq:denominator_lower_bound}
R_{m,\mathcal A}(\hat t_{\mathcal A})\ge c_0
\end{equation}
eventually almost surely. Again by the definition of $\hat{t}_{\mathcal{A}}$, we have
\[
\widehat V_{n,m,\mathcal A}(\hat t_{\mathcal A})
\le
\alpha_{\mathcal A}\max\left\{R_{m,\mathcal A}(\hat t_{\mathcal A}), m^{-1}\right\}.
\]
Using \eqref{eq:denominator_lower_bound}, the maximum on the right-hand side equals $R_{m,\mathcal A}(\hat t_{\mathcal A})$ eventually almost surely. Let
\[
\Delta_{n,m}:=\sup_{t\in[0,1]}\left|\widehat{V}_{n,m,\mathcal{A}}(t)-V_{m,\mathcal{A}}(t)\right|.
\]
Therefore, we have
\[
V_{m,\mathcal A}(\hat t_{\mathcal A})
\le
\widehat V_{n,m,\mathcal A}(\hat t_{\mathcal A})+\Delta_{n,m}
\le
\alpha_{\mathcal A}R_{m,\mathcal A}(\hat t_{\mathcal A})+\Delta_{n,m}\quad \text{a.s.}.
\]
Dividing by $R_{m,\mathcal A}(\hat t_{\mathcal A})$ gives
\begin{align*}
\limsup_{n,m\to\infty}
\frac{V_{m,\mathcal A}(\hat t_{\mathcal A})}
{R_{m,\mathcal A}(\hat t_{\mathcal A})}
&\leq
\alpha_{\mathcal A}+ \limsup_{n,m\to\infty} \frac{\Delta_{n,m}}{R_{m,\mathcal{A}}(\hat{t}_{\mathcal{A}})}\\
&\overset{(i)}{\leq}\alpha_{\mathcal{A}}+\limsup_{n,m\to\infty}\frac{\Delta_{n,m}}{c_0}\\
&\overset{(ii)}{\leq} \alpha_{\mathcal{A}},
\end{align*}
where $(i)$ and $(ii)$ hold by \eqref{eq:denominator_lower_bound} and Lemma~\ref{convergence_lemma}, respectively.
The realized FDP satisfies
\[
\mathrm{FDP}_{\mathcal A}
=
\frac{mV_{m,\mathcal A}(\hat t_{\mathcal A})}
{1\vee mR_{m,\mathcal A}(\hat t_{\mathcal A})}.
\]
By \eqref{eq:denominator_lower_bound}, $mR_{m,\mathcal A}(\hat t_{\mathcal A})\to\infty$ eventually, which implies that
\[
\limsup_{n,m\to\infty}\mathrm{FDP}_{\mathcal A}
\le
\alpha_{\mathcal A}
\qquad\text{a.s.}
\]
Since $0\le\mathrm{FDP}_{\mathcal A}\le1$, the reverse Fatou lemma yields
\[
\limsup_{n,m\to\infty}\mathrm{FDR}_{\mathcal A}
=
\limsup_{n,m\to\infty}\mathbb E[\mathrm{FDP}_{\mathcal A}]
\le
\mathbb E\!\left[\limsup_{n,m\to\infty}\mathrm{FDP}_{\mathcal A}\right]
\le
\alpha_{\mathcal A}.
\]

It remains to prove the power statement. 
We rewrite the selected threshold $\hat{t}_{\mathcal{A}}$ under sample sizes $(n,m)$ by $\hat{t}_{\mathcal{A}}^{(n,m)}$.
We first show that every accumulation point of $\{\hat t_{\mathcal A}^{(n,m)}\}$ maximizes $R^\infty_{\mathcal A}$ over $\mathcal F_{\mathcal A}$. Let $\hat t_{\mathcal A}^{(n_\ell,m_\ell)}\to t^{**}$ be any convergent subsequence. By \eqref{eq:denominator_lower_bound} and \eqref{eq:unif_R}, we have $R^\infty_{\mathcal A}(t^{**})>0.$
By the definition of the data-dependent threshold and the uniform convergences established in Lemma~\ref{convergence_lemma}, we have 
\begin{align*}
V_{\mathcal{A}}^{\infty}(t^{**})
&=
\limsup_{\ell \to\infty}\widehat{V}_{n_{\ell},m_{\ell},\mathcal{A}}(\hat{t}_{\mathcal{A}}^{(n_{\ell},m_{\ell})})\\
&\leq \limsup_{\ell \to\infty} \max\{\alpha_{\mathcal{A}}R_{m_{\ell},\mathcal{A}}(\hat{t}_{\mathcal{A}}^{(n_{\ell},m_{\ell})}),\alpha_{\mathcal{A}}/m_{\ell}\}\\
&=\alpha_{\mathcal{A}}R^\infty_{\mathcal A}(t^{**}).
\end{align*}
Therefore, we have $t^{**}\in\mathcal{F}_{\mathcal A}$.
Let
$R^*_{\mathcal A}:=\sup_{t\in\mathcal F_{\mathcal A}}R^\infty_{\mathcal A}(t)$.
We fix any $\delta>0$. By $(3)$ of Assumption~\ref{techniqe_condition}, there exists $t_\delta\in[0,1]$ such that
\[
V^\infty_{\mathcal A}(t_\delta)
<
\alpha_{\mathcal A}R^\infty_{\mathcal A}(t_\delta),
\quad
R^\infty_{\mathcal A}(t_\delta)\ge R^*_{\mathcal A}-\epsilon.
\]
The same argument used for $t_0$ shows that $t_\delta\in\widehat{\mathcal F}_{n,m}$ eventually. Hence
\[
R_{m_\ell,\mathcal A}(\hat t_{\mathcal A}^{(n_\ell,m_\ell)})
\ge
R_{m_\ell,\mathcal A}(t_\delta)
\]
eventually. Letting $\ell\to\infty$ and using \eqref{eq:unif_R}, we obtain
\[
R^\infty_{\mathcal A}(t^{**})
=
\limsup_{l\to\infty}R_{m_\ell,\mathcal A}(\hat t_{\mathcal A}^{(n_\ell,m_\ell)})
\geq \limsup_{l\to\infty}
R_{m_\ell,\mathcal A}(t_\delta)
=
R^\infty_{\mathcal A}(t_\delta)
\ge
R^*_{\mathcal A}-\delta.
\]
Since $\epsilon>0$ is arbitrary and $t^{**}\in\mathcal F_{\mathcal A}$, we have
$R^\infty_{\mathcal A}(t^{**})=R^*_{\mathcal A}.$

Note that the rejection regions $\{U\le t\}$ are nested for different $t$. Therefore, if $t^{**}$ maximizes $R^\infty_{\mathcal A}$ over $\mathcal F_{\mathcal A}$, then it also maximizes $T^\infty_{\mathcal A}$ over $\mathcal F_{\mathcal A}$. Specifically, for any $s\in\mathcal F_{\mathcal A}$, if $s\le t^{**}$, then $T^\infty_{\mathcal A}(s)\le T^\infty_{\mathcal A}(t^{**})$. If $s>t^{**}$, maximality and monotonicity give $R^\infty_{\mathcal A}(s)=R^\infty_{\mathcal A}(t^{**})$, and hence
\[
\mathbb P(\hat Y\in\mathcal A,t^{**}<U\le s)=0.
\]
This also leads to
\[
\mathbb P(\hat Y=Y\in\mathcal A,t^{**}<U\le s)=0,
\]
which further implies $T^\infty_{\mathcal A}(s)=T^\infty_{\mathcal A}(t^{**})$. Therefore, we have
\[
T^\infty_{\mathcal A}(t^{**})
=
\sup_{t\in\mathcal F_{\mathcal A}}T^\infty_{\mathcal A}(t).
\]
Together with the uniform convergence in \eqref{eq:unif_T}, we have
\[
T_{m,\mathcal A}(\hat t_{\mathcal A}^{(n,m)})
\xrightarrow{\rm a.s.}
\sup_{t\in\mathcal F_{\mathcal A}}T^\infty_{\mathcal A}(t)\quad \text{as }n\to\infty,m\to\infty.
\]
By the strong law of large numbers, as $m\to\infty$, we have
\[
D_m:=\frac1m\sum_{j=1}^m\mathbb I\{Y_{n+j}\in\mathcal A\}
\xrightarrow{\rm a.s.}
\mathbb P(Y\in\mathcal A)>0.
\]
Recall that the realized recovery-type power equals
\[
\frac{
\sum_{k\in\mathcal A}\sum_{j=1}^{m}
\mathbb I\{\hat Y_{n+j}=Y_{n+j}=k,U_{n+j}\le\hat t_{\mathcal A}^{(n,m)}\}
}{
1\vee\sum_{j=1}^{m}\mathbb I\{Y_{n+j}\in\mathcal A\}
}
=
\frac{T_{m,\mathcal A}(\hat t_{\mathcal A}^{(n,m)})}
{m^{-1}\vee D_m}.
\]
By the uniform convergence in \eqref{eq:unif_T}, as $n,m\to \infty$, we have
\[
\frac{T_{m,\mathcal A}(\hat t_{\mathcal A}^{(n,m)})}
{m^{-1}\vee D_m}
\xrightarrow{\rm a.s.}
\frac{\sup_{t\in\mathcal F_{\mathcal A}}T^\infty_{\mathcal A}(t)}
{\mathbb P(Y\in\mathcal A)}
=
\sup_{t\in\mathcal F_{\mathcal A}}\mathrm{rPower}^\infty_{\mathcal A}(t).
\]
By the dominated convergence theorem, we have
\[
\lim_{n,m\to\infty}\mathrm{rPower}_{\mathcal A}
=\mathbb{E}\left[\lim_{n,m\to\infty}\frac{T_{m,\mathcal{A}}(\hat t_{\mathcal{A}}^{(n,m)})}{m^{-1}\vee D_m}\right]
=
\sup_{t\in\mathcal F_{\mathcal A}}\mathrm{rPower}^\infty_{\mathcal A}(t).
\]
This completes the proof of Theorem~\ref{theorem2}.

\subsection{Proof of Theorem~\ref{optimal_uncertainty}}
\begin{proof}
Fix $g\in[G]$ and write $\mathcal A=\mathcal A_g$ and $\alpha_{\mathcal{A}}=\alpha_{\mathcal A_g}$. Define
\[
S(X):=\mu_{\hat f(X)}(X)=\mathbb P\{Y=\hat f(X)\mid X\},
\qquad
B(X):=\mathbb I\{\hat f(X)\in\mathcal A\}.
\]
Let $\phi(X)$ denote the thresholding component of a selection rule. Therefore, the actual group-wise selection indicator is $B(X)\phi(X)$. In particular, a score-threshold pair $(U,\tau)$ corresponds to $\phi_{U,\tau}(X)=\mathbb I\{U(X)\le \tau\}$.
Since the denominator $\mathbb P(Y\in\mathcal A)$ in $\mathrm{rPower}^\infty_{\mathcal A}$ is independent of the selection rule, the population problem is equivalent to maximizing
\[
T(\phi):=\mathbb E\left[B(X)S(X)\phi(X)\right]
\]
subject to the oracle FDP constraint. For any measurable $\phi:\mathcal X\to\{0,1\}$, write
\[
V(\phi)=\mathbb E\left[B(X)(1-S(X))\phi(X)\right],
\qquad
R(\phi)=\mathbb E\left[B(X)\phi(X)\right].
\]
Thus the constraint $V(\phi)\le \alpha_{\mathcal{A}} R(\phi)$ is equivalent to
\begin{equation}\label{eq:t12_constraint}
C(\phi):=\mathbb E\{B(X)(1-\alpha_{\mathcal{A}}-S(X))\phi(X)\}\le0.
\end{equation}
It is sufficient to solve the more general problem in which $\phi$ ranges over all measurable $\{0,1\}$-valued rules.
Let $M$ be the finite measure on $[0,1]$ defined by
\[
M(A):=\mathbb P\{B(X)=1,\ S(X)\in A\},\quad A\in\mathcal{B}([0,1]).
\]
For a measurable rule $\phi$, define
\begin{equation}\label{conditional_exp}
a_{\phi}(s):=\mathbb E\left[\phi(X)\mid B(X)=1,S(X)=s\right],\quad s\in[0,1].
\end{equation}
Then $0\le a_{\phi}\le1$ and by conditioning on $(B(X),S(X))$, we have
\begin{equation}\label{eq:t12_projection}
T(\phi)=\int s a_{\phi}(s)\,dM(s),
\quad
C(\phi)=\int (1-\alpha_{\mathcal{A}}-s)a_{\phi}(s)\,dM(s).
\end{equation}
Therefore, the supremum over all measurable rules $\phi$ is bounded above by the value of the relaxed problem
\begin{equation}\label{eq:t12_relaxed}
\sup_{a}
\int s a(s)\,dM(s)
\quad \text{subject to }
\int (1-\alpha_{\mathcal{A}}-s)a(s)\,dM(s)\le0,
\end{equation}
where the supremum is taken over all measurable functions $a:[0,1]\to[0,1]$. 

We now solve \eqref{eq:t12_relaxed}. Let
\[
A_+:=\{s\in[0,1]:s\ge 1-\alpha_{\mathcal{A}}\},
\qquad
A_-:=\{s\in[0,1]:s<1-\alpha_{\mathcal{A}}\}.
\]
On $A_+$, the coefficient $1-\alpha_{\mathcal{A}}-s$ in the constraint is non-positive. Hence, if a feasible function $a(\cdot)$ satisfies $a(s)<1$ on a subset of $A_+$ with positive $M$-measure, replacing $a(s)$ by $1$ on that subset preserves feasibility and increases the objective. Thus an optimizer of \eqref{eq:t12_relaxed} must satisfy $a(s)=1$ for $M$-almost every $s\in A_+$.
The remaining problem on $A_-$ is
\begin{equation}\label{eq:t12_positive_cost}
\sup_{0\le a\le1}
\int_{A_-} s a(s)\,dM(s)\quad \text{subject to }
\int_{A_-}(1-\alpha_{\mathcal{A}}-s)a(s)\,dM(s)
\le b,
\end{equation}
where $
b:=-\int_{A_+}(1-\alpha_{\mathcal{A}}-s)\,dM(s)\ge0$.
On $A_-$, the cost $1-\alpha_{\mathcal{A}}-s$ is strictly positive, and the value-to-cost ratio is
\[
\rho(s):=\frac{s}{1-\alpha_{\mathcal{A}}-s}.
\]
We note that the function $\rho(s)$ is strictly increasing on $s\in[0,1-\alpha_{\mathcal{A}})$. By the Neyman--Pearson lemma applied to the value measure $sdM(s)$ and the cost measure $(1-\alpha_{\mathcal{A}}-s)dM(s)$ on $A_-$, the supremum in \eqref{eq:t12_positive_cost} is attained by selecting an upper level set of $\rho$, or equivalently an upper level set of $s$. Hence there exists $c\in[0,1-\alpha_{\mathcal{A}}]$ such that an optimal relaxed rule is
\[
a^*(s)=\mathbb I\{s\ge c\},\qquad M\text{-a.e.}
\]
Indeed, if the Neyman--Pearson rule randomizes on the boundary $s=c$, this randomization is immaterial here because $U^*\in\mathcal U$ implies that
\[
M(\{c\})=\mathbb P\{B(X)=1,S(X)=c\}
=\mathbb P\{B(X)=1,U^*(X)=1-c\}=0.
\]
Combining this with $a^*(s)=1$ on $A_+$, which is already included in the event $\{s\ge c\}$, the solution of \eqref{eq:t12_relaxed} is exactly of the form $a^*(s)=\mathbb I\{s\ge c\}$.

Write $\phi_c(X):=\mathbb I\{S(X)\ge c\}$. We have $a_{\phi_c}(s)=a^*(s)$ for $M$-almost every $s$, where $a_{\phi_c}(s)$ follows the formulation of \eqref{conditional_exp}.
Hence, by \eqref{eq:t12_projection}, $\phi_c$ attains the same value as the relaxed optimizer $a^*$. Since the relaxed problem is an upper bound for the original optimization over measurable $\{0,1\}$-valued rules, $\phi_c$ is also optimal for the original problem.
Recall that
\[
U^*(X):=1-S(X)=1-\mu_{\hat f(X)}(X).
\]
Then, for $\tau_c=1-c$, we have
\[
\phi_c(X)B(X)
=
\mathbb I\{\hat f(X)\in\mathcal A,\ U^*(X)\le \tau_c\}.
\]
Therefore the oracle optimum over all measurable rules is attained by thresholding $U^*$.
Since every score-threshold rule is a measurable rule, we have
\[
\sup_{U\in\mathcal U}\sup_{\tau\in\mathcal F_{\mathcal A}(U)}
\mathrm{rPower}_{\mathcal A}^{\infty}(\tau;U)
\le
\sup_{\tau\in\mathcal F_{\mathcal A}(U^*)}
\mathrm{rPower}_{\mathcal A}^{\infty}(\tau;U^*).
\]
The reverse inequality is obvious due to $U^*\in\mathcal U$. Therefore, we have
\[
\sup_{\tau\in\mathcal F_{\mathcal A}(U^*)}
\mathrm{rPower}_{\mathcal A}^{\infty}(\tau;U^*)
=
\sup_{U\in\mathcal U}\sup_{\tau\in\mathcal F_{\mathcal A}(U)}
\mathrm{rPower}_{\mathcal A}^{\infty}(\tau;U).
\]
By applying Theorem~\ref{theorem2} to score $U^*$, we have
\[
\lim_{n,m\to\infty}\mathrm{rPower}_{\mathcal A}(U^*) = \sup_{\tau\in\mathcal F_{\mathcal A}(U^*)} \mathrm{rPower}_{\mathcal A}^{\infty}(\tau;U^*).
\]
This completes the proof.

\end{proof}

\subsection{Proof of Proposition~\ref{prop:estimated_weights}}
\begin{proof}
We have
\begin{align*}
\mathbb{P}(\tilde{Y}=k|X=x)
&=\sum_{j=1}^K \mathbb{P}(\tilde{Y}=k,Y=j|X=x)\\
&=\sum_{j=1}^K\mathbb{P}(\tilde{Y}=k|Y=j,X=x)\mathbb{P}(Y=j|X=x)\\
&\overset{(1)}{=}\sum_{j=1}^K \mathbb{P}(\tilde{Y}=k|Y=j)\mathbb{P}(Y=j|X=x)\\
&=\sum_{j=1}^KT_{jk}\mu_j(x),
\end{align*}
where $(1)$ holds by the assumption that $\tilde Y\perp X\mid Y$. Therefore, we have 
\begin{equation*}
(\mathbb{P}(\tilde{Y}=1~|~X=x),\dots,\mathbb{P}(\tilde{Y}=K~|~X=x))=(\mu_1(x),\dots,\mu_K(x))\mathbf{T}.
\end{equation*}
Consequently, we have
\begin{equation*}
w_k(x)=\frac{\mathbb{P}(Y\neq k|X=x)}{\mathbb{P}(\tilde{Y}\neq k|X=x)}=\frac{1- \mu_k(x)}{1-\sum_{j=1}^KT_{jk}\mu_j(x)},
\end{equation*}
which completes the proof.
\end{proof}
\subsection{Proof of Theorem~\ref{robust_rcs}}
\begin{proof}
Consider any fixed $g\in[G]$ and write $\mathcal A=\mathcal A_g$, $\alpha_{\mathcal{A}}=\alpha_{\mathcal A_g}$. 
Let $\hat t^{\mathrm{est}}_{\mathcal A}$ be the threshold selected by \eqref{threshold} after replacing $w_k(\cdot)$ with $\hat w_{n,m,k}(\cdot)$. The corresponding FDP and FDR are denoted by $\mathrm{FDP}_{\mathcal A}$ and $\mathrm{FDR}_{\mathcal A}$, respectively.
For $t\in[0,1]$, let $R_{m,\mathcal{A}}(t)$ and $V_{m,\mathcal{A}}(t)$ be defined in \eqref{equ:R_m_A} and \eqref{equ:V_m_A}, respectively. Let
\begin{align*}
\widehat V^{\mathrm{oracle}}_{n,m,\mathcal A}(t)
&:=
\sum_{k\in\mathcal A}
\frac{|\mathcal S_k^{\mathrm{test}}|}{m}
\left[
\frac1{|\mathcal S_k^{\mathrm{cal}}|}
\sum_{i\in\mathcal S_k^{\mathrm{cal}}}
w_k(X_i)\mathbb I\{\tilde Y_i\ne k,U_i\le t\}
\right],\\
\widehat V^{\mathrm{est}}_{n,m,\mathcal A}(t)
&:=
\sum_{k\in\mathcal A}
\frac{|\mathcal S_k^{\mathrm{test}}|}{m}
\left[
\frac1{|\mathcal S_k^{\mathrm{cal}}|}
\sum_{i\in\mathcal S_k^{\mathrm{cal}}}
\hat w_{n,m,k}(X_i)\mathbb I\{\tilde Y_i\ne k,U_i\le t\}
\right].
\end{align*}
It suffices to consider the non-trivial case $|\mathcal{S}_k^{\text{cal}}|>0$ since $\mathbb{P}(\hat{Y}=k)>0$ for $k\in \mathcal{A}$.
By the proof of Theorem~\ref{theorem2}, we have
\begin{equation}\label{eq:oracle_weight_uniform_aligned}
\sup_{t\in[0,1]}
\left|
\widehat V^{\mathrm{oracle}}_{n,m,\mathcal A}(t)-V_{m,\mathcal A}(t)
\right|
\xrightarrow{p}0,\quad \text{as }n\to\infty,m\to\infty.
\end{equation}
For $k\in\mathcal A$, define
\[
D_{n,k}(t):=
\frac{1}{n}\sum_{i=1}^{n}
\left(\hat w_{n,m,k}(X_i)-w_k(X_i)\right)
\mathbb I\{\hat Y_i=k,\tilde Y_i\ne k,U_i\le t\}.
\]
Condition on $\hat w_{n,m,k}$, the samples $\{(X_i,Y_i,\tilde{Y}_i)\}_{i=1}^{n+m}$ remain i.i.d. Moreover,
\[
\sup_{t\in[0,1]}|D_{n,k}(t)|
\le
\frac1n\sum_{i=1}^{n}
|\hat w_{n,m,k}(X_i)-w_k(X_i)|
\mathbb I\{\hat Y_i=k,\tilde Y_i\ne k\}.
\]
Taking conditional expectation and using Cauchy's inequality, we have
\begin{align*}
\mathbb{E} \left[\sup_{t\in[0,1]}|D_{n,k}(t)|
\,\middle|\,\hat w_{n,m,k}
\right]
&\le
\mathbb E\left[|\hat w_{n,m,k}(X)-w_k(X)|\right] \\
&\le
\|\hat w_{n,m,k}-w_k\|_{L_2(P_X)}.
\end{align*}
For any $\epsilon,\delta>0$, we have
\begin{align*}
&\quad~\mathbb{P}\left(\sup_{t\in[0,1]}|D_{n,k}(t)|>\epsilon\right)\\
&\le
\mathbb P\left(
\|\hat w_{n,m,k}-w_k\|_{L_2(P_X)}>\epsilon\delta
\right)+
\mathbb P\left(
\sup_{t\in[0,1]}|D_{n,k}(t)|>\epsilon,\
\|\hat w_{n,m,k}-w_k\|_{L_2(P_X)}\le \epsilon\delta
\right)\\
&=
\mathbb P\left(
\|\hat w_{n,m,k}-w_k\|_{L_2(P_X)}>\epsilon\delta
\right)+
\mathbb E\left[
\mathbb I\left\{
\|\hat w_{n,m,k}-w_k\|_{L_2(P_X)}\le \epsilon\delta
\right\}
\mathbb P\left(
\sup_{t\in[0,1]}|D_{n,k}(t)|>\epsilon
\middle|
\hat w_{n,m,k}
\right)
\right]\\
&\overset{(i)}{\le} \mathbb P\left(
\|\hat w_{n,m,k}-w_k\|_{L_2(P_X)}>\epsilon\delta
\right)+\frac{1}{\epsilon}\mathbb E\left[\mathbb I\left\{\|\hat w_{n,m,k}-w_k\|_{L_2(P_X)}\leq \epsilon\delta \right\}\mathbb E\left[ \sup_{t\in[0,1]}|D_{n,k}(t)| \middle|\hat w_{n,m,k}\right]\right]\\
&\leq\mathbb{P}\left(\|\hat w_{n,m,k}-w_k\|_{L_2(P_X)}>\epsilon\delta\right)+\delta,
\end{align*}
where $(i)$ holds by the conditional Markov inequality.
By the $L_2(P_X)$ convergence, we have
\[
\mathbb P\left(
\|\hat w_{n,m,k}-w_k\|_{L_2(P_X)}>\epsilon\delta
\right)\to 0, \quad \text{as }n,m\to\infty.
\]
Therefore, we have
\[
\limsup_{n,m\to\infty}
\mathbb P\left(
\sup_{t\in[0,1]}|D_{n,k}(t)|>\epsilon
\right)
\le \delta.
\]
Since $\delta>0$ is arbitrary, we have
\begin{equation}\label{convergence_d}
\sup_{t\in[0,1]}|D_{n,k}(t)|=o_p(1),
\quad k\in\mathcal A.
\end{equation}
In the proof of Theorem~\ref{theorem2}, we established that
\[
\frac{|\mathcal S_k^{\mathrm{test}}|}{m}\xrightarrow{p}\mathbb{P}(\hat Y=k),
\quad
\frac{|\mathcal S_k^{\mathrm{cal}}|}{n}\xrightarrow{p}\mathbb{P}(\hat{Y}=k),
\]
which then implies that $n|\mathcal S_k^{\mathrm{test}}|/(m|\mathcal S_k^{\mathrm{cal}}|)=O_p(1)$. Since $\mathcal A$ is finite, by \eqref{convergence_d}, we have
\begin{equation}\label{eq:estimated_oracle_aligned}
\sup_{t\in[0,1]}
\left|
\widehat V^{\mathrm{est}}_{n,m,\mathcal A}(t)
-
\widehat V^{\mathrm{oracle}}_{n,m,\mathcal A}(t)
\right|=o_p(1).
\end{equation}
Combining \eqref{eq:oracle_weight_uniform_aligned} and \eqref{eq:estimated_oracle_aligned}, we obtain
\begin{equation}\label{eq:estimated_uniform_aligned}
\Delta_{n,m}
:=
\sup_{t\in[0,1]}
\left|
\widehat V^{\mathrm{est}}_{n,m,\mathcal A}(t)-V_{m,\mathcal A}(t)
\right|=o_p(1).
\end{equation}

We now prove FDR control. By Assumption~\ref{techniqe_condition}, there exists $t_0\in[0,1]$ such that
\[
R^{\infty}_{\mathcal A}(t_0)>0,
\qquad
V^{\infty}_{\mathcal A}(t_0)<\alpha_{\mathcal{A}} R^{\infty}_{\mathcal A}(t_0).
\]
Since $R_{m,\mathcal A}(t_0)\xrightarrow{p}R^{\infty}_{\mathcal A}(t_0)$ and $\widehat V^{\mathrm{est}}_{n,m,\mathcal A}(t_0)\xrightarrow{p}V^{\infty}_{\mathcal A}(t_0)$,
we have
\[
\frac{\widehat V^{\mathrm{est}}_{n,m,\mathcal A}(t_0)}
{R_{m,\mathcal A}(t_0)\vee m^{-1}}
\xrightarrow{p}
\frac{V^{\infty}_{\mathcal A}(t_0)}{R^{\infty}_{\mathcal A}(t_0)}
<\alpha_{\mathcal{A}}.
\]
Therefore, threshold $t_0$ belongs to the estimated feasible set with probability tending to one. On this event, the definition of $\hat t^{\mathrm{est}}_{\mathcal A}$ gives
\[
R_{m,\mathcal A}(\hat t^{\mathrm{est}}_{\mathcal A})
\ge
R_{m,\mathcal A}(t_0).
\]
Consequently, there exists $c_0>0$ such that
\begin{equation}\label{eq:denom_est_aligned}
\mathbb P\left\{R_{m,\mathcal A}(\hat t^{\mathrm{est}}_{\mathcal A})\ge c_0\right\}\to1.
\end{equation}
By the definition of $\hat t^{\mathrm{est}}_{\mathcal A}$, we have
\[
\widehat V^{\mathrm{est}}_{n,m,\mathcal A}(\hat t^{\mathrm{est}}_{\mathcal A})
\le
\alpha_{\mathcal{A}}\{R_{m,\mathcal A}(\hat t^{\mathrm{est}}_{\mathcal A})\vee m^{-1}\}.
\]
Combining this with \eqref{eq:denom_est_aligned} and \eqref{eq:estimated_uniform_aligned}, we get, with probability tending to one,
\begin{align*}
V_{m,\mathcal A}(\hat t^{\mathrm{est}}_{\mathcal A})\le
\hat V^{\mathrm{est}}_{n,m,\mathcal A}(\hat t^{\mathrm{est}}_{\mathcal A})+
\Delta_{n,m}\le
\alpha_{\mathcal{A}} R_{m,\mathcal A}(\hat t^{\mathrm{est}}_{\mathcal A})+o_{p}(1).
\end{align*}
Therefore,
\[
\frac{V_{m,\mathcal A}(\hat t^{\mathrm{est}}_{\mathcal A})}
{R_{m,\mathcal A}(\hat t^{\mathrm{est}}_{\mathcal A})}
\le
\alpha_{\mathcal{A}}+\frac{o_p(1)}{c_0}
=
\alpha_{\mathcal{A}}+o_p(1).
\]
The realized FDP is
\[
\mathrm{FDP}_{\mathcal A}
=
\frac{mV_{m,\mathcal A}(\hat t^{\mathrm{est}}_{\mathcal A})}
{1\vee mR_{m,\mathcal A}(\hat t^{\mathrm{est}}_{\mathcal A})}.
\]
By \eqref{eq:denom_est_aligned}, the denominator equals $mR_{m,\mathcal A}(\hat t^{\mathrm{est}}_{\mathcal A})$ with probability tending to one, which implies that $
\mathrm{FDP}_{\mathcal A}\le \alpha_{\mathcal{A}}+o_p(1)$.
Since $0\le \mathrm{FDP}_{\mathcal A}\le1$, for any $\eta>0$,
\[
\mathbb E[\mathrm{FDP}_{\mathcal A}]
\le
\alpha_{\mathcal{A}}+\eta+
\mathbb P\{\mathrm{FDP}_{\mathcal A}>\alpha_{\mathcal{A}}+\eta\}.
\]
Since $
\mathrm{FDP}_{\mathcal A}\le \alpha_{\mathcal{A}}+o_p(1)$, we have
\[
\limsup_{n,m\to\infty}\mathrm{FDR}_{\mathcal A}
=
\limsup_{n,m\to\infty}\mathbb E[\mathrm{FDP}_{\mathcal A}]
\le \alpha_{\mathcal{A}}+\eta.
\]
Letting $\eta\downarrow0$ gives $\limsup_{n,m\to\infty}\mathrm{FDR}_{\mathcal A}\le\alpha_{\mathcal{A}}$,
which completes the proof.
\end{proof}

\subsection{Proof of Proposition~\ref{FDR_uniform_label}}
\begin{proof}
Fix $g\in[G]$ and write $\mathcal A=\mathcal A_g$ and $\alpha_{\mathcal{A}}=\alpha_{\mathcal A_g}$. Set
\[
\tilde\alpha_{\mathcal{A}}:=(1-\epsilon)\alpha_{\mathcal{A}}+
\epsilon\left(1-\frac1K\right)
=
\alpha_{\mathcal{A}}+
\epsilon\left(1-\alpha_{\mathcal{A}}-\frac1K\right).
\]
For $t\in[0,1]$, let $R_{m,\mathcal A}(t),V_{m,\mathcal A}(t),R^\infty_{\mathcal A}(t)$, and $V^\infty_{\mathcal A}(t)$ be defined as before. Also define
\[
\tilde V^\infty_{\mathcal A}(t)
:=
\sum_{k\in\mathcal A}
\mathbb P\{\hat Y=k,\tilde Y\ne k,U\le t\}.
\]
Under the randomized response model with noise rate $\epsilon\in(0,1)$, we have
\[
\mathbb P(\widetilde Y\ne k\mid X)
=(1-\epsilon)\mathbb P(Y\ne k\mid X)
+
\epsilon\left(1-\frac1K\right).
\]
Since $\hat Y$ and $U$ are functions of $X$, it follows that, for every $t\in[0,1]$, we have
\begin{equation}\label{eq:uniform_noise_population_identity}
\begin{aligned}
\tilde{V}_\mathcal{A}^{\infty}(t)
&=\sum_{k\in\mathcal{A}}\mathbb{E}[\mathbb{I}\{\hat{Y}=k,U\leq t\}\mathbb{P}(\tilde{Y}\neq k|X)]\\
&= \sum_{k\in\mathcal{A}}\mathbb{E} \left[ \mathbb{I}\{\hat{Y} = k, U \le t\} \left( (1 - \epsilon)\mathbb{P}(Y \neq k \mid X) + \epsilon \left(1 - 1/K\right) \right) \right]\\
&= \sum_{k\in\mathcal{A}}(1 - \epsilon) \mathbb{E}_X \big[ \mathbb{I}\{\hat{Y} = k, U \le t\} \mathbb{P}(Y \neq k \mid X) \big] + \sum_{k\in\mathcal{A}}\epsilon (1 - 1/K) \mathbb{E}_X \big[ \mathbb{I}\{\hat{Y} = k, U \le t\} \big]\\
&=\sum_{k\in\mathcal{A}}(1-\epsilon)\mathbb{P}(\hat{Y} = k, Y \neq k, U \le t)+\sum_{k\in\mathcal{A}}\epsilon \left(1 - 1/K\right) \mathbb{P}(\hat{Y} = k, U \le t)\\
&=(1-\epsilon)V_{\mathcal{A}}^{\infty}(t)+\epsilon\left(1-1/K\right) R_{\mathcal{A}}^{\infty}(t).
\end{aligned}
\end{equation}
Consequently, whenever $R^\infty_{\mathcal A}(t)>0$, we have
\begin{equation}\label{eq:uniform_noise_feasible_equiv}
\frac{\tilde V^\infty_{\mathcal A}(t)}{R^\infty_{\mathcal A}(t)}
\le \tilde\alpha_{\mathcal{A}}
\quad\Longleftrightarrow\quad
\frac{V^\infty_{\mathcal A}(t)}{R^\infty_{\mathcal A}(t)}
\le \alpha_{\mathcal{A}}.
\end{equation}
Let
\begin{equation}\label{oracle_threshold_uni}
\tilde{\mathcal{F}}_{\mathcal{A}} = \left\{t \in [0,1]: R_{\mathcal{A}}^{\infty}(t)>0,~ \frac{\tilde{V}_{\mathcal{A}}^{\infty}(t)}{R_{\mathcal{A}}^{\infty}(t)}\le \tilde{\alpha}_{\mathcal{A}}\right\}.
\end{equation}
By \eqref{eq:uniform_noise_feasible_equiv}, $\tilde{\mathcal{F}}_{\mathcal{A}}=\mathcal F_{\mathcal A}$.
Let
\[
\widehat{\tilde V}_{n,m,\mathcal A}(t)
:=
\sum_{k\in\mathcal A}
\frac{|\mathcal S_k^{\mathrm{test}}|}{m}
\left[
\frac1{|\mathcal S_k^{\mathrm{cal}}|}
\sum_{i\in\mathcal S_k^{\mathrm{cal}}}
\mathbb I\{\tilde Y_i\ne k,U_i\le t\}
\right]
\]
be the normalized numerator of $\widetilde{\mathrm{FDP}}_{\mathcal A}(t)$. 
Since $\mathbb P(\hat Y=k)>0$ for $k\in\mathcal A$, we only focus on the case $|\mathcal S_k^{\mathrm{cal}}|>0$.
The same Glivenko--Cantelli argument used in the proof of Theorem~\ref{theorem2}, applied to the one-dimensional threshold class $\{\mathbb I(U\le t):t\in[0,1]\}$ and the finite collection of classes $k\in\mathcal A$, gives
\begin{equation}\label{eq:uniform_noise_calib_uniform}
\sup_{t\in[0,1]}
\left|
\widehat{\tilde V}_{n,m,\mathcal A}(t)-
\tilde V^\infty_{\mathcal A}(t)
\right|=o_p(1).
\end{equation}

We first prove FDR control. By Assumption~\ref{techniqe_condition} and \eqref{eq:uniform_noise_feasible_equiv}, there exists $t_0\in[0,1]$ such that
\[
R^\infty_{\mathcal A}(t_0)>0,
\qquad
\tilde V^\infty_{\mathcal A}(t_0)<\tilde\alpha_{\mathcal A} R^\infty_{\mathcal A}(t_0).
\]
Combining \eqref{eq:uniform_noise_calib_uniform} and $\sup_{t\in[0,1]}|R_{m,\mathcal A}(t)-R^\infty_{\mathcal A}(t)|=o_p(1)$ established before, we have
\[
\frac{\widehat{\tilde V}_{n,m,\mathcal A}(t_0)}
{R_{m,\mathcal A}(t_0)\vee m^{-1}}
\xrightarrow{p}
\frac{\tilde V^\infty_{\mathcal A}(t_0)}{R^\infty_{\mathcal A}(t_0)}
<\tilde\alpha_{\mathcal{A}}.
\]
Therefore $t_0$ belongs to the empirical feasible set in \eqref{uniform_noise_threshold} with probability tending to one. On this event, the definition of $\tilde t_{\mathcal A}$ gives
$R_{m,\mathcal A}(\tilde t_{\mathcal A})\ge R_{m,\mathcal A}(t_0)$.
Thus there exists a constant $c_0>0$ such that
\begin{equation}\label{eq:uniform_noise_denominator}
\mathbb P\left\{R_{m,\mathcal A}(\tilde t_{\mathcal A})\ge c_0\right\}
\to1.
\end{equation}
By the definition of $\tilde t_{\mathcal A}$, we have
$
\widehat{\tilde V}_{n,m,\mathcal A}(\tilde t_{\mathcal A})
\le
\tilde\alpha_{\mathcal{A}}\{R_{m,\mathcal A}(\tilde t_{\mathcal A})\vee m^{-1}\}.
$
Using \eqref{eq:uniform_noise_denominator} and \eqref{eq:uniform_noise_calib_uniform}, we have
\begin{equation}\label{eq:uniform_noise_feasible_population}
\tilde V^\infty_{\mathcal A}(\tilde t_{\mathcal A})
\le
\tilde\alpha_{\mathcal{A}} R_{m,\mathcal A}(\tilde t_{\mathcal A})+o_p(1).
\end{equation}
By \eqref{eq:uniform_noise_population_identity} and $\sup_{t\in[0,1]}|V_{m,\mathcal A}(t)-V^\infty_{\mathcal A}(t)|
=o_p(1)$ established before, we have
\[
\tilde V^\infty_{\mathcal A}(\tilde t_{\mathcal A})
=
(1-\epsilon)V_{m,\mathcal A}(\tilde t_{\mathcal A})
+
\epsilon\left(1-\frac1K\right)R_{m,\mathcal A}(\tilde t_{\mathcal A})
+o_p(1).
\]
Combining this with \eqref{eq:uniform_noise_feasible_population} and using
$\tilde\alpha_{\mathcal{A}}=(1-\epsilon)\alpha_{\mathcal{A}}+
\epsilon(1-1/K)$, we have
\[
(1-\epsilon)V_{m,\mathcal A}(\tilde t_{\mathcal A})
\le
(1-\epsilon)\alpha_{\mathcal{A}} R_{m,\mathcal A}(\tilde t_{\mathcal A})+o_p(1).
\]
Since the randomized response model has $1-\epsilon\in(0,1)$, we have 
\[
V_{m,\mathcal A}(\tilde t_{\mathcal A})
\le
\alpha_{\mathcal A} R_{m,\mathcal A}(\tilde t_{\mathcal A})+o_p(1).
\]
Together with \eqref{eq:uniform_noise_denominator}, we have
\[
\mathrm{FDP}_{\mathcal A}
=
\frac{mV_{m,\mathcal A}(\tilde t_{\mathcal A})}
{1\vee mR_{m,\mathcal A}(\tilde t_{\mathcal A})}
\le
\alpha_{\mathcal A}+o_p(1).
\]
By the same argument in the proof of Theorem~\ref{robust_rcs}, we have $\limsup_{n,m\to\infty}\mathrm{FDR}_{\mathcal A}\leq \alpha_{\mathcal{A}}$.

It remains to verify the power statement. We have established that $\tilde{\mathcal{F}}_{\mathcal{A}}=\mathcal{F}_{\mathcal{A}}$.
Rewrite $\tilde{t}_{\mathcal{A}}$ by $\tilde{t}_{\mathcal{A}}^{(n,m)}$.
Therefore, the same argmax argument used in the proof of Theorem~\ref{theorem2} implies that every accumulation point of $\tilde t_{\mathcal A}^{(n,m)}$ belongs to $\arg\max_{t\in\mathcal F_{\mathcal A}}R_{\mathcal A}^\infty(t)$.
Since the selection set is nested for different $t$, any maximizer of $R_{\mathcal A}^\infty$ over $\mathcal F_{\mathcal A}$ also maximizes $T_{\mathcal A}^\infty$ over $\mathcal F_{\mathcal A}$. Consequently, we have
$
T_{\mathcal A}^\infty(\tilde t^{(n,m)}_{\mathcal A})
\xrightarrow{p}
\sup_{t\in\mathcal F_{\mathcal A}}T_{\mathcal A}^\infty(t)
$.
Together with the uniform convergence of $T_{m,\mathcal A}$ established in Theorem~\ref{theorem2}, we have
\[
T_{m,\mathcal A}(\tilde t_{\mathcal A}^{(n,m)})
\xrightarrow{p}
\sup_{t\in\mathcal F_{\mathcal A}}T_{\mathcal A}^\infty(t).
\]
Moreover, we have
\[
\frac1m\sum_{j=1}^m\mathbb I\{Y_{n+j}\in\mathcal A\}
\xrightarrow{p}
\mathbb P(Y\in\mathcal A)>0,
\]
which then implies that
\[
\frac{T_{m,\mathcal A}(\tilde t_{\mathcal A}^{(n,m)})}
{m^{-1}\vee m^{-1}\sum_{j=1}^m\mathbb I\{Y_{n+j}\in\mathcal A\}}
\xrightarrow{p}
\frac{\sup_{t\in\mathcal F_{\mathcal A}}T_{\mathcal A}^\infty(t)}
{\mathbb P(Y\in\mathcal A)}.
\]
Since the left-hand side is bounded by one, the convergence in probability above implies convergence of expectations. Thus, we have
\[
\lim_{n,m\to\infty}\mathrm{rPower}_{\mathcal A}
=
\frac{\sup_{t\in\mathcal F_{\mathcal A}}T_{\mathcal A}^\infty(t)}
{\mathbb P(Y\in\mathcal A)}
=
\sup_{t\in\mathcal F_{\mathcal A}}\mathrm{rPower}_{\mathcal A}^\infty(t).
\]
This completes the proof.

\end{proof}

\section{Additional Details for Simulations}

\subsection{Supplementary Details about Simulations for Task 1}\label{simul:estimate_T}

In Section~\ref{simu:estimate}, we introduce different strategies for estimating the transition matrix $\mathbf{T}$. In Section~\ref{simu:bounded_asymmetric}, we introduce additional simulation results for Task 1.

\subsubsection{Estimation of contamination models}\label{simu:estimate}

\paragraph*{Estimation of the randomized response model.}
Let 
\[
\psi=\mathbb{P}(\hat{f}(X)=Y),\quad \tilde{\psi}=\mathbb{P}(\hat{f}(X)=\tilde{Y}).
\]
Denote the additional clean and noisy dataset by $\mathcal{D}_0=\left\{(X_i^0,Y_i^0)\right\}_{i=1}^{Kn_{\text{clean}}^{\text{fit}}}$ and $\mathcal{D}_1=\{(X_i^1,\tilde{Y}_i^1)\}_{i=1}^{Kn_{\text{noisy}}^{\text{fit}}}$, respectively. We estimate $\psi$ and $\tilde{\psi}$ by 
\[
\hat{\psi}=\frac{1}{Kn_{\text{clean}}^{\text{fit}}}\sum_{i=1}^{Kn_{\text{clean}}^{\text{fit}}}\mathbb{I}\{\hat{f}(X_i^{0})=Y_i^{0}\},\quad \hat{\tilde{\psi}}=\frac{1}{Kn_{\text{noisy}}^{\text{fit}}}\sum_{i=1}^{Kn_{\text{noisy}}^{\text{fit}}}\mathbb{I}\{\hat{f}(X_i^{1})=\tilde{Y}_i^{1}\}.
\]
By the definition of the randomized response model with level $\epsilon\in(0,1)$, we have
\[
\tilde{\psi}=(1-\epsilon)\psi+\frac{\epsilon}{K}.
\]
Therefore, a plug-in estimate of $\epsilon$ is 
\[
\hat{\epsilon}_{\text{plug-in}}=\frac{\hat{\psi}-\hat{\tilde{\psi}}}{\hat{\psi}-1/K}.
\]

We also leverage bootstrap~\citep{bootstrap1979} to obtain a $(1-\beta)\times100\%$ one-sided lower bound. In practice, practitioners can also use other methods such as the random weighting method~\citep{ming2025random} to obtain it.
Specifically, for $b\in[B]$, we resample a dataset $\mathcal{D}_0^{(b)}$ with size $Kn_{\text{clean}}^{\text{fit}}$ from $\mathcal{D}_0$ and a dataset $\mathcal{D}_1^{(b)}$ with size $Kn_{\text{noisy}}^{\text{fit}}$ from $\mathcal{D}_1$.
By the same estimation procedure, we can obtain $\hat{\psi}^{(b)}$ and $\hat{\tilde{\psi}}^{(b)}$, and then compute $\hat{\epsilon}_{\text{plug-in}}^{(b)}=(\hat{\psi}^{(b)}-\hat{\tilde{\psi}}^{(b)})/(\hat{\psi}^{(b)}-1/K)$.
Finally, the $(1-\beta)\times100\%$ one-sided lower bound of $\epsilon$ can be estimated by the $\beta$ quantile of $\left\{\hat\epsilon^{(1)}_{\text{plug-in}},\dots,\hat{\epsilon}^{(B)}_{\text{plug-in}}\right\}$.
In our simulations, we set $\beta=0.05$ and $B=200$.

\paragraph*{Estimation of general contamination models.}
\cite{sesia2025adaptive} provide a strategy to estimate an unknown label contamination model with sets $\mathcal{D}_0$ and $\mathcal{D}_1$. In our experiments, we consider a more straightforward setting where we have a dataset of $\mathcal{D}_{\text{fit}}=\{(X_i^{\text{fit}},Y_i^{\text{fit}},\tilde{Y}_i^{\text{fit}})\}_{i=1}^{Kn_{\text{pair}}}$, where both the clean and noisy labels are observed. This setting is common in practice, as practitioners often manually relabel a small subset to obtain paired clean and noisy labels.
The plug-in estimator is then constructed by
\begin{equation*}    \label{eq:plugin_transition_block}
    \hat T_{jk}
    =
    \frac{\#\{i:Y_i^{\text{fit}}=j,\tilde Y_i^{\text{fit}}=k\}+\lambda}
    {\#\{i:Y_i^{\text{fit}}=j\}+K\lambda},
\end{equation*}
where the term $\lambda$ represents the Laplace smoothing. In our simulations, we set $\lambda=1$.

\paragraph{Details of model mis-specification.}
We introduce details on the experiments about model mis-specification in Section~\ref{task1:case3}. The underlying contamination model is a block-wise matrix $\mathbf{T}$. We incorrectly assume the randomized-response model and therefore estimates only a scalar noise level $\epsilon$ from the paired fitting sample. Specifically, let
\[
\hat{q}=\frac{1}{K n_{\mathrm{pair}}^{\mathrm{fit}}}\sum_{i=1}^{Kn_{\text{pair}}} \mathbb{I}\{\tilde Y_i^{\text{fit}}\neq Y_i^{\text{fit}}\}.
\]
Under the randomized response model, we have $\mathbb{P}(\tilde Y\neq Y)=\epsilon(1-1/K)$. Therefore, the estimate of $\epsilon$ under model mis-specification is given by $\hat\epsilon_{\mathrm{mis}}=\hat{q}/(1-1/K)$.
The corresponding working transition matrix is
\begin{equation*}  \label{eq:rr_misspecified_transition}
    \hat T^{\mathrm{mis}}_{jk}=(1-\hat\epsilon_{\mathrm{mis}})\mathbb{I}\{j=k\} +\frac{\hat\epsilon_{\mathrm{mis}}}{K},
    \quad j,k\in[K].
\end{equation*}
Then the \texttt{RCS (mis-specified $\mathbf{T}$)} is exactly defined by the \texttt{RCS (plug-in $\mathbf{T}$)} with the mis-specified $\hat{\mathbf{T}}^{\mathrm{mis}}=(\hat T^{\mathrm{mis}}_{jk})$.

\subsubsection{Results under bounded asymmetric label noise}\label{simu:bounded_asymmetric}
We report the results under bounded asymmetric label noises in Figure~\ref{1_asym_overall}, where the parameter settings are identical to those in Figure~\ref{1_uni_overall}. 

\begin{figure}[ht]
\centering
\includegraphics[width=0.8\textwidth]{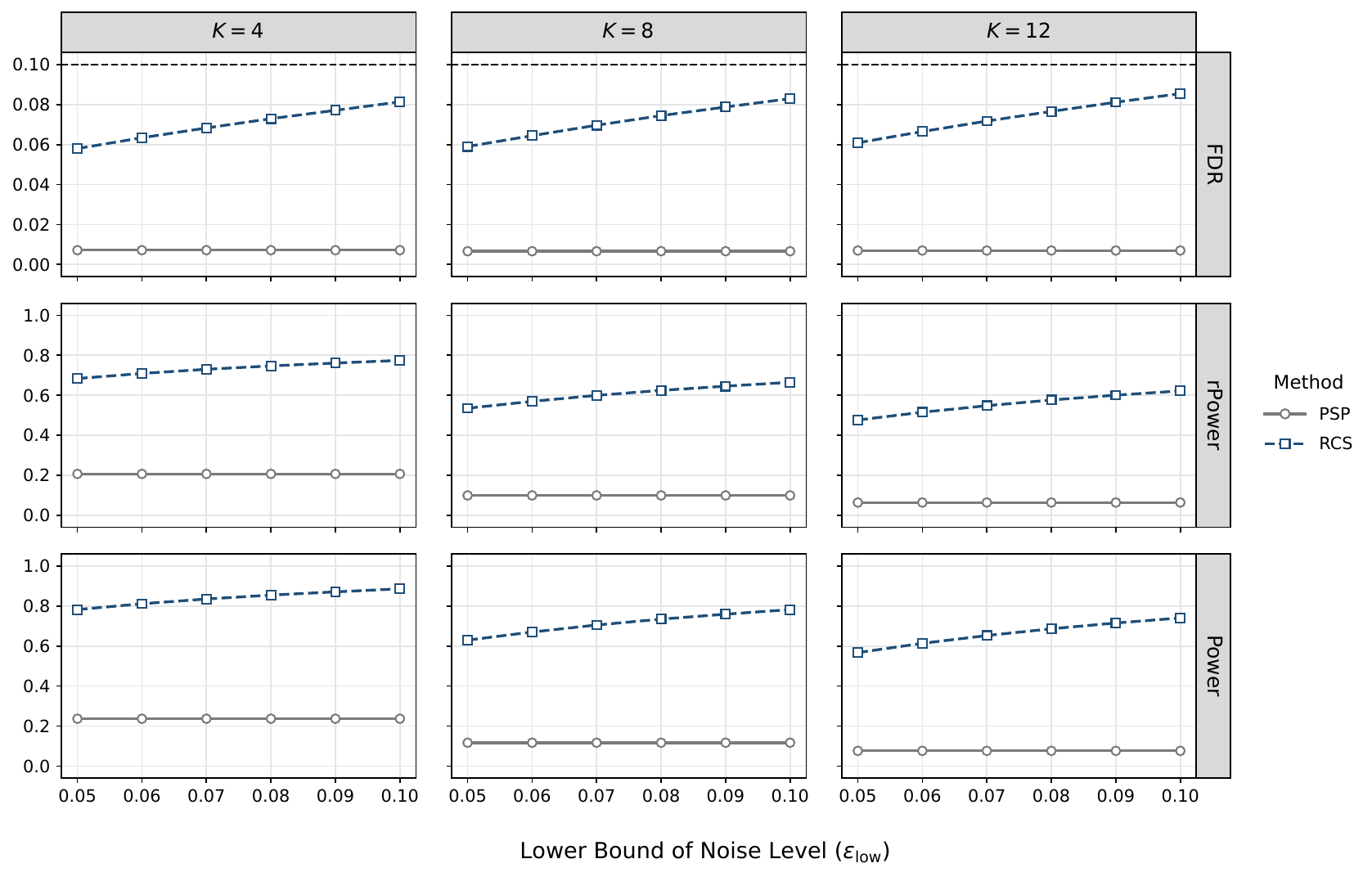}
    \caption{Performance of methods under asymmetric label noises ($\epsilon=0.1$) with bounded $\epsilon_{\text{low}}$ in overall classification.}
    \label{1_asym_overall}
\end{figure}

\subsection{Supplementary Details about Simulations for Task 2}\label{task_2_appendix}

In Section~\ref{task_2_appendix_1}, we introduce detailed information about the additive Gaussian contamination used in Section~\ref{Additive_response_noise}.
In Section~\ref{task_2_appendix_2}, we introduce the score-localized weight estimator used in Figure~\ref{fig:task2-additive}.

\subsubsection{Details of additive Gaussian contamination}\label{task_2_appendix_1}
In Section~\ref{sec:regression}, we consider the common additive Gaussian contamination model. The contaminated response is constructed by $\tilde Y=Y+\eta$, where $\eta\mid X=x\sim N(0,\tau_\rho^2(x))$.
For each replication, we let $\mathcal Y_{\rm all}$ and $\mathcal M_{\rm all}$  denote the pooled clean responses $\{Y_i\}$ and the pooled scores $\{\mu(X_i)\}$ from the training, calibration, test, and audit samples, respectively.
Let
\begin{align*}
s_Y
&= \max\left\{\frac{\mathrm{Quantile}_{0.75}(\mathcal Y_{\rm all})-\mathrm{Quantile}_{0.25}(\mathcal Y_{\rm all})}{1.349},0.05 \right\},
\end{align*}
where $1.349$ is the interquartile range (IQR) of a standard normal distribution, so that $s_Y$ is an IQR-based robust estimate of the marginal response standard deviation.
Similarly, let
\begin{align*}
s_\mu
&= \max\left\{ \mathrm{Quantile}_{0.75}(\mathcal M_{\rm all})-\mathrm{Quantile}_{0.25}(\mathcal M_{\rm all}),10^{-6}
    \right\}.
\end{align*}
Fixed a noise-level parameter $\rho\geq0$.
For the homoscedastic contamination, we set $\tau_{\rho}=\rho s_Y$ such that the additive measurement error has the same standard deviation for all $x\in\mathcal{X}$.
For the score-dependent setting, we let
\[
\tau_\rho(x) =\rho s_Y(0.5+2r(x)), \quad
    r(x) =\frac{1}{1+\exp\left\{ -3\mu(x)/s_\mu \right\} }.
\]
Thus, the contamination level ranges approximately from $0.5\rho s_Y$ in low-response regions to \(2.5\rho s_Y\) in high-response regions, mimicking heterogeneous measurement error that is stronger for promising candidates with larger conditional responses.

\subsubsection{Details of score-localized RCS}\label{task_2_appendix_2}

We describe the implementation of \texttt{RCS (score-localized)} used in Figure~\ref{fig:task2-additive}. We have
\begin{align*}
\tilde{\nu}_1(x)
&=\mathbb{P}(\tilde{L}=1~|~L=1,X=x)\mathbb{P}(L=1~|~X=x)+\mathbb{P}(\tilde{L}=1~|~L=2,X=x)\mathbb{P}(L=2~|~X=x)\\
&=(a_1(x)-a_2(x))\mathbb{P}(L=1~|~X=x)+a_2(x).
\end{align*}
Therefore, the localized weight can be represented by
\begin{equation*}
w(x)=\frac{\mathbb{P}(L=1~|~X=x)}{\mathbb{P}(\tilde{L}=1~|~X=x)}=\frac{\tilde{\nu}_1(x)-a_2(x)}{(a_1(x)-a_2(x))\tilde{\nu}_1(x)}
\end{equation*}
whenever \(a_1(x)\neq a_2(x)\).
To estimate $\tilde{\nu}_1(x)$, we first translate the original training data $\{(X_i^{\text{train}},\tilde{Y}_i^{\text{train}})\}$ to $\{(X_i^{\text{train}},\tilde{L}_i^{\text{train}})\}$, where $\tilde{L}_i^{\text{train}}=1+\mathbb{I}\{\tilde{Y}_i^{\text{train}}>c\}$. Then the estimator $\hat{\nu}^{\text{noisy}}_1(x)$ is obtained by training a soft-classifier on the translated training data.

The remaining problem is to estimate $\{a_{\ell}(x)\}_{\ell=1}^2$ by leveraging an additional dataset 
$\mathcal{D}_{\text{fit}}=\{(X_i^{\text{fit}},Y_i^{\text{fit}},\tilde{Y}_i^{\text{fit}})\}_{i=1}^{n_{\text{fit}}}$, which is a well-studied problem in the past few decades. Although there are several parametric and nonparametric methods, in our simulations, to show the robustness of the RCS framework, we simply estimate $a_{\ell}(x)$ through a one-dimensional score-localized approximation. Specifically, we define the signed score $S(x)=\hat\mu(x)-c$, where $\hat\mu(x)$ is the noisy-response regressor used by both \texttt{BH\_clip}
and RCS. We divide $\mathcal{D}_{\text{fit}}$ into $B$ bins, denoted by $\{\mathcal I_b\}_{b=1}^B$, based on the empirical quantiles of
\(\{S(X_i^{\rm fit})\}_{i=1}^{n_{\rm fit}}\). Let $b(x)$ be the index of $S(x)$, i.e., $S(x)\in\mathcal{I}_{b(x)}$.
For each $b\in[B]$, we compute
\begin{equation*}
    \hat a_{\ell b}= \frac{  \sum_{i=1}^{n_{\rm fit}}\mathbb I\{L_i^{\rm fit}=\ell,\, \tilde L_i^{\rm fit}=1,\,
              S(X_i^{\rm fit})\in\mathcal I_b\} +1}{\sum_{i=1}^{n_{\rm fit}}
    \mathbb I\{L_i^{\rm fit}=\ell, S(X_i^{\rm fit})\in\mathcal I_b\}+2}, \quad \ell\in\{1,2\}.
\end{equation*}
For a new point \(x\), we set
\begin{equation*}
    \hat a_\ell(x)=\hat a_{\ell,b(x)},\quad \ell=1,2.
\end{equation*}
Then we obtain the plug-in estimated localized weight
\[
    \hat w(x)
    =
    \frac{\hat{\nu}^{\text{noisy}}_1(x)-\hat a_2(x)}
    {(\hat a_1(x)-\hat a_2(x))\hat{\nu}^{\text{noisy}}_1(x)}.
\]
In our simulations, we take a small $n_{\text{fit}}=500$ and simply set $B=3$.


\vskip 0.2in
\bibliography{reference}

\end{document}